\documentclass[11pt,oneside]{article}

\usepackage[utf8]{inputenc} 
\usepackage[T1]{fontenc}    
\usepackage[english]{babel}

\usepackage{amsfonts}
\usepackage{nicefrac}       
\usepackage{microtype}      
\usepackage{lmodern}
\usepackage{amssymb,amsmath,amsthm}
\usepackage{stmaryrd}
\usepackage{bbm,bm}
\usepackage{latexsym}
\usepackage{xcolor}

\usepackage{enumerate}
\usepackage{verbatim}
\usepackage{booktabs}       

\usepackage{url}            
\usepackage[colorlinks=true,citecolor=blue]{hyperref}
\usepackage[numbers,sort&compress,square,comma]{natbib}
\usepackage[capitalise,noabbrev]{cleveref}

\usepackage{parskip}
\usepackage{geometry}
\usepackage[short]{optidef}
\usepackage{algorithm}

\usepackage{thmtools}

\declaretheorem{theorem}
\declaretheorem{corollary}
\declaretheorem{lemma}

\declaretheorem{proposition}

\declaretheoremstyle[qed=$\square$]{definitionwithend}
\declaretheorem[style=definitionwithend]{definition}
\declaretheorem[style=definitionwithend]{assumption}
\declaretheoremstyle[qed=$\blacksquare$]{remarkwithend}
\declaretheorem{example}
\declaretheorem[style=remarkwithend]{remark}

\crefname{assumption}{Assumption}{Assumptions}
\crefname{conjecture}{Conjecture}{Conjectures} 
\usepackage{enumitem}
\usepackage{graphicx}
\usepackage{import}

\usepackage{booktabs} 

\usepackage{mathtools}

\usepackage[capitalise]{cleveref}

\usepackage{multirow}
\usepackage{algorithm}
\usepackage{algpseudocode}
\usepackage{tikz}
\usepackage{makecell}
\usepackage{lscape}         

\allowdisplaybreaks

\newcommand{\hinge}{\textrm{hinge}}

\newcommand{\replace}[2]{#2} 
\newcommand{\obj}[1]{h(#1;\mathcal{A}^+,\mathcal{A}^-)}  
\newcommand{\objt}[1]{h(#1;\widetilde{\mathcal{A}}_t^+,\widetilde{\mathcal{A}}_t^-)}  

\DeclareMathOperator{\lbl}{label}
\DeclareMathOperator{\plbl}{\widetilde{label}}
\DeclareMathOperator{\phlbl}{\widehat{label}}
\DeclareMathOperator{\cost}{cost}
\newcommand{\tilD}{\widetilde{D}}
\DeclareFontFamily{U}{mathx}{\hyphenchar\font45}
\DeclareFontShape{U}{mathx}{m}{n}{<-> mathx10}{}
\DeclareSymbolFont{mathx}{U}{mathx}{m}{n}
\DeclareMathAccent{\widebar}{0}{mathx}{"73}
\newcommand{\barD}{\widebar{D}}

\definecolor{gold}{rgb}{0.85,0.65,0}
\colorlet{dgreen}{green!60!black}

\newcommand{\grad}{\ensuremath{\nabla}}

\let\emptyset\varnothing
\newcommand{\set}[1]{\left\{#1\right\}}

\def\L{{\mathbb{L}}}
\def\bbL{{\mathbb{L}}}

\def\N{{\mathbb{N}}}
\def\bbN{{\mathbb{N}}}

\def\bbP{{\mathbb{P}}}

\def\R{{\mathbb{R}}}
\def\bbR{{\mathbb{R}}}

\def\cA{\mathcal{A}}

\def\cD{{\mathcal D}}
\def\cE{{\mathcal E}}

\def\cN{{\mathcal N}}

\DeclareMathOperator*{\argmin}{arg\,min}
\DeclareMathOperator*{\argmax}{arg\,max}

\DeclarePairedDelimiterX{\inner}[2]{\langle}{\rangle}{#1, #2}

\DeclareMathOperator{\conv}{conv}

\DeclareMathOperator{\sign}{sign}

\DeclareMathOperator{\Proj}{Proj}

 \graphicspath{{./fig/}}

\setcitestyle{authoryear}

\crefformat{equation}{(#2#1#3)}

\usepackage{authblk}

\begin{document}

\title{Mistake, Manipulation and Margin Guarantees in Online Strategic Classification}

\author[1]{Lingqing Shen}
\author[2]{Nam Ho-Nguyen}
\author[2]{Khanh-Hung Giang-Tran}
\author[1]{Fatma K{\i}l{\i}n\c{c}-Karzan}
\affil[1]{Tepper School of Business, Carnegie Mellon University}
\affil[2]{Discipline of Business Analytics, The University of Sydney}
\date{}

\maketitle

\begin{abstract}

We consider an online strategic classification problem where each arriving agent can manipulate their true feature vector to obtain a positive predicted label, while incurring a cost that depends on the amount of manipulation. The learner seeks to predict the agent's true label given access to only the manipulated features. After the learner releases their prediction, the agent's true label is revealed. Previous algorithms such as the strategic perceptron guarantee finitely many mistakes under a margin assumption on agents' true feature vectors. However, these are not guaranteed to encourage agents to be truthful. Promoting truthfulness is intimately linked to obtaining adequate margin on the predictions, thus we provide two new algorithms aimed at recovering the maximum margin classifier in the presence of strategic agent behavior. We prove convergence, finite mistake and finite manipulation guarantees for a variety of agent cost structures. We also provide generalized versions of the strategic perceptron with mistake guarantees for different costs. Our numerical study on real and synthetic data demonstrates that the new algorithms outperform previous ones in terms of margin, number of manipulation and number of mistakes.
\end{abstract}

\section{Introduction}

Binary classification is a well-known problem in supervised learning, with applications in numerous important domains such as marketing, finance, natural language processing and medicine. The traditional binary classification problem aims to learn a \emph{decision rule} that maps \emph{feature vectors} to \emph{binary labels} $\pm 1$, with the aim of predicting a true underlying label for a feature vector. For example, features may correspond to identifying information of customers of a bank who apply for a loan, and in this context the label may indicate whether the bank will approve or deny the loan. The true underlying label, whether the bank \emph{should} approve or deny given all future outcomes, is unknown at the time of the loan application, which necessitates the need to use a classification rule. 

Like the example above, binary classification is now regularly applied to various applications involving human \emph{agents}. Customers obviously prefer that their loan application be approved rather than denied, and in many other applications there may similarly be one label that is preferred by agents over the other. One can imagine that in practice this leads to \emph{strategic behavior} of agents, where feature vectors are manipulated in order to achieve the desired label prediction. Of course, there is often also a \emph{cost} associated with manipulation, so agents may manipulate only if the payoff for achieving the desirable label is worth the cost.

Amidst potential misleading strategic behavior of agents, the \emph{learner} wishes to find a decision rule that still accurately predicts the true underlying label, despite given access to only potentially manipulated feature vectors. This is the problem of \emph{strategic classification}. In this paper, we study an \emph{online} model for strategic classification with \emph{continuous} features (i.e., feature vectors are in $\R^d$), where agents arrive sequentially, observe the current decision rule, and present their potentially manipulated feature vectors to the learner. Upon each arrival, the learner uses the current decision rule to classify, i.e., predict the label of the agent. Then, the true label is revealed, and the learner may update the decision rule based on the revealed information of the agent's potentially manipulated feature vector and true label.

\subsection{Literature}

The idea of strategic agents dates back to \citet{bruckner_stackelberg_2011,dalvi_adversarial_2004,dekel_incentive_2008}, where the focus is on agents' adversarial behavior. 
In this setting, the agents' goal is to completely mislead the learner, and the agents are agnostic to the label they receive. 
The strategic classification problem as described above, where agents have preferences over labels and behave in a self-interested fashion, is introduced by \citet{hardt_strategic_2016}, where a Stackelberg game is used to model the learner (leader) and agent (follower) behavior. 

Strategic classification is studied in two settings: (i) the \emph{offline setting} in which data points come from an underlying distribution and the learner aims to find the equilibrium of the resulting Stackelberg game;
and (ii) in the \emph{online setting} where the data points arrive sequentially and the learner has the opportunity to revise their classifier. 
In the offline setting the focus has been on developing methods to approximately compute the Stackelberg equilibrium, and understanding their sample complexity
\citep{hardt_strategic_2016,zhang_incentive_2021,sundaram_pac_2021,lechner_learning_2022,perdomo_performative_2020,zrnic_leads_2021}.
On the other hand, the algorithms in the online setting aim to minimize the Stackelberg regret \citep{chen_learning_2020,dong_strategic_2018,ahmadi_fundamental_2023} or the number of mistakes over a time horizon \citep{ahmadi_strategic_2021,ahmadi_fundamental_2023}. As mentioned above, our work studies the latter \emph{online} setting, and we analyze the number of mistakes, manipulations and convergence of the classifiers over time.

The strategic classification literature can also be categorized based on the assumptions on the agent manipulation structures. The most common assumption is  \emph{continuous} manipulation of feature vectors in $\R^d$, where any perturbation of the feature vectors is possible within a bounded set and the cost is a continuous function (typically a weighted norm)~\citep{dong_strategic_2018,chen_learning_2020,haghtalab_maximizing_2020,ahmadi_strategic_2021,ghalme_strategic_2021,sundaram_pac_2021}. 
Alternatively, a \emph{discrete} manipulation structure where only some manipulations are plausible, defined through a manipulation graph, is also studied in~\citep{lechner_learning_2022,zhang_incentive_2021,ahmadi_fundamental_2023}. 
Relatedly, some works have focused on relaxing assumptions on the learner-agent interaction model, such as allowing for imperfect information \citep{ghalme_strategic_2021,bechavod_information_2022,jagadeesan_alternative_2021} and collective agent actions \citep{zrnic_leads_2021}. 
\citet{ghalme_strategic_2021,bechavod_information_2022,jagadeesan_alternative_2021} all examine the offline problem, while \citet{zrnic_leads_2021} considers the online problem where the agents' learn their best actions collectively and over time. 
In contrast, our work examines the online model, with continuous manipulations, and we consider the standard full information learner-agent interaction model, where agents are individualistic and can respond optimally instantly. 

Another concept worth mentioning, although not yet formalized in the literature, is the ability of an algorithm to encourage \emph{truthfulness} of agents, by presenting classifiers which do not incentivize agents to manipulate their features. In contrast to earlier work \citep{hardt_strategic_2016} that mainly focused on the misclassification error of an algorithm as a performance metric, \citet{zhang_incentive_2021,lechner_learning_2022} take the vulnerability to manipulation into account. In particular, \citet{zhang_incentive_2021} propose an algorithm that prevents strategic behaviors, and \citet{lechner_learning_2022} propose a strategic empirical manipulation loss that balances the classification accuracy and manipulation. Nevertheless, the context therein is different to ours: these works all consider the model of manipulation graph and offline learning, and their results mainly involve sample complexity analysis.

The possibility of strategic manipulation of feature vectors has also brought up a number of other closely related questions. In particular, \citet{kleinberg_how_2020,miller_strategic_2020,ahmadi_classification_2022,bechavod_information_2022,shavit_causal_2020,haghtalab_maximizing_2020,alon_multiagent_2020,tsirtsis_optimal_2024} study how the agents can be \emph{incentivized} to actually improve their qualification and eventually changing their true labels. 
\citep{hu_disparate_2019,milli_social_2019,braverman_role_2020,levanon_strategic_2021} examine the \emph{societal implications} of strategic behaviors. Strategic behaviors in other machine learning tasks such as linear regression have also gained interest \citep{chen_strategyproof_2018,bechavod_gaming_2021}.

Our setting is closest to \citet{dong_strategic_2018,chen_learning_2020,ahmadi_strategic_2021}, which all consider the online strategic classification problem with continuous feature manipulation model. In particular, \citet{dong_strategic_2018} study the online strategic classification problem with interaction between the agent and the learner, where the agent sends to the learner the best response that maximizes the agent's utility function, and the learner suffers from a certain loss for misclassification. In order to derive a regret-minimizing algorithm, they identify conditions on the structure of agent's utility function under which the learner's problem is convex. Subsequently, whenever their proposed conditions are met, they propose an algorithm achieving a sublinear regret when the learner receives a mixture of strategic and non-strategic data points in an online manner. 
Our algorithms, in contrast, focus on different performance measures such as mistake bounds and margin guarantees. Moreover, we assume all data points are subject to manipulation whenever they have sufficient budget, whereas their algorithm essentially relies on the fraction of non-strategic data points. \citet{chen_learning_2020} make a more general assumption on the agent's utility function and deal with the case where the exact form of the utility function is unknown to the learner. In this setting, they propose an algorithm that relies on a special oracle, and provide a Stackelberg regret guarantee for their algorithm. Although their algorithm does not rely on the specific structure of the agent's utility, their oracle access based model and analysis rely on fundamentally different assumptions than ours. \citet{ahmadi_strategic_2021} consider two specific families of manipulation cost functions based on $\ell_2$- and weighted $\ell_1$-norms. They generalize the well-known perceptron algorithm to the strategic setting and establish finite mistake bounds. In our paper, we generalize their algorithms and results to more general norms, suggest two new algorithms with theoretical guarantees on the number of mistakes, the number of manipulations and convergence to the best margin classifier for the classification problem based on the agents' true features and show that our algorithms have superior practical performance as well.

Our new algorithms are generalizations of maximum margin classifiers to the strategic setting. For non-strategic classification, online algorithms for finding maximum margin classifiers have been studied in number of works \citep{li_relaxed_1999,gentile_new_2000,freund_large_1998,friess_kernel_1998,kivinen_online_2004,kowalczyk_maximal_2000}. These propose scalable algorithms that update incrementally like the perceptron, aimed at finding a classifier with margin guarantees analogous to the support vector machine.
We develop tools to address intricacies of the strategic setting and establish that max-margin algorithms in similar spirit can be developed in strategic setting as well even though none of our algorithms are generalization of these works.

\subsection{Outline and contributions}

\cref{sec:problem-setting} gives the formal problem setting of online strategic classification, introducing our notation and defining critical concepts such as cost of manipulation, agent response, margins and proxy data. In \cref{sec:algorithms}, we describe three algorithms for solving the online strategic classification problem with linear classifiers. The first two of our algorithms are aimed at recovering the maximum margin classifier in the presence of strategic behavior, while the third is a generalization of the strategic perceptron from \citep{ahmadi_strategic_2021}. Our theoretical contributions can be summarized as follows:
\begin{itemize}
	\item We study the concept of proxy data which is constructed from agent responses after their true label has been revealed, and show that it plays a critical role in the design of algorithms in the presence of strategic behavior. In \cref{sec:algorithms} we provide a condition (see \cref{cor:proxy-inclusion}) that ensures the proxy data to remain separable, which is critical for \cref{alg:data-driven,alg:data-driven-subgradient-averaging,alg:projected-perceptron}.
It is noteworthy to mention that even if the classification problem based on the true feature vectors may be separable, as the agent responses and consequently their proxies built are functions of the classifiers released to the agents, ensuring the separability of the proxy data is in no way obvious or guaranteed. This is thus one of the key challenges in the strategic setting, and our solution to this issue forms the crux for developing sensible classification algorithms in the strategic setting.
	
	\item \cref{alg:data-driven}, described in \cref{sec:smm}, solves an offline maximum margin problem at each iteration on the proxy data. Through tools from convex analysis (given in \cref{sec:guarantee-preliminaries}), in \cref{sec:guarantee-SMM} we provide explicit finite bounds on the number of mistakes and manipulations \cref{alg:data-driven} can make (\cref{thm:margin-best_mistake-bound,thm:margin-best_manipulation-bound}). Furthermore, under a probabilistic data generation assumption, we show that \cref{alg:data-driven} almost surely recovers the maximum margin classifier for the non-manipulated agent features (\cref{thm:data-driven-convergence}) even through it has access to only (possibly) manipulated features of the agents.
	
	\item In each iteration, \cref{alg:data-driven} solves an optimization problem of growing size. To mitigate the possibly expensive per-iteration cost of \cref{alg:data-driven}, in \cref{sec:gradient-smm} we describe \cref{alg:data-driven-subgradient-averaging}, which is an approximate version of \cref{alg:data-driven}, inspired by the Joint Estimation and Optimization (JEO) framework \citep{ahmadi_data-driven_2014,ho-nguyen_dynamic_2021}. \cref{alg:data-driven-subgradient-averaging} replaces the need to solve a full maximum margin problem at each iteration with a subgradient-type update. In \cref{sec:guarantee-SMM-grad}, we use tools from online convex optimization to give conditions under which \cref{alg:data-driven-subgradient-averaging} enjoys finite mistake and manipulation guarantees, and under the same probabilistic data generation assumption used for \cref{alg:data-driven}, we show that \cref{alg:data-driven-subgradient-averaging} also almost surely recovers the maximum margin classifier on non-manipulated agent features (\cref{thm:averaging-convergence}).

	\item \cref{alg:projected-perceptron}, described in \cref{sec:strategic-perceptron}, is a generalized version of the strategic perceptron developed by \citet{ahmadi_strategic_2021}, who also provided mistake bounds for the strategic setting when the agent manipulation costs are based on $\ell_2$- or weighted $\ell_1$-norm. (In particular, \citet{ahmadi_strategic_2021} generalizes bounds for the classical non-strategic setting \citep{rosenblatt_perceptron_1958,novikoff_convergence_1962}.) 
	In \cref{sec:guarantee-perceptron}, we use online convex optimization tools to provide mistake bounds for more general cost functions (\cref{thm:perceptron-mistake-bound,thm:perceptron-mistake-bound2}).
Moreover, our methods give improved mistake bounds in cases where the optimal linear classifier may not pass through the origin, without requiring information on the maximum margin (\cref{rem:perceptron-comparison}). 
	Our generalized version can also exploit a priori domain information of the classifier (whenever present) via a projection step to provide improved bounds in  certain settings (\cref{thm:nonnegative-perceptron-mistake-bound}).
			
	\item In \cref{sec:guarantee-examples} we examine the assumptions that are used in developing our performance guarantees. Through various examples, we show that removing any of these will result in failure of the bounds, thus demonstrating the necessity of these assumptions. 
	In particular, we show that the strategic perceptron algorithm of \citet{ahmadi_strategic_2021} as well as its generalization given in \cref{alg:projected-perceptron} may incur infinitely many manipulations, and consequently may fail to converge to the max-margin classifier.
\end{itemize}
In \cref{sec:numerical}
we test the numerical performance of our algorithms in terms of the number of mistakes, number of manipulations, and distance to the non-strategic maximum margin classifier. using real data from a lending platform as well as synthetic data.
We show that \cref{alg:data-driven} generally has superior performance on all three metrics, and \cref{alg:data-driven-subgradient-averaging} often outperforms \cref{alg:projected-perceptron}.

\subsection{Notation}

Let $[n]\coloneq\set{1,2,\cdots,n}$ for any positive integer $n\in\N$. Let $\R^d_+ \coloneq \set{x\in\R^d:x\geq0}$ denote the nonnegative orthant of the $d$-dimensional Euclidean space. We define $\argmin_{x\in\mathcal{X}}f(x)$ be the set of minimizers and $\argmax_{x\in\mathcal{X}}f(x)$ the set of maximizers of the problems respectively. For a given convex (resp.\ concave) function $f$, let $\partial f(x)$ be the subdifferential (resp.\ superdifferential) set of $f$ at $x$. For a given function $f$ differentiable at $x$, let $\nabla f(x)$ denote the gradient of $f$ evaluated at $x$. Let $\conv(\mathcal{A})$ denote the convex hull of the set $\mathcal{A}$. Let $\|\cdot\|_2$ be the $\ell_2$-norm, i.e., $\|x\|_2=\left(\sum_{i\in[d]}x_i^2\right)^{1/2}$ for any $x=(x_1,\dots,x_d)\in\R^d$. We define $B_{\|\cdot\|_2}:=\set{x\in\R^d:\|x\|_2\leq1}$ to be the unit $\ell_2$-norm ball.  

\section{Problem Setting}
\label{sec:preliminary}\label{sec:problem-setting}

In this section we formalize the problem definition for strategic classification. Each agent comprises of a data pair $(A,\lbl(A))$, where $A \in \cA \subseteq \R^d$ is a feature vector and $\lbl(A) \in \{\pm 1\}$ is the associated label. Here, $\lbl(\cdot)$ is a function mapping feature vectors in $\R^d$ to $\pm 1$, and $\cA$ is the space of all possible feature vectors. 
For notational convenience, we also define 
	the sets of positively and negatively labeled agents, respectively, as
	\begin{align}\label{eq:pos-neg-sets}
		\cA^+ := \left\{ A \in \cA : \lbl(A) = +1 \right\}, \quad
		\cA^- := \left\{ A \in \cA : \lbl(A) = -1 \right\}. 
	\end{align}
We focus on the case of continuous feature vectors and consider the set of \emph{linear classifiers} parametrized by $(y,b) \in \bbR^d \times \bbR$: given a feature vector $x \in \bbR^d$, the \emph{predicted label} for $\lbl(x)$ is defined to be
\begin{align*}
	\phlbl(x,y,b) := \sign(y^\top x + b) \in \{\pm 1\},
\end{align*}
where we follow the convention of $\sign(0) := +1$.

\subsection{Manipulation and agent responses}

In the strategic setting, labels of $+1$ are assumed to be more desirable than $-1$. Therefore, an agent $(A,\lbl(A))$ may manipulate their own \emph{true} feature vector $A$ to change their predicted label, as long as it does not cost them too much. More precisely, we define the \emph{cost} of manipulating a feature vector $A$ to $x$ as $\cost(A,x)$, with the properties that $\cost(A,x) \geq 0$, and $\cost(A,A) = 0$. Given the classifier $x \mapsto \sign(y^\top x + b)$, the agent then chooses a (possibly new) \emph{manipulated} feature vector by solving the following optimization problem:
\begin{align*}
	\argmax_{x\in\R^d}\left\{\sign\left(y^\top x + b\right) - \cost\left(A, x\right)\right\}.
\end{align*}
We next derive the precise form of the agent response.
Given a vector $x \in \bbR^d$, an agent will be incentivized to move from $A$ to $x$ if and only if $\sign(y^\top x + b) - \cost(A,x) \geq \sign(y^\top A + b)$. Clearly, if $y^\top A + b \geq 0$, then the agent will never move, as $A$ is optimal, so we consider the case of $y^\top A + b < 0$ now. If $y^\top x + b < 0$ also, then once again the agent is never incentivized to move. Therefore, we consider the case when $y^\top x + b \geq 0 > y^\top A + b$, which means the agent moves from $A$ to $x$ if and only if $\cost(A,x) \leq 2$. To find the most cost effective move, the agent solves
\begin{align}\label{eq:cost-to-move}
	\bar{c}(A,y,b) := \min_{x \in \bbR^d} \set{ \cost(A,x) : y^\top x + b \geq 0 }, \  \bar{r}(A,y,b) \in \argmin_{x \in \bbR^d} \set{ \cost(A,x) : y^\top x + b \geq 0 }.
\end{align}
Then, the agent changes from $A$ to $\bar{r}(A,y,b)$ if and only if $\bar{c}(A,y,b) \leq 2$. Note that, for certain cost functions, it is possible that there may be several choices for $\bar{r}(A,y,b)$. In such cases we assume that the agent and learner agree on the selection. We will formalize this in \cref{assum:unique-direction} below.
Throughout this paper, we make the following structural assumption on the cost function.
\begin{assumption}[Cost function]\label{assum:cost-norm}
	The cost function takes the form $\cost(A,x) := c\|x - A\|$ where $\|\cdot\|$ is a norm on $\bbR^d$ and $c>0$ is a parameter. Moreover, this structural form and the value of $c$ is known to both the learner and the agents.
\end{assumption}
In contrast to previous works, we will study various different norms and associated guarantees. Under \cref{assum:cost-norm}, we have via basic convex analysis
\begin{align*}
	&\argmin_{x \in \bbR^d} \set{ \cost(A,x) : y^\top x + b \geq 0 }\\
&= \begin{cases}
	\set{A}, &y=0 \text{ and } b \geq 0\\
	\emptyset, &y=0 \text{ and } b < 0\\ \set{A + \max\left\{ 0, -\frac{y^\top A + b}{\|y\|_*} \right\} v : v \in \argmax_{w : \|w\| \leq 1} y^\top w }, &y \neq 0.
\end{cases}
\end{align*}
In order to define a unique response $\bar{r}(A,y,b)$ and to  prevent ambiguities, we make the following simplifying assumption on the \emph{manipulation direction}:
\begin{assumption}[Manipulation direction]\label{assum:unique-direction}
	Given $y \in \bbR^d$, assume that both agent and learner have knowledge of a selection $v(y)$ from $\argmax_{w : \|w\| \leq 1} y^\top w$. \par
	Note that $y^\top v(y) = \|y\|_*$, $v(y) \in \partial \|y\|_*$, and when $y \neq 0$, $\|v(y)\| = 1$. When $y = 0$, we define $v(y) := 0$.
\end{assumption}
We now derive the agent responses under two cases: $y \neq 0$ and $y = 0$.

\subsubsection{The case \texorpdfstring{$y \neq 0$}{y≠0}.} Under \cref{assum:unique-direction}, when $y \neq 0$, we define
\[ \bar{r}(A,y,b) := A + \max\left\{ 0, -\frac{y^\top A + b}{\|y\|_*} \right\} v(y), \quad \bar{c}(A,y,b) = c \cdot \max\left\{ 0, -\frac{y^\top A + b}{\|y\|_*} \right\}. \]
Recalling that $A$ will move if and only if $\bar{c}(A,y,b) \leq 2$, we deduce that
\begin{align*}
	A \text{ moves to } \bar{r}(A,y,b) &\iff  y^\top A + b \geq - \frac{2 \|y\|_*}{c}\\
	A \neq \bar{r}(A,y,b) &\iff y^\top A + b < 0.
\end{align*}
In other words, given the classifier $x \mapsto \sign(y^\top x + b)$, the agent with true feature vector $A$ will present a \emph{different} feature vector $\bar{r}(A,y,b)$ if and only if $- \frac{2 \|y\|_*}{c} \leq y^\top A + b < 0$. Based on this, we denote the agent response (manipulated data) to be
\begin{align}\label{eq:agent-response}
	\tilde{r}(A,y,b) = \begin{cases}
	\bar{r}(A,y,b), & \text{if } \bar{c}(A,y,b) \leq 2\\
	A, & \text{if } \bar{c}(A,y,b) > 2
\end{cases} = \begin{cases}
	A - \frac{y^\top A + b}{\|y\|_*} \cdot v(y), & \text{if } - \frac{2 \|y\|_*}{c} \leq y^\top A + b < 0\\
	A, & \text{otherwise}.
\end{cases}
\end{align}
Note that based on these definitions, we assume that when the agent is indifferent between manipulating or not, i.e., $\bar{c}(A,y,b)=2$, the agent will always manipulate.

\subsubsection{The case \texorpdfstring{$y=0$}{y=0}.} When $y=0$, we have
\[ \bar{c}(A,y,b) = \begin{cases} 0, &\text{if } b \geq 0\\ +\infty, &\text{if }b < 0. \end{cases} \]
Thus, when $b\ge0$, we have $\bar{c}(A,y,b)=0$ and $\set{A} = \argmin_{x \in \bbR^d} \set{\cost(A,x) : y^\top x + b}$, so the optimal response of the agent is to present their true feature vector $A$. 
In the case of $b<0$, $\bar{c}(A,y,b)=+\infty > 2$, so the agent does not manipulate the feature vector. In summary, when $y = 0$, regardless of the sign of $b$, we have $\tilde{r}(A,y,b) = A$. Since $-2\|y\|_*/c = 0$ when $y=0$, the first case of \eqref{eq:agent-response} can never be satisfied, so \eqref{eq:agent-response} is also correct for $y=0$.

\subsection{Margins and truthfulness}

We make the assumption that $\cA^+$ and $\cA^-$ from \eqref{eq:pos-neg-sets} are linearly separable. The distance of feature vector $A$ from the misclassification region $\{ x \in \bbR^d : \lbl(A) (y^\top x + b) \leq 0 \}$, as measured by the norm $\|\cdot\|$, is given by
\[ \min_x \left\{ \|x-A\| : x \in \bbR^d, \lbl(A) (y^\top x + b) \leq 0 \right\} = \max\left\{0,\, \lbl(A) \cdot \frac{y^\top A + b}{\|y\|_*} \right\}. \]
Under linear separability (and mild topological conditions such as closedness of $\cA$), we will satisfy the following positive margin assumption.
\begin{assumption}[Margin]\label{assum:data-driven-deterministic}\label{assum:separable}\label{assum:margin}
	The following optimization problem 
	\begin{align}\label{eq:data-driven-best}
		d_* := \max_{y \neq 0, b \in \bbR} \min_{A \in \cA} \left\{ \lbl(A) \cdot \frac{y^\top A + b}{\|y\|_*} \right\}
	\end{align}
	has an optimum solution $y_* \in \bbR^d \setminus \{0\}$, $b_* \in \bbR$ and an optimal value of $d_* > 0$.
\end{assumption}
This is the standard definition of maximum margin classifiers for non-strategic settings \citep[Chapter 5.2]{mohri_foundations_2012}. \cref{assum:margin} implies 
\[ \lbl(A) \cdot \frac{y_*^\top A + {b_*}}{\|y_*\|_*} \geq d_* \text{ for all } A \in \cA. \]
This means that, in the non-strategic setting, when the feature vector $A$ is presented to the classifier $x\mapsto\sign(y_*^\top x + b_*)$, it will be classified correctly simply by predicting $\sign(y_*^\top A + b_*)$.

However, in the strategic setting, even a classifier $x \mapsto \sign(y^\top x + b)$ which correctly predicts the labels of the non-manipulated feature vectors (i.e., $\phlbl(A,y,b) := \sign(y^\top A + b) = \lbl(A)$) it may still not predict the labels of the manipulated feature vectors correctly (i.e., $\phlbl(\tilde{r}(A,y,b),y,b) = \sign(y^\top \tilde{r}(A,y,b) + b) \neq \lbl(A)$). This is simply because when manipulation occurs, i.e., $\tilde{r}(A,y,b) \neq A$, we also have $-2\|y\|_*/c \leq y^\top A + b < 0$, thus
\[ \lbl(A) = \sign(y^\top A + b) = -1 \quad \text{and} \quad \sign(y^\top \tilde{r}(A,y,b) + b) = \sign(0) = +1. \] 
A simple adjustment of the offset, first proposed by \citet{ahmadi_strategic_2021}, provides a remedy to this issue. 
\begin{lemma}\label{lem:offset-label}
Consider a classifier $x\mapsto\sign\left(y^\top x+b-\frac{2\|y\|_*}{c}\right)$. Then, the agent response is
\begin{align}\label{eq:agent-response-offset}
	\tilde{r}\left(A,y,b - \tfrac{2\|y\|_*}{c} \right) = \begin{cases}
		A + \left( \frac{2}{c} - \frac{y^\top A + b}{\|y\|_*} \right) v(y), &\text{if }~ 0 \leq y^\top A + b < \frac{2 \|y\|_*}{c},\\
		A, & \text{otherwise}.
	\end{cases}
\end{align}
Let $(y,b)$ be such that
for every
$A \in \cA$ we have $\phlbl(A,y,b) = \lbl(A)$, i.e., the classifier $x\mapsto\sign\left(y^\top x+b\right)$ is correct on the non-manipulated features.
Then, 
\[  \phlbl\left( \tilde{r}\left(A,y,b - \tfrac{2\|y\|_*}{c} \right), y, b-\tfrac{2\|y\|_*}{c} \right) = \lbl(A) , \]
that is, the classifier $x\mapsto\sign\left(y^\top x+b-\frac{2\|y\|_*}{c}\right)$ is correct on the manipulated features.
\end{lemma}
\begin{proof}[Proof of \cref{lem:offset-label}]
Plugging in $(y,b-2\|y\|_*/c)$ into \eqref{eq:agent-response} leads to 
\[ \tilde{r}\left(A,y,b - \tfrac{2\|y\|_*}{c} \right) = \begin{cases}
A + \left( \frac{2}{c} - \frac{y^\top A + b}{\|y\|_*} \right) v(y), & \text{if } 0 \leq y^\top A + b < \frac{2 \|y\|_*}{c},\\
A, & \text{otherwise},
\end{cases} \]
thus by $y^\top v(y)=\|y\|_*$ we arrive at 
\begin{align*} 
\phlbl\left( \tilde{r}\left(A,y,b - \tfrac{2\|y\|_*}{c} \right), y, b-\tfrac{2\|y\|_*}{c} \right) &= 
\sign\left(y^\top \tilde{r}\left(A,y,b - \tfrac{2\|y\|_*}{c} \right) + b - \tfrac{2\|y\|_*}{c} \right)\\
&= \begin{cases}
0, & \text{if } 0 \leq y^\top A + b < \frac{2 \|y\|_*}{c}\\
\sign\left(y^\top A + b - \frac{2\|y\|_*}{c} \right), & \text{otherwise}.
\end{cases} 
\end{align*}
When $y^\top A + b < 0$, i.e., $\lbl(A) = -1$, we have
\[ \phlbl\left( \tilde{r}\left(A,y,b - \tfrac{2\|y\|_*}{c} \right), y, b-\tfrac{2\|y\|_*}{c} \right) 
=\sign\left(y^\top A + b - \tfrac{2\|y\|_*}{c} \right)
= -1 \]
as required. When $0 \leq y^\top A + b < 2\|y\|_*/c$, we have $\lbl(A) = +1$ and
\[ \phlbl\left( \tilde{r}\left(A,y,b - \tfrac{2\|y\|_*}{c} \right), y, b-\tfrac{2\|y\|_*}{c} \right) = \sign\left(0 \right) = +1 \]
as required. When $y^\top A + b \geq 2\|y\|_*/c$, we have
\[ \phlbl\left( \tilde{r}\left(A,y,b - \tfrac{2\|y\|_*}{c} \right), y, b-\tfrac{2\|y\|_*}{c} \right) 
=\sign\left(y^\top A + b - \tfrac{2\|y\|_*}{c} \right)
= +1 \]
as required.
\end{proof}

\cref{lem:offset-label} states that given any classifier $x \mapsto \sign(y^\top x + b)$ that classifies all \emph{non-manipulated} features correctly, we can create a classifier $x \mapsto \sign(y^\top x + b - 2\|y\|_*/c)$ that classifies all \emph{manipulated} features correctly.

In a strategic setting, in addition to correct classification, another property of interest is \emph{truthfulness}. 
\begin{definition}\label{def:truthful}
We say that an agent with feature vector $A$ is \emph{truthful for the classifier $x\mapsto\sign\left(y^\top x+b-\frac{2\|y\|_*}{c}\right)$}, if there is no manipulation, i.e., $\tilde{r}(A,y,b-2\|y\|_*/c) = A$.
\end{definition}
In the strategic setting, we wish to find a classifier that not only correctly classifies the points but also incentivizes many, if not all, agents to report truthful features. 

It turns out that a classifier promoting truthfulness is intrinsically connected to its margin. Thus, the magnitude of $\frac{y^\top A + b}{\|y\|_*}$ is of interest to us, and this motivates us to explore methods which recover the maximum margin classifier $(y_*,b_*)$ that solves \eqref{eq:data-driven-best}. The following example demonstrates the benefit of obtaining a maximum margin classifier.
\begin{example}\label{ex:truthful-max-margin}
Consider $\cA = \cA^+ \cup \cA^- \subset \bbR^2$, where
\begin{align*}
	&& \cA^+ &= \{(-1,1),(0, 1), (1,1), (2,1)\}, \quad &&\text{where }\lbl(A) = +1 \ \forall A \in \cA^+,&&&\\
	&& \cA^- &= \{(-1,-1),(1,-1)\}, \quad &&\text{where }\lbl(A) = -1 \ \forall A \in \cA^-.&&& 
\end{align*}
Then, \cref{assum:margin} is satisfied with $y_* = (0,1)$, $b_* = 0$ and $d_* = 1$.
Suppose $\|\cdot\| = \|\cdot\|_* = \|\cdot\|_2$ with $c=4$ (so $2/c = 1/2$) in \cref{assum:cost-norm}. 
Note that we have $y_*^\top A+b_*=-1$ for all $A \in \cA^-$ and $y_*^\top A+b_*=+1>{2\|y_*\|_*/c}=1/2$ for all $A\in\cA^+$. Thus, from~\cref{lem:offset-label}, we deduce that $\tilde{r}(A,y_*,b_*-2\|y_*\|_*/c)=A$ for all $A\in\cA$ and so every agent is truthful for the classifier $x\mapsto\sign\left(y_*^\top x+b_*-\frac{2\|y_*\|_*}{c}\right)$.

Now consider the classifier $x\mapsto\sign\left(\bar{y}^\top x+\bar{b}-\frac{2\|\bar{y}\|_*}{c}\right)$, where $\bar y = (1,2)$, $\bar b=0$. This classifier  also correctly classifies the non-manipulated data. We can check that $\frac{\bar{y}^\top A + \bar{b}}{\|\bar{y}\|_*} > 2/c$ for all $A \in \cA^+\setminus\{(-1,1)\}$ and $\bar{y}^\top A + \bar{b} < 0$ for $A \in \cA^-$. Therefore, by~\cref{lem:offset-label}, we have $\tilde{r}(A,\bar{y},\bar{b}-2\|\bar{y}\|_*/c)=A$ for all $A\in\cA\setminus\{(-1,1)\}$ and thus every agent except $A = (-1,1)$ is truthful for the classifier $x\mapsto\sign\left(\bar{y}^\top x+\bar{b}-\frac{2\|\bar{y}\|_*}{c}\right)$.
In fact, $\frac{\bar{y}^\top (-1,1) + \bar{b}}{\|\bar{y}\|_*} = \frac{1}{\sqrt{5}} < \frac{2}{c} = \frac{1}{2}$, so
\[ \tilde{r}\left((-1,1),\bar{y}, -\frac{2\|\bar{y}\|_*}{c} \right)
= (-1,1)+\left(\frac{1}{2}-\frac{1}{\sqrt{5}}\right)\frac{(1,2)}{\sqrt{5}} = \left(-\frac{6}{5}+\frac{\sqrt{5}}{10},\frac{3}{5}+\frac{\sqrt{5}}{5}\right) \neq (-1,1). \]
Thus, the agent $A = (-1,1)$ is never truthful with this classifier. \hfill$\blacksquare$
\end{example}

\subsection{Proxy data}

\cref{ex:truthful-max-margin} motivates us to find $(y_*,b_*)$ which solves \eqref{eq:data-driven-best}, and then predict labels of features $x$ according to $\phlbl(x,y_*,b-2\|y_*\|_*/c)$. This will ensure that all agents are classified correctly, and will encourage truthfulness of the agents. 
We now consider how to algorithmically find $(y_*,b_*)$ when we have access to only the agent responses $\tilde{r}(A,y,b-2\|y\|_*/c)$. The key insight lies in the concept of \emph{proxy data}, which we adapt from \citet{ahmadi_strategic_2021}. The proxy data are estimates of $A$ computed from the observations $r(A,y,b-2\|y\|_*/c)$ after also observing the true label $\lbl(A)$. As we will show in \cref{cor:proxy-inclusion} below, the proxy data will be separable under certain conditions.

As before, suppose the learner presents the classifier $x \mapsto \sign(y^\top x + b - 2\|y\|_*/c)$ to the agent. In response the agent with true features $A$ reveals manipulated features $\tilde{r}(A,y,b-2\|y\|_*/c)$. From \eqref{eq:agent-response-offset}, we compute
\begin{align}
	\frac{y^\top \tilde{r}(A,y,b-2\|y\|_*/c) + b}{\|y\|_*} &= \begin{cases}
		\frac{2}{c}, & \text{if } 0 \leq \frac{y^\top A+b}{\|y\|_*} < \frac{2}{c}\\
		\frac{y^\top A+b}{\|y\|_*}, &\text{otherwise}.
	\end{cases}\label{eq:y-transpose-response}
\end{align}
This shows that whenever we see that $\frac{y^\top \tilde{r}(A,y,b-2\|y\|_*/c)+b}{\|y\|_*} = 2/c$, we may reasonably guess that the feature vector has likely been manipulated. (The exception, of course, is when $(y^\top A+b)/\|y\|_* = 2/c$, in which case $A = \tilde{r}(A,y,b-2\|y\|_*/c)$.) Furthermore, the predicted label is $\phlbl(\tilde{r}(A,y,b-2\|y\|_*/c),y,b-2\|y\|_*/c) = \sign(0)=+1$. When $\lbl(A) = +1$, the manipulation does \emph{not} make the prediction incorrect, but when $\lbl(A)=-1$, the prediction is incorrect.
In order to correct for the manipulation in the case when $\lbl(A)=-1$ and given the classifier $x \mapsto \sign(y^\top x + b)$ to the agent, under \cref{assum:cost-norm,assum:unique-direction}, we define the \emph{proxy data} $\tilde{s}(A, y, b)$ as follows: 
\begin{align}\label{eq:proxy-0}
	\tilde{s}(A, y,b) := \begin{cases}
		\tilde{r}(A, y, b) - \frac{2}{c} v(y), &\text{if }  \frac{y^\top \tilde{r}(A, y,b) + b}{\|y\|_*} = 0 \text{ and } \lbl(A) = -1, \\
		\tilde{r}(A, y, b), & \text{otherwise}. 
	\end{cases}
\end{align}
Combining \eqref{eq:agent-response-offset} and \eqref{eq:proxy-0}, when we include the $-2\|y\|_*/c$ offset, the proxy data becomes
\begin{align}\label{eq:proxy-offset}
	\tilde{s}\left(A, y,b-\frac{2\|y\|_*}{c}\right) := \begin{cases}
		A - \frac{y^\top A + b}{\|y\|_*} v(y), & \text{if } 0 \leq \frac{y^\top A + b}{\|y\|_*} < \frac{2}{c} \text{ and } \lbl(A) = -1, \\
		A + \left( \frac{2}{c} - \frac{y^\top A + b}{\|y\|_*}\right) v(y), & \text{if } 0 \leq \frac{y^\top A + b}{\|y\|_*} < \frac{2}{c} \text{ and } \lbl(A) = +1, \\
		A, & \text{otherwise}. 
	\end{cases}
\end{align}

In words, the proxy data is obtained by moving points that are on the decision boundary $\frac{y^\top x+b}{\|y\|_*} = \frac{2}{c}$ and with true label $-1$ back to the hyperplane $\frac{y^\top x+b}{\|y\|_*} = 0$, via the direction $-v(y)$.

\subsection{Online strategic classification problem}\label{sec:online_setting}

Below we summarize the terms defined so far, and introduce simplified notation for the agent response and proxy data when the agents are given the classifier $x \mapsto \sign(y^\top x + b - 2\|y\|_*/c)$ (to ease our notation, from now on whenever we write the classifier $(y,b)$, we will be referring to the classifier $x \mapsto \sign(y^\top x + b - 2\|y\|_*/c)$ with the additional offset of $ - 2\|y\|_*/c$):
\begin{subequations}\label{eq:notation}
	\begin{align}
		&& \text{predicted label:} & & \plbl(x,y,b) &:= \phlbl\left(x,y,b-\frac{2\|y\|_*}{c}\right) = \sign\left( y^\top x+b - \frac{2\|y\|_*}{c} \right) &\\
		&& \text{true label:} & & \lbl(A) &\in \{\pm 1\} &\\
		&& \text{direction:} & & v(y) &\in \partial \|y\|_* = \argmax_{x\in\R^d} \left\{ y^\top x : \|x\| \leq 1 \right\} & \label{eq:manipulation-direction} \\ 
		&& \text{response:} & & r(A,y,b) &:= \tilde{r}\left(A,y,b-\frac{2\|y\|_*}{c} \right)& \nonumber\\
		&& & & &= \begin{cases}
			A + \left( \frac{2}{c} - \frac{y^\top A+b}{\|y\|_*} \right) v(y), &\text{if } 0 \leq \frac{y^\top A+b}{\|y\|_*} < \frac{2}{c} \\
			A, & \text{otherwise}\\
		\end{cases} & \label{eq:manipulated}\\
		&& \text{proxy data:} & & s(A,y,b) &:= \tilde{s}\left(A,y,b-\frac{2\|y\|_*}{c} \right)& \nonumber\\
		&& & & &= \begin{cases}
			A - \frac{y^\top A + b}{\|y\|_*} v(y), & \text{if } 0 \leq \frac{y^\top A + b}{\|y\|_*} < \frac{2}{c} \text{ and } \lbl(A) = -1, \\
			A + \left( \frac{2}{c} - \frac{y^\top A + b}{\|y\|_*}\right) v(y), &  \text{if }  0 \leq \frac{y^\top A + b}{\|y\|_*} < \frac{2}{c} \text{ and } \lbl(A) = +1, \\
			A, & \text{otherwise}. 
		\end{cases} & \label{eq:proxy}
	\end{align}
\end{subequations}
The formulas in \eqref{eq:notation} are for the case $y \neq 0$. However, when $y=0$, it is easy to see that we have $r(A,y,b) = s(A,y,b) = A$.
Also, \eqref{eq:proxy} implies that for any classifier $(y,b)$ we have $s(A,y,b)=A+\lbl(A)\cdot\alpha(A,y,b)\cdot v(y)$ for some $\alpha(A,y,b)\in[0,2/c]$.

Finally, we define the online strategic classification problem.
\begin{definition}[Online strategic classification problem]\label{def:online-strategic-problem}
	At time $t=1,\ldots,T$, the following sequence of events happens:
	\begin{itemize}
		\item the learner picks $(y_t,b_t) \in \R^d \times \bbR$ and reveals classifier $x \mapsto \plbl(x,y_t,b_t)$ to the agent;
		\item an agent with true feature vector-label pair $(A_t, \lbl(A_t)) \in \cA \times \{\pm 1\}$ reveals the response $r(A_t,y_t,b_t)$;
		\item the learner classifies the agent according to $\plbl(r(A_t,y_t,b_t),y_t,b_t)$;
		\item then the agent reveals $\lbl(A_t)$.
	\end{itemize}
\end{definition}

We illustrate the response $r(A,y,b)$ and the proxy data $s(A,y,b)$ for the data points $A$  in \cref{fig:proxy-data}.

\tikzset{
	data point/.style={circle,inner sep=.7pt,text=white},
	negative/.style={data point,fill=red!40!gray}, 
	positive/.style={data point,fill=blue!40!gray},
	unmanipulated negative/.style={
		negative,fill=none,draw=red!40!gray,text=red!40!gray,densely dashed,
		inner sep=.3pt,line width=1pt
	},
	unmanipulated positive/.style={
		positive,fill=none,draw=blue!40!gray,text=blue!40!gray,densely dashed,
		inner sep=.3pt,line width=1pt
	},
	manipulated negative/.style={
		negative,draw=red!40!gray,fill=red!40!gray,fill opacity=.5,text opacity=1,
		inner sep=.3pt,line width=1pt,
},
	manipulated positive/.style={
		positive,draw=blue!40!gray,fill=blue!40!gray,fill opacity=.5,text opacity=1,
		inner sep=.3pt,line width=1pt,
},
	proxy/.style={
		data point,fill=gray 
	}
}

\begin{figure}[ht!]
	\begin{center}
	\scalebox{.9}{\begin{tikzpicture}[scale=.8]
			\draw[draw=blue!40!gray,thick,opacity=.6] (1.9,-3.3) -- (-1.84,3.3);
			\draw[draw=blue!40!gray,thick,opacity=.6,-latex] (-1.51,2.71) -- (-1,3);
			\draw[draw=none,fill=blue!60!gray,opacity=.2] (1.9,-3.3) -- (5,-3.3) -- (5,3.3) -- (-1.84,3.3) -- cycle;
			\draw[draw=none,fill=red!60!gray,opacity=.2] (1.9,-3.3) -- (-5,-3.3) -- (-5,3.3) -- (-1.84,3.3) -- cycle;
\draw[draw=gray,thick,dashed] (-1.8,-3.3) -- (2.2,3.3);
			\draw[draw=gray,thick,-latex] (2,2.97) -- (2.55,2.64);
\draw[draw=none,fill=gray,opacity=.2] (-3.8,-3.3) -- (0.2,3.3) -- (2.2,3.3) -- (-1.8,-3.3) -- (-3.8,-3.3) -- cycle;
			\foreach \Point in {(-2.64,-0.5), (-4.13,-1.6), (-4.22,0.56), (-3,2.25), (-2.13,0.85), (0.25,-2.2)}{
				\node[negative] at \Point {$\vphantom{+}-$};
			}
			\foreach \Point in {(1.03,0.06), (2.64,0.82), (2.48,-0.58), (2.91,2.06), (3.99,1.27), (3.97,-0.48), (3.64,-1.73)}{
				\node[positive] at \Point {$+$};
			}
			\foreach \name/\orig/\mani in {1/{(-2.76,-2.34)}/{(-1.675,-3)}, 2/{(-1.73,-1.57)}/{(-1.06,-1.98)}, 3/{(-1,-0)}/{(-0.165,-0.505)}}{
				\node[unmanipulated negative] (o\name) at \orig {$\vphantom{+}-$};
				\node[manipulated negative] (\name) at \mani {$\vphantom{+}-$};
				\draw[-stealth,gray,thick] (o\name) -- (\name);
			}
			\foreach \name/\orig/\mani in {4/{(0.67,2.34)}/{(1.32,1.95)}}{
				\node[unmanipulated positive] (o\name) at \orig {$+$};
				\node[manipulated positive] (\name) at \mani {$+$};
				\draw[-stealth,gray,thick] (o\name) -- (\name);
			}
			\node[yshift=.3cm,xshift=-.2cm,text=blue!40!gray] at (-1.84,3.3) {$y_*^\top x+b_*\geq0$};
			\node[yshift=.3cm,xshift=.2cm,text=gray] at (2.2,3.3) {$y^\top x+b\geq0$};
			\node[yshift=0cm,xshift=-.4cm,text=red!40!gray] at (-2.76,-2.34) {$A$};
			\node[yshift=0cm,xshift=.9cm,text=red!40!gray] at (-1.675,-3) {$r(A,y,b)$};
\end{tikzpicture}
\quad
\begin{tikzpicture}[scale=.8]
			\draw[draw=blue!40!gray,thick,opacity=.6] (1.9,-3.3) -- (-1.84,3.3);
			\draw[draw=blue!40!gray,thick,opacity=.6,-latex] (-1.51,2.71) -- (-1,3);
			\draw[draw=none,fill=blue!60!gray,opacity=.2] (1.9,-3.3) -- (5,-3.3) -- (5,3.3) -- (-1.84,3.3) -- cycle;
			\draw[draw=none,fill=red!60!gray,opacity=.2] (1.9,-3.3) -- (-5,-3.3) -- (-5,3.3) -- (-1.84,3.3) -- cycle;
\draw[draw=gray,thick,dashed] (-1.8,-3.3) -- (2.2,3.3);
			\draw[draw=gray,thick,-latex] (2,2.97) -- (2.55,2.64);
\draw[draw=none,fill=gray,opacity=.2] (-3.8,-3.3) -- (0.2,3.3) -- (2.2,3.3) -- (-1.8,-3.3) -- (-3.8,-3.3) -- cycle;
			\foreach \Point in {(-2.64,-0.5), (-4.13,-1.6), (-4.22,0.56), (-3,2.25), (-2.13,0.85), (0.25,-2.2)}{
				\node[negative] at \Point {$\vphantom{+}-$};
			}
			\foreach \Point in {(1.03,0.06), (2.64,0.82), (2.48,-0.58), (2.91,2.06), (3.99,1.27), (3.97,-0.48), (3.64,-1.73)}{
				\node[positive] at \Point {$+$};
			}
			\foreach \name/\proxy/\mani in {1/{(-3.095,-2.14)}/{(-1.675,-3)}, 2/{(-2.48,-1.12)}/{(-1.06,-1.98)}, 3/{(-1.585,0.354)}/{(-0.165,-0.505)}}{
				\node[manipulated negative] (\name) at \mani {$\vphantom{+}-$};
				\node[proxy] (p\name) at \proxy {$\vphantom{+}-$};
				\draw[-stealth,gray,thick] (\name) -- (p\name);
			}
			\foreach \name/\proxy/\mani in {4/{(-0.095,2.81)}/{(1.32,1.95)}}{
				\node[proxy] (\name) at \mani {$+$};
}
			\node[yshift=.3cm,xshift=-.2cm,text=blue!40!gray] at (-1.84,3.3) {$y_*^\top x+b_*\geq0$};
			\node[yshift=.3cm,xshift=.2cm,text=gray] at (2.2,3.3) {$y^\top x+b\geq0$};
			\node[yshift=0cm,xshift=.9cm,text=red!40!gray] at (-1.675,-3) {$r(A,y,b)$};
			\node[yshift=0cm,xshift=-.9cm,text=red!40!gray] at (-3.095,-2.14) {$s(A,y,b)$};
\end{tikzpicture}
	}
	\end{center}
	\caption{
		Illustration for the response $r(A,y,b)$ and proxy $s(A,y,b)$. 
In both figures, the solid blue line represents the maximum margin hyperplane $y_*^\top x + b_* = 0$
		which separates the positive and negative agents. 
		The blue points have label $+1$ while red points have label $-1$. The dashed gray line represents the hyperplane $y^\top x + b = \frac2c \|y\|_*$, which is the decision boundary of the classifier $x\mapsto \sign(y^\top x + b - \frac2c \|y\|_*) = \plbl(x,y,b)$ being presented to the agents.
		The gray shaded region represents the $0 \leq y^\top x + b < \frac2c \|y\|_*$ region, i.e., it is the region in which the agents will manipulate their feature vectors.
In the left figure, solid-colored and dashed points denote the true feature vectors $A$ of agents. The solid points fall outside the gray region, thus those agents will not manipulate their feature vectors. Dashed points fall inside the gray region, thus the agents in this region will manipulate to the light-shaded colored points (i.e., $r(A,y,b)$) which lie on the dashed gray line $y^\top x + b = \frac2c \|y\|_*$. Notice that the true label of the three shaded red points is $-1$, however by manipulating to the dashed gray line, they will instead be classified with an incorrect label of $+1$. On the other hand, the dashed blue point has true label $+1$, and if it does not manipulate, it would be given an incorrect label of $-1$; however, by manipulating to the light-shaded blue point, it is classified with its correct $+1$ label.
In the right figure, the solid and light-shaded colored points show the agents' response vectors; note that this is all that is available to the learner, as the true feature vectors (dashed points on the left figure) are not revealed. The learner then constructs the proxy data points, which are the solid gray points. Notice that the positive manipulated point (on the blue side) is the same as the response vector, but the three manipulated negative points (on the red side) are shifted back to the line $y^\top x + b = 0$.
}

\label{fig:proxy-data}
\end{figure}
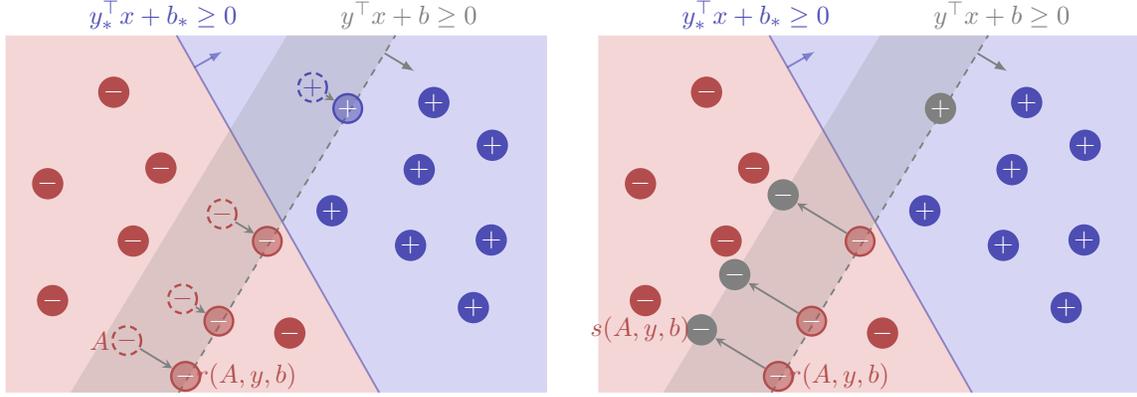

In order to provide guarantees for algorithms solving this strategic classification  problem, we make a mild boundedness assumption on the feature vectors.
\begin{assumption}[Boundedness]\label{assum:bounded}
	The support set for features is bounded, i.e., $\sup_{A \in \cA} \|A\|_2 < \infty$.
\end{assumption}

The following quantities appear in our analysis (which are all finite under \cref{assum:bounded}):
\begin{subequations}\label{eq:proxy-bounded}
\begin{align}
C_{\|\cdot\|} &\coloneqq\max_{y \in \bbR^d}\|v(y)\|_2, \quad \text{where $\|\cdot\|$ is the norm in \cref{assum:cost-norm}}\\
D &:= \sup_{A \in \cA} \|A\|_2, \quad \tilD \coloneqq D+\frac{2}{c}C_{\|\cdot\|}\\
D^{\pm} &:= \frac{1}{2} \sup_{A,A' \in \cA} \|A - A'\|_2, \quad \tilD^{\pm} := D^{\pm} + \frac{2}{c} C_{\|\cdot\|}\\
D^{+} &:= \frac{1}{2} \sup_{A,A' \in \cA^+} \|A - A'\|_2, \quad \tilD^- := D^{-} + \frac{2}{c} C_{\|\cdot\|}\\
D^{-} &:= \frac{1}{2} \sup_{A,A' \in \cA^-} \|A - A'\|_2, \quad \tilD^+ := D^{+} + \frac{2}{c} C_{\|\cdot\|}, \quad \barD := \max\{\widetilde{D}^+, \widetilde{D}^-\}.
\end{align}
\end{subequations}
Notice that $D \geq \max\{D^{\pm},D^+,D^-\}$ and $\tilD \geq \max\{\tilD^{\pm},\tilD^+,\tilD^-\}$.
From \cref{eq:proxy,assum:bounded}, we have that for any two classifiers $(y,b)$ and $(y',b')$, the following bounds hold for the proxy data:
$\sup_{A\in\cA}\|s(A,y,b)\|_2 \leq \tilD$, $\sup_{A,A'\in\cA} \|s(A,y,b) - s(A',y',b')\|_2 \leq \tilD^{\pm}$, $\sup_{A,A'\in\cA^+} \|s(A,y,b) - s(A',y',b')\|_2 \leq \tilD^+$, and $\sup_{A,A'\in\cA^-} \|s(A,y,b) - s(A',y',b')\|_2 \leq \tilD^-$.

\section{Algorithms}
\label{sec:algorithms}

We now introduce algorithms for the online strategic classification problem (\cref{def:online-strategic-problem}), all of which make use of the proxy data \eqref{eq:proxy}. Before doing so, we provide some useful properties of the proxy data that the algorithms will utilize.

We start with an important observation that the proxy data captures the ``correctness'' of a given classifier: a response $r(A,y,b)$ is misclassified by $x \mapsto \sign(y^\top x + b - 2\|y\|_*/c) = \plbl(x,y,b)$ if and only if its proxy $s(A,y, b)$ is misclassified by $x \mapsto \sign(y^\top x + b) = \phlbl(x,y,b)$.
\begin{lemma}\label{lem:classifier-proxy-inner-product}
Given $(y,b)$, agent $A \in \cA$ and response $r(A,y,b)$, consider a prediction of the agent's label $\plbl(r(A,y,b),y,b)$. 
	If the prediction is correct, i.e., $\plbl(r(A,y,b),y,b)=\lbl(A)$, then we have $\lbl(A) [y^\top s(A,y,b)+b] \geq 0$; otherwise, $\lbl(A) [y^\top s(A,y,b) + b] \leq 0$. 
\end{lemma}
\begin{proof}[Proof of \cref{lem:classifier-proxy-inner-product}]
	First, consider the case when $y = 0$. In this case, $r(A,y,b) = s(A,y,b) = A$, and $\plbl(r(A,y,b),y,b) = \sign(b)$. If $\sign(b) = \lbl(A)$, then we have $\lbl(A)(y^\top s(A,y,b)+b) = \lbl(A) b \geq 0$. On the other hand, if $\sign(b) \neq \lbl(A)$, we have $\lbl(A) (y^\top s(A,y,b) + b) = \lbl(A) b \leq 0$. This proves the result when $y = 0$.
	
	Now, we consider $y \neq 0$. From \eqref{eq:manipulated}, we have 
	\begin{align}
		y^\top \replace{\tilde{r}\left( A,y,b-\frac{2\|y\|_*}{c} \right)}{r(A,y,b)} + b 
		&= \begin{cases}
			y^\top \left[A + \left(\frac{2}{c}-\frac{y^\top A+b}{\|y\|_*}\right)v(y)\right] + b, & \text{if } 0 \leq \frac{y^\top A+b}{\|y\|_*} < \frac{2}{c}, \\
			y^\top A + b, & \text{otherwise} 
		\end{cases}\notag\\
		&= \begin{cases}
			\frac2c \|y\|_*, & \text{if } 0 \leq \frac{y^\top A+b}{\|y\|_*} < \frac{2}{c}, \\
			y^\top A + b, & \text{otherwise}, 
		\end{cases}\label{eq:manipulated_product}
	\end{align}
	where the second equality follows as $y^\top v(y) = \|y\|_*$.  

	In addition, from \eqref{eq:proxy}, we obtain
	\begin{align}
	y^\top s(A,y,b) + b 
	&=\begin{cases}
		y^\top A +b - \frac{y^\top A + b}{\|y\|_*} y^\top v(y), & \text{if } 0 \leq \frac{y^\top A + b}{\|y\|_*} < \frac{2}{c} \text{ and } \lbl(A) = -1, \\
		y^\top A + b + \left( \frac{2}{c} - \frac{y^\top A + b}{\|y\|_*}\right) y^\top v(y), &  \text{if }  0 \leq \frac{y^\top A + b}{\|y\|_*} < \frac{2}{c} \text{ and } \lbl(A) = +1, \\
		y^\top A +b, & \text{otherwise} 
	\end{cases} \notag \\
	&=\begin{cases}
		0, & \text{if } 0 \leq \frac{y^\top A + b}{\|y\|_*} < \frac{2}{c} \text{ and } \lbl(A) = -1, \\
	 \frac{2}{c} \|y\|_*, &  \text{if }  0 \leq \frac{y^\top A + b}{\|y\|_*} < \frac{2}{c} \text{ and } \lbl(A) = +1, \\
		y^\top A +b, & \text{otherwise} .
	\end{cases} \label{eq:proxy_product}		
	\end{align}

	Now, suppose $0 \leq \frac{y^\top A + b}{\|y\|_*} < \frac{2}{c}$. Then, we deduce from \eqref{eq:manipulated_product} that $y^\top r(A,y,b) + b=\frac2c\|y\|_*>0$, so $\plbl(r(A,y,b),y,b) = \sign(0) = +1$. Also, from \eqref{eq:proxy_product} we deduce that $y^\top s(A,y,b) + b\ge0$. Hence, if the prediction is correct, i.e., $\plbl(r(A,y,b),y,b)=\lbl(A)$, then we have $\lbl(A)=1$ as well as $\lbl(A) [y^\top s(A,y,b)+b] \geq 0$; otherwise, $\lbl(A) [y^\top s(A,y,b) + b] \leq 0$, as desired. 
	
	Suppose $\frac{y^\top A + b}{\|y\|_*} < 0$. Then from \eqref{eq:manipulated_product} we have $y^\top r(A,y,b) + b = y^\top A + b < 0$, and from \eqref{eq:proxy_product} we have $y^\top s(A,y,b) + b = y^\top A + b < 0$. Notice also that $\plbl(r(A,y,b),y,b) = -1$. If $\lbl(A) = 1$, then the prediction is incorrect, and $\lbl(A)(y^\top s(A,y,b)+b) < 0$. If $\lbl(A) = -1$, then the prediction is correct, and $\lbl(A) (y^\top s(A,y,b)+b) > 0$. This proves the result in this case.
	
	Finally, suppose that $\frac{y^\top A + b}{\|y\|_*} > \frac2c$. Then from \eqref{eq:manipulated_product} we have $y^\top r(A,y,b) + b = y^\top A + b > \frac2c \|y\|_*$, and from \eqref{eq:proxy_product} we have $y^\top s(A,y,b) + b = y^\top A + b > 0$. Notice also that $\plbl(r(A,y,b),y,b) = +1$. If $\lbl(A) = 1$, then the prediction is correct, and $\lbl(A)(y^\top s(A,y,b)+b) > 0$. If $\lbl(A) = -1$, then the prediction is incorrect, and $\lbl(A) (y^\top s(A,y,b)+b) < 0$. This proves the result.
\end{proof}

Another important property is that the margin of proxy data $s(A,y,b)$ on a classifier $x \mapsto \sign(\bar{y}^\top x + \bar{b})$ is lower bounded by the margin of the true features $A$ on the same classifier, as long as $y$ and $\bar{y}$ are \emph{suitably aligned}.
	\begin{lemma}\label{lem:proxy-inclusion}
Let $(y,b),(\bar{y},\bar{b}) \in \bbR^d \times \bbR$ be such that $\bar{y}^\top v(y) \geq 0$. Then, for all $A \in \cA$,
		\[ \lbl(A) \cdot \left( \bar{y}^\top \replace{\tilde{s}\left( A,y,b - \frac{2\|y\|_*}{c} \right)}{s(A,y,b)} + \bar{b} \right) \geq \lbl(A) \cdot \left( \bar{y}^\top A + \bar{b} \right). \]
	\end{lemma}
	\begin{proof}[Proof of \cref{lem:proxy-inclusion}]
		Note that when $y=0$, $s(A,y,b) = A$ and $v(y)=0$ by definition, so the inequality trivially holds. Furthermore, when $\bar{y}=0$, $\bar{y}^\top v(y) = \bar{y}^\top s(A,y,b) = \bar{y}^\top A = 0$, so the inequality again trivially holds. We thus can consider $y,\bar{y} \in \bbR^d \setminus \{0\}$. 
		
		By the discussion right after \eqref{eq:proxy}, we have for any $A \in \cA$, $\replace{\tilde{s}(A,y,b-2\|y\|_*/c)}{s(A,y,b)}=A+\lbl(A)\cdot\alpha(A,y,b)\cdot v(y)$ for some $\alpha(A,y,b)\in[0,2/c]$. Thus, 
		\begin{align*}
			\lbl(A) \cdot \left( \bar{y}^\top \replace{\tilde{s}\left( A,y,b - \frac{2\|y\|_*}{c} \right)}{s(A,y,b)} + \bar{b} \right) &=\lbl(A) \cdot (\bar{y}^\top A + \bar{b}) + \alpha(A,y,b)\cdot \bar{y}^\top v(y)\\
			&\geq \lbl(A) \cdot (\bar{y}^\top A + \bar{b}),
		\end{align*}
		where the inequality follows from the premise of the lemma that $\bar{y}^\top v(y) \geq 0$.
	\end{proof}
	The assumption that $\bar{y}^\top v(y) \geq 0$ is crucial for \cref{lem:proxy-inclusion}; if removed, the result no longer holds.
	As an immediate corollary of \cref{lem:proxy-inclusion}, whenever the classifier $x \mapsto \sign(\bar{y}^\top x + \bar{b})$ separates all agents $A \in \cA$, it will also separate the proxy data $\replace{\tilde{s}\left( A,y,b - 2\|y\|_*/c \right)}{s(A,y,b)}$ generated based on $(y,b)$ as long as $y$ satisfies $\bar{y}^\top v(y) \geq 0$.
	\begin{corollary}\label{cor:proxy-inclusion}
		Let $(y,b),(\bar{y},\bar{b}) \in \bbR^d \times \bbR$ be such that $\bar{y}^\top v(y) \geq 0$. 
		If for some $\rho>0$ we have $\lbl(A) (\bar{y}^\top A + \bar{b}) \geq \rho$ for all $A\in\cA$, then
		\[ \lbl(A) (\bar{y}^\top s(A,y,b) + \bar{b}) \geq \rho, \quad  \forall A \in \cA.\]
		In particular, this holds when $(\bar{y},\bar{b}) = (y_*,b_*)$ from \cref{assum:margin} and for every $y$ such that $y_*^\top v(y)\ge0$, with $\rho = d_* \|y_*\|$: since $\lbl(A) (y_*^\top A + b_*)/\|y_*\|_* \geq d_*$, thus $\lbl(A) \cdot \left(y_*^\top s(A,y,b)+b_*\right)/\|y_*\|_* \geq d_*$ for all $A \in \cA$, i.e., the proxy points $s(A,y,b)$ are separable by the classifier $x \mapsto \sign(y_*^\top x + b_*)$ with a margin of at least $d_*$.
		
\end{corollary}

The algorithms we will present next, and their subsequent analysis in \cref{sec:theory}, rely on a certain margin function: given sets $\widetilde{\cA}^+, \widetilde{\cA}^-$ of feature vectors, we define
\begin{subequations}\label{eq:h-general}
\begin{align}
	h(y,b;\widetilde{\cA}^+, \widetilde{\cA}^-) &:= \min\left\{ \min_{x \in \widetilde{\cA}^+} \left\{ y^\top x + b \right\}, \min_{x \in \widetilde{\cA}^-} \left\{ -y^\top x - b \right\} \right\}\label{eq:h-general-def}\\
	&= \frac{1}{2} \left( \min_{x \in \widetilde{\cA}^+} y^\top x - \max_{x \in \widetilde{\cA}^-} y^\top x \right) - \left| b + \frac{1}{2} \left(\min_{x \in \widetilde{\cA}^+} y^\top x + \max_{x \in \widetilde{\cA}^-} y^\top x\right) \right|.\label{eq:h-general-alternate}
\end{align}
\end{subequations}
Note that \eqref{eq:h-general-alternate} follows from using the transformation $\min\{u,v\} = \frac{1}{2} (u + v - |u-v|)$ for any $u,v\in\R$.
By its definition in \eqref{eq:h-general-def}, it is easy to see that $h(y,b;\widetilde{\cA}^+, \widetilde{\cA}^-)$ is jointly concave in $(y,b)$ for any given pair of sets $\widetilde{\cA}^+, \widetilde{\cA}^-$.

\subsection{Strategic maximum margin algorithm}\label{sec:smm}

Our first algorithm is an adaptation of the well-known maximum margin classifier to the strategic setting. 
Recall that by \cref{cor:proxy-inclusion}, the proxy data $s(A,y,b)$ is separable by the classifier $x \mapsto \sign(y_*^\top x + b_*)$ when $y_*^\top v(y) \geq 0$; \cref{alg:data-driven} is based on exploiting this property explicitly.
Specifically, in \cref{alg:data-driven},
we will guarantee $y_*^\top v(y_t) \geq 0$ holds for all $y_t$  via a proper initialization and subsequent updates. This will then guarantee that the proxy data points \sloppy $(s(A_1,y_1,b_1),\lbl(A_1)),\ldots,(s(A_t,y_t,b_t),\lbl(A_t))$ are separable by the classifier $x \mapsto \sign(y_*^\top x + b_*)$. Therefore, \cref{alg:data-driven} then computes $y_{t+1}$ by solving \eqref{eq:data-driven} below, which is simply a maximum margin classification problem on the proxy data revealed so far.

We present the initialization scheme for \cref{alg:data-driven} in \cref{alg:initialization}, which is specifically designed to find a $(y_1,b_1)$ that satisfies $y_*^\top v(y_1) \geq 0$ under \cref{assum:margin}. This holds due to a technical property of maximum margin classifiers (\cref{lem:margin-prod}) that we prove in \cref{sec:theory}; see \cref{lem:y_star-y-nonnegative}. Furthermore, the initialization procedure given in~\cref{alg:initialization} will make at most $2$ mistakes.
\setcounter{algorithm}{-1}
\begin{algorithm}[ht]
	\caption{Initialization scheme.}
	\label{alg:initialization}
	\begin{algorithmic}
		\State Set $(y, b)=(0,1)$ and $\widetilde{\cA}_0^+ = \widetilde{\cA}_0^- = \emptyset$.
		\While{$\widetilde{\cA}_0^+ = \emptyset$ or $\widetilde{\cA}_0^- = \emptyset$}
			\State Present $(y,b)$ to agent and receive response $r(A,y,b)=A$ (since $y=0$).
			\State Predict $\plbl(A,y,b)=\sign(b)$ and receive $\lbl(A)$.
			\State Add $A$ to $\widetilde{\cA}_0^+$ or $\widetilde{\cA}_0^-$ according to its label.
			\State If $\widetilde{\cA}_0^+ = \emptyset$, set $b = -1$. Otherwise, if $\widetilde{\cA}_0^- = \emptyset$, set $b=+1$.
		\EndWhile
		\State \Return $\widetilde{\cA}_0^+$ and $\widetilde{\cA}_0^-$ and the optimal solution $(y_1,b_1)$ of the problem
		\begin{align*}\label{eq:data-driven-init}
			d_1 := \max_{\|y\|_* \leq 1, b \in \bbR} h(y,b; \widetilde{\cA}_0^+, \widetilde{\cA}_0^-). \tag{${\rm P}_0$}
		\end{align*}
	\end{algorithmic}
\end{algorithm}

\begin{algorithm}[ht]
  \caption{Strategic max-margin (SMM) algorithm.
  }
  \label{alg:data-driven}
  \begin{algorithmic}
    \State Use \cref{alg:initialization} to obtain $\widetilde{\cA}_0^+, \widetilde{\cA}_0^-$ and $(y_1,b_1)$.
    \For{$t=1,2,\dots$}
      \State Step 1. Declare classifier $x \mapsto \plbl(x,y_t,b_t)$, and receive agent response $r(A_t,y_t,b_t)$. Predict $\plbl(r(A_t,y_t,b_t),y_t,b_t)$.
      \State Step 2. Receive $\lbl(A_t)$ and compute the proxy data $s(A_t, y_t, b_t)$. If $\lbl(A_t)=+1$, set $\widetilde{\cA}_t^+ := \widetilde{\cA}_{t-1}^+ \cup \{s(A_t,y_t,b_t)\}$, $\widetilde{\cA}_t^- := \widetilde{\cA}_{t-1}^-$. Otherwise, if $\lbl(A_t) = -1$, set $\widetilde{\cA}_t^+ := \widetilde{\cA}_{t-1}^+$, $\widetilde{\cA}_t^- := \widetilde{\cA}_{t-1}^- \cup \{s(A_t,y_t,b_t)\}$.
      \State Step 3. Compute $(y_{t+1},b_{t+1})$ by solving the following problem 
      \begin{align}
      	d_{t+1} \coloneqq \max_{\|y\|_*\leq 1, b \in \bbR} h(y,b; \widetilde{\cA}_t^+, \widetilde{\cA}_t^-).  \tag{${\rm P}_t$}\label{eq:data-driven}
\end{align}
    \EndFor
  \end{algorithmic}
\end{algorithm}

By analyzing the optimality condition of maximum margin classifiers (\cref{lem:margin-prod}), we will establish in \cref{lem:y_star-y-nonnegative}  that every classifier $(y_t,b_t)$ generated in \cref{alg:data-driven} satisfies $y_*^\top v(y_t)\ge0$, and thus through \cref{cor:proxy-inclusion}, we have $s(A_t,y_t,b_t)$ are separable for all $t$. 
Consequently, we observe that \eqref{eq:data-driven} is simply a convex formulation of the following maximum margin problem that separates the positive and negative proxy data points:
\[
 \max_{y \neq 0, b \in \bbR} \min\left\{ \min_{x \in \widetilde{\cA}_t^+} \left\{\frac{y^\top x + b}{\|y\|_*} \right\}, \min_{x \in \widetilde{\cA}_t^-} \left\{\frac{-y^\top x - b}{\|y\|_*} \right\} \right\} 
= \max_{y \neq 0, b \in \bbR} h\left( \frac{y}{\|y\|_*}, \frac{b}{\|y\|_*}; \widetilde{\cA}_t^+, \widetilde{\cA}_t^- \right). 
\]
In particular, this problem aims to find the maximum margin affine hyperplane to separate $\widetilde{\cA}^+_t$ and $\widetilde{\cA}^-_t$.
In our implementation of \cref{alg:data-driven}, we solve the convex reformulation, which has the benefit that even if the data were inseparable, it is still solvable with optimal solution $(y_{t+1},b_{t+1}) = (0,0)$. In contrast, the non-convex maximum margin problem would be infeasible if the points are not separable. 
Note that if one of $\widetilde{\cA}_t^+$ or $\widetilde{\cA}_t^-$ were empty, then \eqref{eq:data-driven} would be unbounded above. However, our initialization procedure in \cref{alg:initialization} is designed to ensure that both sets are non-empty.

\subsection{Gradient-based strategic maximum margin algorithm}\label{sec:gradient-smm}

Note that at each iteration the size of the set $\widetilde{\cA}^+_t\cup \widetilde{\cA}^-_t$ grows by 1, and as a result the size of the problem  \eqref{eq:data-driven} increases with the iteration count $t$. Thus, \eqref{eq:data-driven}  may become very expensive to solve in \cref{alg:data-driven}. In \cref{alg:data-driven-subgradient-averaging}, we aim to ease this computational cost for the case when the cost function from \cref{assum:cost-norm} is the $\ell_2$-norm.
Observe that since $\objt{y,b}$ is concave in $(y,b)$, so \eqref{eq:data-driven} is simply a (nonsmooth) concave maximization problem. One can then, for example, use the projected subgradient ascent to solve \eqref{eq:data-driven}.

On the other hand, as we see more and more data we do not expect the function $h(y,b;\widetilde{\cA}^+_{t+1},\widetilde{\cA}^-_{t+1})$ to differ significantly from the function $\objt{y,b}$.
To exploit this, we employ ideas from the Joint Estimation-Optimization (JEO) framework introduced in \citet{ahmadi_data-driven_2014,ho-nguyen_dynamic_2021}.
In particular, instead of solving~\eqref{eq:data-driven} completely at each iteration $t$, we perform only a single projected subgradient ascent update based on $\objt{y,b}$.

However, we also wish to preserve the inequality $y_*^\top v(y_t) \geq 0$ throughout our procedure, so that by \cref{cor:proxy-inclusion} the proxy data $s(A_t,y_t,b_t)$ remain separable. In order to do this, we instead perform one step of subgradient ascent on the concave function
\[ g_t(y) = \max_{b \in \bbR} \objt{y,b} = \frac{1}{2} \left( \min_{x \in \widetilde{\cA}^+} y^\top x - \max_{x \in \widetilde{\cA}^-} y^\top x \right), \]
where the second inequality is due to \eqref{eq:h-general-alternate}. We then set $b_{t+1}=\argmax_{b \in \bbR} \objt{y_{t+1},b}$.
\begin{algorithm}[ht]
	\caption{Gradient-based SMM algorithm for $\ell_2$-norm costs.}
	\label{alg:data-driven-subgradient-averaging}
	\begin{algorithmic}
		\State Input sequence of stepsizes $\{\gamma_t\}_{t \geq 1}$.
		\State Use \cref{alg:initialization} to obtain $\widetilde{\cA}_0^+, \widetilde{\cA}_0^-$ and $(y_1,b_1)$. Set $z_1 = y_1$. 
		\For{$t=1,2,\dots$}
		\State Step 1. Declare classifier $x \mapsto \plbl(x,y_t,b_t)$, and receive agent response $r(A_t,y_t,b_t)$. Predict $\plbl(r(A_t,y_t,b_t),y_t,b_t)$. \State Step 2. Receive $\lbl(A_t)$ and compute the proxy data $s(A_t, y_t, b_t)$. If $\lbl(A_t)=+1$, set $\widetilde{\cA}_t^+ := \widetilde{\cA}_{t-1}^+ \cup \{s(A_t,y_t,b_t)\}$, $\widetilde{\cA}_t^- := \widetilde{\cA}_{t-1}^-$. Otherwise, if $\lbl(A_t) = -1$, set $\widetilde{\cA}_t^+ := \widetilde{\cA}_{t-1}^+$, $\widetilde{\cA}_t^- := \widetilde{\cA}_{t-1}^- \cup \{s(A_t,y_t,b_t)\}$.
		\State Step 3a. Compute
		\begin{align*}
			s_t^+ &\in \argmin_x \left\{ z_t^\top x : x \in \widetilde{\cA}_t^+ \right\}, \quad s_t^- \in \argmax_x \left\{ z_t^\top x : x \in \widetilde{\cA}_t^- \right\}.
		\end{align*}
		\State Step 3b. Compute
		\[ z_{t+1} 
		:= \Proj_{B_{\|\cdot\|_2}}\left(z_t + \gamma_t (s_t^+ - s_t^-)\right). \]
		\State Step 3c. Compute 
		\begin{align}\label{eq:data-driven-subgradient-averaging}
		y_{t+1} := \frac{\sum_{\tau\in[t+1]}\gamma_\tau z_{\tau}}{\sum_{\tau\in[t+1]}\gamma_\tau}, \quad b_{t+1} := -\frac{1}{2} \left( \min_{x \in \widetilde{\cA}_t^+} y_{t+1}^\top x + \max_{x \in \widetilde{\cA}_t^-} y_{t+1}^\top x \right). 
		\end{align}
		\EndFor
	\end{algorithmic}
\end{algorithm}

Note that the existing JEO framework in earlier works does not directly apply to our situation, because it works with a sequence of problems $\max_{z}g(z;p_t)$, each parameterized by a vector $p_t\in\R^m$ satisfying $p_t\to p_*$ as $t\to\infty$, with the goal of solving the problem $\max_{z}g(z;p_*)$ in the limit. Although we are also dealing with a sequence of problems $\max_{y,b} \objt{y,b}$, our objective function $\objt{y,b}$ cannot be parameterized by a single vector with a fixed dimension. Rather, $\objt{y,b}$ relies on all history data $\widetilde{\cA}_t^+ \cup \widetilde{\cA}_t^-$, the size of which keeps growing as $t$ increases. Therefore, convergence of \cref{alg:data-driven-subgradient-averaging} does not immediately follows from earlier results in JEO, but requires a more careful analysis.

\subsection{Generalized strategic perceptron}\label{sec:strategic-perceptron}

In the non-strategic setting where the agent cannot manipulate, the well-celebrated perceptron algorithm can find a linear classifier after making only finitely many mistakes~\citep{rosenblatt_perceptron_1958,novikoff_convergence_1962}.  
	The counterpart of the perceptron in the strategic setting was introduced recently by \citet{ahmadi_strategic_2021}. 
We present \cref{alg:projected-perceptron}, which is a generalized and unified version of two separate \emph{strategic perceptron algorithms} of \citet[Algorithms 2 and 3]{ahmadi_strategic_2021} designed to handle the cases of $\ell_2$- or $\ell_1$-norms respectively.
Note also that \cref{alg:projected-perceptron} covers different norms in the cost function beyond just $\ell_2$- or $\ell_1$-norm.

\begin{algorithm}[ht]
	\caption{Projected strategic perceptron algorithm.}
	\label{alg:projected-perceptron}
	\begin{algorithmic}
		\State $\L \subseteq \R^d \times \bbR$ is a given closed convex cone and $\gamma>0$ is a fixed stepsize. Initialize $q_0:=\begin{pmatrix} y_0 \\ b_0 \end{pmatrix} = \begin{pmatrix}
			0 \\ 0 \end{pmatrix}$.
		\For{$t=0,1,2,\dots,T$}
		\State Declare classifier $x \mapsto \plbl(x,y_t,b_t)$, and receive agent response $r(A_t,y_t,b_t)$. 
\State Predict the label to be $\plbl(r(A_t,y_t,b_t),y_t,b_t)$.
\State Receive $\lbl(A_t)$, compute the proxy data $s(A_t, y_t, b_t)$, and set $\xi_t:=\begin{pmatrix} s(A_t, y_t, b_t) \\ 1 \end{pmatrix}$. 
		\State Update by $q_{t+1}=\begin{pmatrix} y_{t+1} \\ b_{t+1} \end{pmatrix} :=\Proj_{\L}(z_{t+1})$ where
		\begin{align}\label{eq:projected-gd-stepsize}
			z_{t+1} = \begin{cases}
				q_t + \gamma \lbl(A_t) \cdot  \xi_t, & \text{if } \plbl(r(A_t,y_t,b_t),y_t,b_t) \neq \lbl(A_t), \\
				q_t, & \text{otherwise}. 
			\end{cases}
		\end{align}
\EndFor
	\end{algorithmic}
\end{algorithm}

In the non-strategic setting, the classical perceptron algorithm seeks to learn a classifier $x \mapsto \sign(y^\top x)$ with no intercept term $b=0$, under the following assumption.
\begin{assumption}[Margin without intercept]\label{assum:margin-zero-b}
The following optimization problem
\begin{align}
d_* := \max_{y \neq 0} \min_{A \in \cA} \left\{ \lbl(A) \cdot \frac{y^\top A}{\|y\|_*} \right\}
\end{align}
has an optimal solution $y_* \neq 0$ and optimal value $d_* > 0$.
\end{assumption}	
This is extended to classifiers with possibly a nonzero intercept $b$ by the usual technique of appending an extra unit coordinate to feature vectors: $x$ becomes $\bar{x} = (x,1)$ and $\bar{y} = (y,b)$, then $\bar{y}^\top \bar{x} = y^\top x + b$. The strategic perceptron of \citet{ahmadi_strategic_2021} is also designed to find classifiers with $b=0$ using the proxy data \eqref{eq:proxy}.
However, the technique of appending a unit coordinate to find a classifier with $b \neq 0$ does not {immediately} work in the strategic setting since the new unit coordinate should not be manipulated. \citet[Section 7, Algorithm 5]{ahmadi_strategic_2021} instead propose a different technique. Suppose that \cref{assum:margin} holds, so $x \mapsto \sign(y_*^\top x + b_*)$ separates the true feature vectors $\cA$. The idea is to find a point $p$ for which $y_*^\top p + b_* \approx 0$ via a linear search between some $A_+ \in \cA^+$ and $A_- \in \cA^-$ (which can be done without explicit knowledge of $y_*,b_*$). Then, the set $\cA - \{p\} = \{A-p : A \in \cA\}$ is linearly separable by the classifier $x \mapsto \sign(y_*^\top x)$, since $y_*^\top (A-p) \approx y_*^\top A + b_*$. The linear search uses the strategic perceptron \citep[Algorithms 2 and 3]{ahmadi_strategic_2021} as a subroutine, checking whether the actual number of mistakes exceeds the theoretical bound under \cref{assum:margin-zero-b}, in which case the method will then try a different $p$.
However, this technique requires knowledge of the maximum margin $d_*$, which is in general not available to the learner.

\citet{ahmadi_strategic_2021} employs a perceptron-style update to the classifier $y_t$ when a mistake is made, i.e., $\plbl(r(A_t,y_t,0),y_t,0) \neq \lbl(A_t)$. The main innovation of \citet{ahmadi_strategic_2021} is to replace the unobservable true feature vector $A_t$ by its proxy $s(A_t,y_t,0)$ in the perceptron update rule. We utilize the same idea in \cref{alg:projected-perceptron},
but with a projection step onto a convex cone $\bbL$ applied to the perceptron update. This is aimed to take advantage of further domain information on the classifiers $y_t$ whenever present. For example, if we know that $y_*\in\R^d_+$ (\cref{assum:non-negative}), we utilize this information in \cref{alg:projected-perceptron} by taking $\L=\R^d_+\times\R$. We illustrate how exploiting such further information on the classifiers leads to better mistake bounds in \cref{prop:projection-mistake-bound}. On the other hand, by setting $\L=\R^d\times\{0\}$ in \cref{alg:projected-perceptron}, we recover the strategic perceptron algorithm of \citet[Algorithm 2]{ahmadi_strategic_2021}. 
\citet{ahmadi_strategic_2021} analyzed the strategic perceptron and its variant for when the norm $\|\cdot\|$ used in \cref{assum:cost-norm} is $\ell_2$- and weighted $\ell_1$-norm respectively. By contrast, we will consider a more general form for the norm $\|\cdot\|$, which will cover both $\ell_2$- and weighted $\ell_1$-norm as special cases. In addition, the theoretical analysis of our generalized method in \cref{sec:guarantee-perceptron} will provide conditions for which finite mistake bounds are attainable for these general norms, as well as when \cref{assum:margin} holds instead of \cref{assum:margin-zero-b}, i.e., when the intercept $b_*$ is potentially nonzero.

\section{Theoretical Guarantees}\label{sec:theory}

In this section, we present conditions under which we can provide performance guarantees for the algorithms from \cref{sec:algorithms}. We are interested in three types of performance guarantees: (i) bounds on the number of mistakes an algorithm makes; (ii) bounds on the number of data points that are manipulated in the history of the algorithm; and (iii) convergence to the (non-strategic) maximum margin classifier $(y_*,b_*)$ from \cref{assum:margin}.

We show that when certain conditions are met, \cref{alg:projected-perceptron,alg:data-driven,alg:data-driven-subgradient-averaging} are all guaranteed to make finitely many mistakes, with explicit mistake bounds for \cref{alg:projected-perceptron,alg:data-driven}. We compare these explicit bounds in \cref{rem:compare-mistake-bounds}, and point out that the bound for \cref{alg:data-driven} involves different constants in \cref{eq:proxy-bounded} which may be much smaller than the constants used in the bound for \cref{alg:projected-perceptron}. That said, if we compare the bounds with the same constants, we show that the mistake bound for \cref{alg:data-driven} is at most an extra logarithmic factor worse than than that of \cref{alg:projected-perceptron}. We also observe numerically (see \cref{sec:numerical}) that \cref{alg:data-driven} makes much fewer mistakes than \cref{alg:projected-perceptron} on both real and synthetic data.
Our finite mistake bound result for \cref{alg:projected-perceptron} extends
the results of \citet{ahmadi_strategic_2021} to general norms. However, in  \cref{sec:guarantee-examples} through several examples, we demonstrate that a finite mistake bound is not guaranteed if certain conditions on the margin and the norm are not met.

When $d_* > 2/c$ in \cref{assum:margin}, we provide an explicit bound on the number of times the agent manipulates their feature vector $A_t \neq r(A_t,y_t,b_t)$ for \cref{alg:data-driven}, and we also show that \cref{alg:data-driven-subgradient-averaging} always has finitely many manipulations, though without an explicit bound.
The assumption $d_* > 2/c$ makes sense intuitively because if $d_*$ were less than $2/c$, then even if we found the maximum margin classifier $(y_*,b_*)$, there would exist some $A \in \cA$ such that $r(A,y_*,b_*) \neq A$. More generally, when $d_* < 2/c$, we can show that for any $(y,b)$, there will exist $A \in \cA$ that is either misclassified by $(y,b)$, or $r(A,y,b) \neq A$. 

As discussed in \cref{sec:problem-setting}, if $x \mapsto \sign(y_*^\top x + b_*)$ correctly classifies all true feature vectors $A \in \cA$, then $x \mapsto \plbl(x,y_*,b_*) = \sign(y_*^\top x + b_* - 2\|y_*\|_*/c)$ will correctly predict labels of all agents when responses $r(A,y_*,b_*)$ are presented instead of true feature vectors. Furthermore, as evidenced by \cref{ex:truthful-max-margin}, there can be other classifiers besides $(y_*,b_*)$ which correctly predict agent labels in the strategic setting, but using $(y_*,b_*)$ can improve truthfulness. Thus, recovering $(y_*,b_*)$ is of interest, and we show that \cref{alg:data-driven,alg:data-driven-subgradient-averaging} converges in the limit to $(y_*,b_*)$, though we do not provide convergence rates (see \cref{thm:data-driven-convergence,thm:averaging-convergence}). We also demonstrate that \cref{alg:projected-perceptron} in general may \emph{not} converge to $(y_*,b_*)$ even when $b_*=0$ (see \cref{ex:perceptron-margin}).

We summarize the critical conditions needed for finite mistake bounds,  finite manipulation bounds, and convergence to $(y_*,b_*)$ for these algorithms in \cref{tbl:finite_mistake_assum_compare}.

\begin{table}[h]
  \begin{tabular}{lccc}
    \toprule
    Algorithm & \makecell{Condition for\\finitely many mistakes} & \makecell{Condition for\\finitely many manipulations} & \makecell{Condition for\\convergence to $(y_*,b_*)$} \\
    \midrule
    \cref{alg:data-driven} 
			& \makecell{$d_*>0$\\general norm\\(\cref{thm:margin-best_mistake-bound})} 
			& \makecell{$d_*>2/c$\\general norm\\(\cref{thm:margin-best_manipulation-bound})}
			& \makecell{$d_*>2/c$\\general norm\\(\cref{thm:data-driven-convergence})} \\\midrule
    \cref{alg:data-driven-subgradient-averaging} 
			& \makecell{$d_*>0$\\$\ell_2$-norm\\(\cref{thm:averaging-convergence})} 
			& \makecell{$d_*>2/c$\\$\ell_2$-norm\\(\cref{thm:averaging-convergence})} 
			& \makecell{$d_*>2/c$\\$\ell_2$-norm\\(\cref{thm:averaging-convergence})} \\\midrule
    \makecell{\cref{alg:projected-perceptron}\\$\L=\R^d\times\R$}
			& \makecell{$d_*>2/c$\\general norm\\(\cref{prop:mistake-bound})} 
			& \makecell{does not hold\\(\cref{ex:perceptron-margin})} 
			& \makecell{does not hold\\(\cref{ex:perceptron-margin})} \\\midrule
			\makecell{\cref{alg:projected-perceptron}\\$\L=\R^d\times\{0\}$}
			& \makecell{$d_*>0$, $b_*=0$\\$\ell_2$-norm\\(\cref{prop:mistake-bound2})} 
			& \makecell{does not hold\\(\cref{ex:perceptron-margin})} 
			& \makecell{does not hold\\(\cref{ex:perceptron-margin})} \\\midrule
			\makecell{\cref{alg:projected-perceptron}\\$\L=\R^d_+\times\R$}
			& \makecell{$d_*>0$, $y_*\in\R^d_+$\\$\ell_p$-norm\\(\cref{prop:projection-mistake-bound})} 
			& \makecell{does not hold\\(\cref{ex:perceptron-margin})} 
			& \makecell{does not hold\\(\cref{ex:perceptron-margin})} \\
    \bottomrule
  \end{tabular}
	\caption{
Conditions for theoretical guarantees of \cref{alg:data-driven,alg:data-driven-subgradient-averaging,alg:projected-perceptron}. All the guarantees are under \cref{assum:bounded} of boundedness of $\{A_t\}$. The guarantees for finite mistakes and finite manipulations are under \cref{assum:margin} of separability, while the convergence guarantees further make i.i.d.\ assumptions on $\{A_t\}$ as in \cref{assum:data-driven-stochastic}. In addition, the guarantees for finite mistakes and finite manipulations for \cref{alg:data-driven,alg:projected-perceptron} come with explicit bounds.}
	\label{tbl:finite_mistake_assum_compare}
\end{table}

\subsection{Preliminaries}\label{sec:guarantee-preliminaries}
Our analysis relies on some properties of the $h(\cdot)$ function introduced in \eqref{eq:h-general}. We will also impose a mild structural assumption on $\|\cdot\|$ to ensure that maximizers of $h(\cdot)$ are unique, and $v(\cdot)$ from \cref{assum:unique-direction} is uniquely defined.
\begin{assumption}\label{assum:strictly-convex-norm}
	The norm $\|\cdot\|$ from \cref{assum:cost-norm} and its dual are both strictly convex, so that $\|\beta w + (1-\beta) z\| < \beta \|w\| + (1-\beta) \|z\|$ and $\|\beta w + (1-\beta) z\|_* < \beta \|w\|_* + (1-\beta) \|z\|_*$ whenever $\beta \in (0,1)$ and $w,z$ are not collinear. Thus, both the norm and its dual are differentiable away from $0$, and $\grad \|y\|_* = v(y)$ for all $y \neq 0$.
\end{assumption}

\begin{lemma}\label{lem:unique-sol}
Given two sets $\widetilde{\cA}^{+}, \widetilde{\cA}^{-} \subseteq \mathbb{R}^d$,
	if the problem
	\begin{align}\label{eq:limit-problem}
			\max_{\|y\|_* \leq 1, b \in \R} h(y,b; \widetilde{\cA}^+, \widetilde{\cA}^-)
	\end{align}
	is solvable with positive optimal value, then all optimal solutions $(\hat{y},\hat{b})$ satisfy $\|\hat{y}\|_* = 1$. Furthermore, under \cref{assum:strictly-convex-norm} there is a unique optimal solution.
\end{lemma}
\begin{proof}
Note that $\tilde{d}:=\max_{\|y\|_* \leq 1, b \in \bbR} h(y,b;\widetilde{\cA}^+, \widetilde{\cA}^-) \geq 0$ since $(y,b) = (0,0)$ is feasible. Furthermore, since $h$ is positive homogeneous in $(y,b)$, and we are assuming that the optimal value is positive, we must have that $\|\hat{y}\|_* = 1$. Otherwise, we could scale $(\hat{y}/\|\hat{y}\|_*, \hat{b}/\|\hat{y}\|_*)$ to get a solution with better objective value.

Now, we claim that $(\hat{y},\hat{b})$ is unique under \cref{assum:strictly-convex-norm}. First, note that from the structure of $h$ in \eqref{eq:h-general-alternate}, we must have
\[ \hat{b} = -\frac{1}{2} \left( \min_{x \in \widetilde{\cA}^+} \hat{y}^\top x + \max_{x \in \widetilde{\cA}^-} \hat{y}^\top x \right). \]
Let the optimal value be $\tilde{d} > 0$. If there exists an alternate solution $(\hat{y}',\hat{b}')$ then $\|\hat{y}'\|_*=1$ as well, and
\[ \hat{b}' = -\frac{1}{2} \left( \min_{x \in \widetilde{\cA}^+} (\hat{y}')^\top x + \max_{x \in \widetilde{\cA}^-} (\hat{y}')^\top x \right). \]
Since $h$ is concave, we have for any $\beta \in (0,1)$
\[ h(\beta \hat{y}+(1-\beta)\hat{y}', \beta \hat{b}+(1-\beta)\hat{b}'; \widetilde{\cA}^+, \widetilde{\cA}^-) \geq \tilde{d}. \]
Suppose $(\hat{y},\hat{b}) \neq (\hat{y}',\hat{b}')$. Then, $\hat{y} \neq \hat{y}'$ or $\hat{b} \neq \hat{b}'$. If $\hat{y} = \hat{y}'$, then by the expressions above for $\hat{b},\hat{b}'$ from \eqref{eq:h-general-alternate}, we must have $\hat{b}=\hat{b}'$. Therefore, it must be the case that $\hat{y} \neq \hat{y}'$. Thus, for any $\beta \in (0,1)$, since $\|\cdot\|_*$ is strictly convex, $\|\beta \hat{y}+(1-\beta)\hat{y}'\|_* < 1$. However, this means that we can obtain a better objective value than $\tilde{d}$ by normalizing, i.e., the objective value of the solution $\frac{1}{\|\beta \hat{y}+(1-\beta)\hat{y}'\|_*}[\beta (\hat y,\hat b)+(1-\beta)(\hat y',\hat b')]$ will be strictly higher than $\tilde{d}$, contradicting the optimality of $\tilde{d}$.
\end{proof}

	\tikzset{
	pline/.style={every plot/.style={
			mark=x,
			mark options={
				gray,
				very thick,
				solid,
				color=red!40!gray
			},
			mark size=6pt,
		},
		thick, gray, dashed
	},
}

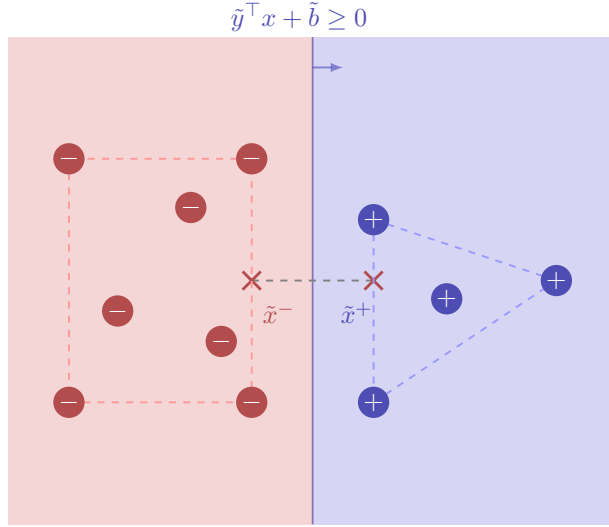
\begin{figure}[t!h]
	\begin{center}
	\scalebox{.9}{\begin{tikzpicture}[scale=.9]
			\draw[draw=blue!40!gray,thick,opacity=.6] (0,-4) -- (0,4);
			\draw[draw=blue!40!gray,thick,opacity=.6,-latex] (0,3.5) -- (0.5,3.5);
			\draw[draw=none,fill=blue!60!gray,opacity=.2] (0,-4) -- (0,4) -- (5,4) -- (5,-4) -- cycle;
			\draw[draw=none,fill=red!60!gray,opacity=.2] (0,-4) -- (0,4) -- (-5,4) -- (-5,-4) -- cycle;
			
			\draw[pline] plot coordinates {(-1,0) (1,0)};
			
			\draw[draw=red!40,thick,dashed] (-1,2) --  (-4,2);
			\draw[draw=red!40,thick,dashed] (-4,2) --   (-4,-2);
			\draw[draw=red!40,thick,dashed] (-4,-2) -- (-1,-2);
			\draw[draw=red!40,thick,dashed] (-1,-2) -- (-1,2);

			\draw[draw=blue!40,thick,dashed] (1,1) --  (1,-2);
			\draw[draw=blue!40,thick,dashed] (1,-2) --  (4,0);
			\draw[draw=blue!40,thick,dashed] (4,0) --  (1,1);
			
			\foreach \Point in {(-1,2), (-1,-2), (-4,-2), (-4,2), (-2,1.2), (-3.2,-.5), (-1.5,-1)}{
				\node[negative] at \Point {$\vphantom{+}-$};
			}
			\foreach \Point in {(1,1), (1,-2), (4,0), (2.2,-.3)}{
				\node[positive] at \Point {$+$};
			}
			
			\node[yshift=.3cm,xshift=-.2cm,text=blue!40!gray] at (0,4) {$\tilde{y}^\top x+\tilde{b}\geq0$};
			
			\node[yshift=.2cm,xshift=.4cm,text=red!40!gray] at (-1,-0.75) {$\tilde{x}^-$};
			\node[yshift=.2cm,xshift=-.25cm,text=blue!40!gray] at (1,-0.75) {$\tilde{x}^+$};
		\end{tikzpicture}
	}
	\end{center}
	\caption{Illustration for \cref{lem:margin-prod} where $\|\cdot\|=\|\cdot\|_2$.
}\label{fig:margin-prod}
\end{figure}

\begin{lemma}\label{lem:margin-prod}
Consider two compact and linearly separable sets $\widetilde{\cA}^+$ and $\widetilde{\cA}^-$, so that
\begin{align}\label{eq:arbitrary-margin}
	0<  \tilde{d} := \max_{y : \|y\|_* \leq 1, b \in \bbR} h\left( y,b; \widetilde{\cA}^+, \widetilde{\cA}^- \right) .
\end{align}
Let $(\tilde{y},\tilde{b})$ be the optimal solution to \eqref{eq:arbitrary-margin}. Then,  (see \cref{fig:margin-prod}) there exists $\tilde{x}^+\in \conv(\widetilde{\cA}^+)$ and $\tilde{x}^- \in \conv(\widetilde{\cA}^-)$ such that
\begin{align*}
\tilde{y}^\top (\tilde{x}^+ - \tilde{x}^-) &= \|\tilde{x}^+ - \tilde{x}^-\| \cdot \|\tilde{y}\|_*,\\
\text{and } \tilde{d} \cdot \|\tilde{y}\|_* &= \tilde{y}^\top \tilde{x}^+ + \tilde{b} = -\tilde{y}^\top \tilde{x}^- - \tilde{b} \implies \tilde{d} = \frac{\|\tilde{x}^+ - \tilde{x}^-\|}{2}, \ \tilde{b} = -\frac{\tilde{y}^\top (\tilde{x}^+ + \tilde{x}^-)}{2}.
\end{align*}
Furthermore, let $(\bar{y},\bar{b})$ be such that $h\left( \bar{y},\bar{b}; \widetilde{\cA}^+, \widetilde{\cA}^- \right) \geq \bar{d} > 0$. Under \cref{assum:strictly-convex-norm}, we have
\[ \bar{y}^\top v(\tilde{y}) \geq \frac{\bar{d}}{\tilde{d}} > 0. \] 
\end{lemma}

\begin{proof}[Proof of \cref{lem:margin-prod}]
Let $\partial h\left( \tilde{y},\tilde{b}; \widetilde{\cA}^+, \widetilde{\cA}^- \right)$ denote the \emph{super}differential of $h$ at $(\tilde{y},\tilde{b})$. By the optimality condition of \eqref{eq:arbitrary-margin}, there exists $(u,v)\in\partial h\left( \tilde{y},\tilde{b}; \widetilde{\cA}^+, \widetilde{\cA}^- \right)$ such that
\[ u^\top(\tilde{y} - y) + v(\tilde{b} - b) \geq0 \quad \forall y \text{ s.t. } \|y\|_* \leq 1, \ b \in \bbR. \]
Clearly, this means that $v = 0$, and $u^\top \tilde{y} \geq \|u\|$.

Now we derive the structure of $(u,v)$. As $(\tilde{y},\tilde{b})$ is an optimum solution to \eqref{eq:arbitrary-margin}, using the expressions for $(\tilde{y},\tilde{b})$ given by \eqref{eq:h-general-def} and \cref{lem:unique-sol}, we must have
\[ \min_{x \in \widetilde{\cA}^+} \left\{ \tilde{y}^\top x + \tilde{b} \right\} = \min_{x \in \widetilde{\cA}^-} \left\{ - \tilde{y}^\top x - \tilde{b} \right\}. \]
We define the sets $\widetilde{\cA}^+_*\coloneqq\argmin_{x\in\widetilde{\cA}^+}\{\tilde{y}^\top x+\tilde{b}\}$ and $\widetilde{\cA}^-_*\coloneqq\argmin_{x\in\widetilde{\cA}^-}\{-\tilde{y}^\top x-\tilde{b}\}$. Then,  $h\left( \tilde{y},\tilde{b}; \widetilde{\cA}^+, \widetilde{\cA}^- \right) = \tilde{y}^\top x^+ + \tilde{b} = -\tilde{y}^\top x^- - \tilde{b}$ for any $x^+ \in\widetilde{\cA}^+_*$ and $x^- \in\widetilde{\cA}^-_*$.
By compactness of $\widetilde{\cA}^+_*$ and $\widetilde{\cA}^-_*$, based on \citet[Chapter D, Theorem 4.4.2]{hiriart-urruty_fundamentals_2001} we arrive at
\[\partial h\left( \tilde{y},\tilde{b}; \widetilde{\cA}^+, \widetilde{\cA}^- \right)=\conv\left( \widetilde{\cA}^+_*\times\{1\}\cup(-\widetilde{\cA}^-_*)\times\{-1\} \right).\]
Therefore, there exist convex combination weights $\{\lambda_j\}_{j\in J^+\cup J^-}$ and points $\{x_j\}_{j\in J^+}\subseteq\widetilde{\cA}^+_*$ and $\{x_j\}_{j\in J^-}\subseteq\widetilde{\cA}^-_*$ such that $\sum_{j\in J^+}\lambda_j + \sum_{j\in J^-}\lambda_j=1$ and
\begin{align*}
	u &= \sum_{j\in J^+} \lambda_j x_j - \sum_{j\in J^-} \lambda_j x_j, \\
	0 = v &= \sum_{j\in J^+}\lambda_j - \sum_{j\in J^-}\lambda_j. 
\end{align*}
Note that since $\sum_{j\in J^+}\lambda_j = \sum_{j\in J^-}\lambda_j$ and $\sum_{j\in J^+}\lambda_j + \sum_{j\in J^-}\lambda_j=1$, we have $\sum_{j\in J^+}\lambda_j = \sum_{j\in J^-}\lambda_j = 1/2$. Let 
\begin{align*}
	\tilde{x}^+ &\coloneqq \frac{\sum_{j\in J^+}\lambda_jx_j}{\sum_{j\in J^+}\lambda_j} = 2 \sum_{j\in J^+}\lambda_jx_j, \quad 
	\tilde{x}^- \coloneqq \frac{\sum_{j\in J^-}\lambda_jx_j}{\sum_{j\in J^-}\lambda_j} = 2\sum_{j\in J^-}\lambda_jx_j. 
\end{align*}
As an immediate result, $\tilde{x}^+\in\conv(\widetilde{\cA}^+_*)\subseteq\conv(\widetilde{\cA}^+)$ and $\tilde{x}^-\in\conv(\widetilde{\cA}^-_*)\subseteq\conv(\widetilde{\cA}^-)$, and $u = \frac{1}{2} (\tilde{x}^+ - \tilde{x}^-)$. Note also that if $u=0$, then $\tilde{x}^+ = \tilde{x}^-$, which violates the linear separability of $\widetilde{\cA}^+$ and $\widetilde{\cA}^-$, so $u \neq 0$.

From before, we have $u^\top \tilde{y} \geq \|u\|$. Since $\|\tilde{y}\|_* \leq 1$, this implies $u^\top \tilde{y} = \|u\|$, and thus
\[ \tilde{y}^\top (\tilde{x}^+ - \tilde{x}^-) = \|\tilde{x}^+ - \tilde{x}^-\|. \]
Furthermore,
\[ \tilde{d} = h\left( \tilde{y},\tilde{b}; \widetilde{\cA}^+, \widetilde{\cA}^- \right) = \tilde{y}^\top \tilde{x}^+ + \tilde{b} = -\tilde{y}^\top \tilde{x}^- - \tilde{b}. \]
We can immediately deduce that $2\tilde{d} = \|\tilde{x}^+ - \tilde{x}^-\|$ hence $\|u\| = \tilde{d}$.

As $\|u\| \neq 0$, dividing both sides of $u^\top \tilde{y} = \|u\|$ by $\|u\|$, we get $1 = (u/\|u\|)^\top \tilde{y} \leq \|\tilde{y}\|_*$. Since $\|\tilde{y}\|_* \leq 1$, we thus have $(u/\|u\|)^\top \tilde{y} = \|\tilde{y}\|_* = 1$, hence $u/\|u\| = v(\tilde{y})$ by \cref{assum:strictly-convex-norm}. Now, observe that from $h\left( \bar{y},\bar{b}; \widetilde{\cA}^+, \widetilde{\cA}^- \right) \geq \bar{d} $, we deduce $\bar{y}^\top x + \bar{b} \geq \bar{d}$ for all $x \in \widetilde{\cA}^+$ and $-\bar{y}^\top x - \bar{b} \geq \bar{d}$ for all $x \in \widetilde{\cA}^-$. Therefore,
\begin{align*}
	\bar{y}^\top v(\tilde{y}) = \frac{1}{\|u\|} (\bar{y}^\top u) &= \frac{1}{\|u\|} \left(\sum_{j\in J^+} \lambda_j \bar{y}^\top x_j - \sum_{j\in J^-} \lambda_j \bar{y}^\top x_j\right)\\
	&\geq \frac{1}{\|u\|} \left(\sum_{j\in J^+} \lambda_j (\bar{d} - \bar{b}) + \sum_{j\in J^-} \lambda_j (\bar{d} + \bar{b})\right)\\
	&= \frac{1}{\|u\|} \left(\bar{d} \left(\sum_{j\in J^+} \lambda_j  + \sum_{j\in J^-} \lambda_j\right) + \bar{b} \left(\sum_{j\in J^-} \lambda_j - \sum_{j\in J^+} \lambda_j\right)\right)\\
	&= \frac{\bar{d}}{\|u\|} = \frac{\bar{d}}{\tilde{d}} > 0.
	\qedhere
\end{align*}
\end{proof}

\begin{lemma}[{Follows from \citet[Theorem 10.8]{rockafellar_convex_1970}}]\label{lem:uniform-convergence}\label{lem:unique-accumulation}
Let
\begin{align*}
	&\widetilde{\cA}_1^+ \subseteq \widetilde{\cA}_2^+ \subseteq \ldots \subseteq \widetilde{\cA}_\infty^+ \subset \bbR^d\\
	&\widetilde{\cA}_1^- \subseteq \widetilde{\cA}_2^- \subseteq \ldots \subseteq \widetilde{\cA}_\infty^- \subset \bbR^d
\end{align*}
be two nested sequences of subsets of $\bbR^d$. If both sets $\widetilde{\cA}_\infty^+$ and $\widetilde{\cA}_\infty^-$ are bounded, then the functions
\[ h_t(y,b) := h\left( y,b; \widetilde{\cA}_t^+, \widetilde{\cA}_t^- \right) \]
converge uniformly to the function
\[ h_\infty(y,b) := h\left( y,b; \widetilde{\cA}_\infty^+, \widetilde{\cA}_\infty^- \right) \]
over any compact domain $\cD \subset \bbR^d \times \bbR$.
More precisely, for any $\epsilon > 0$, there exists $t_0 \in \mathbb{N}$ such that when $t \geq t_0$, we have
\[ \sup_{(y,b) \in \cD} \left| h_t(y,b) - h_\infty(y,b) \right| \leq \epsilon. \]
\end{lemma}
\begin{proof}[{Proof of \cref{lem:unique-accumulation}}]
Since the sets are nested, $h_t(y,b) \to h_\infty(y,b)$ pointwise on $\cD$. The conclusion then follows from \citet[Theorem 10.8]{rockafellar_convex_1970}.  
\end{proof}

\subsection{Guarantees for \texorpdfstring{\cref{alg:data-driven}}{Algorithm 1}}\label{sec:guarantee-SMM}

We first verify that the critical condition $y_*^\top v(y_t) \geq 0$ required to apply \cref{cor:proxy-inclusion} to ensure separability of the proxy data holds throughout \cref{alg:data-driven}.

\begin{lemma}\label{lem:y_star-y-nonnegative}
For all $t\in\mathbb{N}$, let $(y_t,b_t)$ along with $d_t$ be generated by \cref{alg:data-driven}.
	Under \cref{assum:margin,assum:strictly-convex-norm}, we have $y_*^\top v(y_t) \geq \|y_*\|_* \tfrac{d_*}{d_t} > 0$ for all $t \in \mathbb{N}$.
\end{lemma}
\begin{proof}[Proof of \cref{lem:y_star-y-nonnegative}]
		We prove this by induction. For the base case, note that by its definition $(y_1,b_1)$ maximizes $h(y,b; \widetilde{\cA}^+_0, \widetilde{\cA}^-_0)$ in \cref{alg:initialization}. Furthermore, $\widetilde{\cA}^+_0 \subset \cA^+$, $\widetilde{\cA}^-_0 \subset \cA^-$ are separable by $(y_*,b_*)$ with margin $d_*$ by \cref{assum:margin}. Then, under \cref{assum:strictly-convex-norm}, by \cref{lem:margin-prod} we have $y_*^\top v(y_1) \geq \|y_*\|_* d_*/d_1 > 0$.
		
		Now assume as the induction hypothesis that $y_*^\top v(y_\tau) \geq \|y_*\|_* d_*/d_{\tau} > 0$ holds for $\tau \in [t]$. This implies that $\ell_\tau( y_*^\top s_\tau + b_* ) \geq \|y_*\|_* d_*$ for all $\tau \in [t]$ by \cref{cor:proxy-inclusion}, where $s_\tau = s(A_\tau,y_\tau,b_\tau)$ and $\ell_\tau = \lbl(A_\tau)$. Notice also that $(y_{t+1},b_{t+1})$ maximizes $h(y,b; \widetilde{\cA}_t^+,  \widetilde{\cA}_t^-)$, where $\widetilde{\cA}_t^+ := \{ s_\tau : \tau \in [t], \ell_\tau=+1 \}$ and $\widetilde{\cA}_t^- := \{ s_\tau : \tau \in [t], \ell_\tau=-1 \}$. Then, by \cref{lem:margin-prod}, $y_*^\top v(y_{t+1}) \geq \|y_*\|_* d_*/d_{t+1} > 0$. Therefore, by induction it holds for all $t \in \mathbb{N}$.
	\end{proof}

\begin{lemma}\label{lem:margin-best}
Under \cref{assum:margin,assum:strictly-convex-norm}, the optimal value $d_{t+1}$ of \eqref{eq:data-driven} in \cref{alg:data-driven} satisfies $d_* \leq d_{t+1} \leq \tilD^{\pm}$ for all $t \geq 0$, where $\tilD^{\pm}$ is defined in \cref{eq:proxy-bounded}.
\end{lemma}
\begin{proof}[Proof of \cref{lem:margin-best}]
	We have $(y_*/\|y_*\|_*)^\top v(y_{t+1}) \leq 1$ since $\|v(y_{t+1})\| \leq 1$ holds by \eqref{eq:manipulation-direction}. By \cref{lem:y_star-y-nonnegative}, this implies $d_{t+1} \geq d_*$. The inequality $d_{t+1}\leq \tilD^{\pm}$ follows immediately from the definition of $\tilD^{\pm}$ as the diameter of proxy data.
\end{proof}

We now establish that \cref{alg:data-driven} has a finite mistake bound.
To do this, we need to relate the margins $d_t$ found, i.e., the optimal values of \eqref{eq:data-driven}, between consecutive steps of \cref{alg:data-driven}.
\begin{proposition}\label{prop:margin-decrease}
Suppose the classifiers $(y_t,b_t)$ for $t\in\N$ are generated by \cref{alg:data-driven}.
Suppose for some $t \in \mathbb{N}$ we have $\lbl(A_t)[y_t^\top s(A_t,y_t,b_t) + b_t]\leq a\|y_t\|_*$ for some $a<d_*$. Then, 
under \cref{assum:margin,assum:bounded,assum:strictly-convex-norm}, we have
\[ d_{t+1} \leq \kappa(a,d_*,\barD) d_t,\]
where $d_t,d_{t+1}$ are the margins found by \cref{alg:data-driven}, $\barD$ is defined in \cref{eq:proxy-bounded}, and
\begin{align}\label{eq:kappa}
	\kappa(a,d_*, \barD) := \max_{\substack{w,z : \|w\|=1\\ \grad \|w\|^\top z \geq (1-a/d_*)/2 \\ \|z\|_2 \leq \barD/d_*}} \min_{\beta \in [0,1]} \|w-\beta z\| < 1,
\end{align}
\end{proposition}
When we use the $\ell_2$-norm in \cref{assum:cost-norm}, the $\kappa(\cdot)$ function has the following bound.
\begin{lemma}\label{lem:kappa-L2}
When $\|\cdot\| = \|\cdot\|_2$ in \cref{assum:cost-norm}, we have
\[ \kappa(a,d_*,\barD) \leq \sqrt{\max\left\{ 1 - \frac{(d_* - a)^2}{4 \barD^2}, \frac{d_*+a}{2d_*} \right\}}. \]
\end{lemma}
\begin{proof}[Proof of \cref{lem:kappa-L2}]
Note that when $\|w\|_2 = 1$, $\grad \|w\|_2 = w$. Let us compute
\[ \bar{\kappa}(\gamma,\delta) := \max_{\substack{w,z : \|w\|_2=1\\ w^\top z \geq \gamma \\ \|z\|_2 \leq \delta}} \min_{\beta \in [0,1]} \|w-\beta z\|_2. \]
First, for fixed $w$ and $z$ such that $\|w\|_2=1$ and $z \neq 0$, we have
\begin{align*}
\min_{\beta \in [0,1]} \|w-\beta z\|_2 &= \min_{\beta \in [0,1]} \sqrt{1 - 2 \beta w^\top z + \beta^2 \|z\|_2^2}\\
&= \sqrt{\min_{\beta \in [0,1]} \left\{1 - 2 \beta w^\top z + \beta^2 \|z\|_2^2\right\}}\\
&= \begin{cases}
1, &\text{if } w^\top z \leq 0\\
\sqrt{1 - \left(\frac{w^\top z}{\|z\|_2}\right)^2}, &\text{if } 0 < w^\top z \leq \|z\|_2^2\\
\|w-z\|_2, &\text{if } w^\top z > \|z\|_2^2, 
\end{cases}
\end{align*}
where the last equality follows from minimizing the quadratic, and recognizing that $\|w\|_2=1$. Then, if $\|z\|_2 \leq \delta$ and $w^\top z \geq \gamma > 0$, the first case cannot happen, and the second case $\leq \sqrt{1 - \left(\frac{\gamma}{\delta}\right)^2}$. For the third case, observe that $w^\top z > \|z\|_2^2$ implies $\|w-z\|_2^2 \leq 1-w^\top z \leq 1-\gamma$. Therefore, we have
\begin{align*}
\min_{\beta \in [0,1]} \|w-\beta z\|_2 &\leq \sqrt{\max\left\{ 1 - \left(\frac{\gamma}{\delta}\right)^2, 1-\gamma  \right\}},
\end{align*}
and thus $\bar{\kappa}(\gamma,\delta)$ is bounded by the term on the right hand side also. Finally, observe that
\[ \kappa(a,d_*,\barD) = \bar{\kappa}\left( \frac{1}{2} \left(1-a/d_*\right), \barD/d_* \right) \leq \sqrt{\max\left\{ 1 - \frac{(d_* - a)^2}{4 \barD^2}, \frac{d_*+a}{2d_*} \right\}}. 
\qedhere
\]
\end{proof}

Using \cref{prop:margin-decrease} we derive the following mistake bound for \cref{alg:data-driven}.
\begin{theorem}\label{thm:margin-best_mistake-bound}
Suppose $\{(y_t,b_t)\}_{t \in \bbN}$ are generated by \cref{alg:data-driven}. Let the iterations in which a mistake is made be denoted by $\mathcal{M}\coloneqq\{ t\in\N :  \plbl(r(A_t,y_t,b_t),y_t,b_t)\neq\lbl(A_t)\}$.
{Under the same assumptions as \cref{prop:margin-decrease}, we have}
\[ |\mathcal{M}| \leq\frac{\log\left( \tilD^{\pm}/d_* \right)}{\log\left( 1/\kappa\left( 0,d_*,\barD \right) \right)} < \infty,\]
where {$d_1$ is the optimal value of \eqref{eq:data-driven-init}.}
\end{theorem}

\cref{prop:margin-decrease} also leads to the following bound on the number of times an agent will manipulate their feature vector throughout the course of \cref{alg:data-driven}.
\begin{theorem}\label{thm:margin-best_manipulation-bound}
Suppose $\{(y_t,b_t)\}_{t \in \bbN}$ are generated by \cref{alg:data-driven}. Let
\[ \cN \coloneqq \{t\in\N : r(A_t,y_t,b_t) \neq A_t\}, \ \cN^+ \coloneqq \{t\in\cN : \lbl(A_t) = +1\}, \ \cN^- \coloneqq \{t\in\cN : \lbl(A_t) = -1\}. \]
Under the same assumptions as \cref{prop:margin-decrease}, we have
\[ |\cN^-| \cdot \log\left( 1/\kappa\left( 0,d_*,\barD \right) \right) + |\cN^+| \cdot \log\left( 1/\kappa\left( 2/c,d_*,\barD \right) \right) \leq \log\left( \tilD^{\pm}/d_* \right). \]
Therefore,
\[ |\cN^-| \leq \frac{\log\left( \tilD^{\pm}/d_* \right)}{\log\left( 1/\kappa\left( 0,d_*,\barD \right) \right)}< \infty, \]
and if we further assume $d_*>\frac2c$, then
\[ |\cN^+| \leq \frac{\log\left( \tilD^{\pm}/d_* \right)}{\log\left( 1/\kappa\left( 2/c,d_*,\barD \right) \right)} < \infty. \]
\end{theorem}

\begin{proof}[Proof of \cref{prop:margin-decrease}]
	To ease our notation, we let $s_t := s(A_t,y_t,b_t)$ and $\ell_t := \lbl(A_t)$. By \cref{lem:margin-prod}, for any $t\geq1$, there exists points $x_{+}$, $x_{-}$ such that
	\begin{align*}
		&x_{+} \in \conv\left(\widetilde{\cA}_{t-1}^+\right), \quad x_{-} \in \conv\left(\widetilde{\cA}_{t-1}^-\right), \\
		&y_t^\top x_{+} + b_t = d_t\|y_t\|_*, \quad y_t^\top x_{-} + b_t = -d_t\|y_t\|_*, \\
		&y_t^\top (x_+ - x_-) = \|x_+ - x_-\| \|y_t\|_*. \end{align*}
	This shows that $v(y_t) = \frac{x_+-x_-}{\|x_+-x_-\|}$ and $\frac{y_t}{\|y_t\|_*} = \argmax_{y : \|y\|_* \leq 1} y^\top (x_+-x_-)$. Furthermore, $\|x_+-x_-\| = 2d_t$.
	
	Suppose $\ell_t = +1$. Then \eqref{eq:data-driven} will find $(y_{t+1},b_{t+1})$ such that
	\begin{align*}
		y_{t+1}^\top s_t + b_{t+1} &\geq d_{t+1} \|y_{t+1}\|_*\\
		y_{t+1}^\top x_+ + b_{t+1} &\geq d_{t+1} \|y_{t+1}\|_*\\
		y_{t+1}^\top x_- + b_{t+1} &\leq -d_{t+1} \|y_{t+1}\|_*.
	\end{align*}
	Therefore, by H\"{o}lder's inequality, we have for any $\beta \in [0,1]$,
	\begin{align*}
		2d_{t+1} &\leq \left(\frac{y_{t+1}}{\|y_{t+1}\|_*}\right)^\top (\beta s_t + (1-\beta) x_+ - x_-)\\
		&\leq \|\beta s_t + (1-\beta) x_+ - x_-\|\\
		&= \|x_+-x_- - \beta(x_+-s_t)\|\\
		&= 2d_t \left\| \frac{x_+-x_-}{2d_t} - \beta\frac{x_+ - s_t}{2d_t} \right\|.
	\end{align*}
	Now let $f(\cdot) = \|\cdot\|$, and since by \cref{assum:strictly-convex-norm} the dual norm $\|\cdot\|_*$ is strictly convex, $f$ is differentiable everywhere except $0$. In particular, from $\frac{y_t}{\|y_t\|_*} = \argmax_{y : \|y\|_* \leq 1} y^\top (x_+-x_-)$ and $\|x_+-x_-\| = 2d_t$, we conclude  $y_t/\|y_t\|_* = \grad f((x_+ - x_-)/(2d_t))$. By the premise of the lemma, we have $y_t^\top s_t + b_t \leq a \|y_t\|_*$ and $y_t^\top x_+ + b_t = d_t \|y_t\|_*$, so
	\[ \grad f((x_+ - x_-)/(2d_t))^\top \frac{x_+ - s_t}{2d_t} = \frac{y_t^\top (x_+ - s_t)}{2 d_t \|y_t\|_*} \geq \frac{d_t - a}{2d_t} = \frac{1}{2} \left(1 - \frac{a}{d_t}\right). \]
	Furthermore, we have $\|x_+-s_t\|_2 \leq 2\tilD^+ \leq 2\barD$ and $\|x_+ - x_-\| = 2d_t$. To summarize, we have two vectors $w = \frac{x_+-x_-}{2d_t}$ and $z=\frac{x_+ - s_t}{2d_t}$ such that $f(w) = 1$, $\|z\|_2 \leq \barD/d_t$, $\grad f(w)^\top z \geq \frac{1}{2} (1-a/d_t)$, and for all $\beta \in [0,1]$
	\[ \frac{d_{t+1}}{d_t} \leq f(w - \beta z). \]
	We thus define
	\[ \bar{\kappa}(\gamma,\delta) := \max_{\substack{w,z : f(w)=1\\ \grad f(w)^\top z \geq \gamma\\ \|z\|_2 \leq \delta}} \kappa_0(w,z), \quad \kappa_0(w,z) := \min_{\beta \in [0,1]} f(w - \beta z). \]
	First notice that for any $w,z$ such that $f(w) = 1$, $\kappa_0(w,z) = \min_{\beta \in [0,1]} f(w - \beta z) \leq f(w)= 1$. Thus, {$\bar{\kappa}(\gamma,\delta) \leq 1$} holds whenever there exists a $w,z$ such that $\|w\| = 1$, $\grad f(w)^\top z \geq \gamma$ and $\|z\|_2 \leq \delta$. Furthermore, it is easy to see that as $\gamma$ decreases and $\delta$ increases, $\bar{\kappa}(\gamma,\delta)$ increases.
	
	We now show that $\bar{\kappa}(\gamma,\delta) < 1$ holds whenever $\gamma > 0$. First, for fixed $w,z$, we compute the derivative of $f(w-\beta z)$ at $\beta = 0$. This is $-\grad f(w)^\top z \leq -\gamma < 0$. Therefore, since $f(w) = 1$, for small $\beta > 0$, we have $f(w-\beta z) < 1$. This means $\kappa_0(w,z) = \min_{\beta \in [0,1]} f(w-\beta z) < 1$. Next, notice that since $f$ is continuous, $\kappa_0(w,z)$ is continuous in $w$ and $z$. Finally, the set
	\[ \{ (w,z) :~ f(w) = 1,~ \grad f(w)^\top z \geq \gamma,~ \|z\|_2 \leq \delta \} \]
	is compact (boundness is obvious due to the constraints $f(w) = 1$ and $\|z\|_2 \leq \delta$, and for closedness recall that $\|\cdot\|$ is strictly convex and thus $\grad f(\cdot)$ is continuous over unit vectors $f(w)=1$). Therefore, the maximum of $\kappa_0(w,z)$ is attained, and is $<1$.
	
	This shows that when $a < d_*$,
	\[ \frac{d_{t+1}}{d_t} \leq \bar{\kappa}\left( \frac{1}{2} \left(1-a/d_t\right), \barD/d_t \right) \leq \bar{\kappa}\left( \frac{1}{2} \left(1-a/d_*\right), \barD/d_* \right) = \kappa(a,d_*,\barD) < 1, \]
	where the second inequality follows since $d_t \geq d_*$ and the third inequality follows since $a < d_*$.
	
	An analogous argument follows for the case of $\ell_t = -1$.
\end{proof}

\begin{proof}[Proof of \cref{thm:margin-best_mistake-bound}]
For any $A_t\in\mathcal{M}$, \cref{lem:classifier-proxy-inner-product} implies that $\lbl(A_t)[y_t^\top s(A_t,y_t,b_t)+b_t]\leq0$. Therefore, we can apply \cref{prop:margin-decrease} with $a=0$ to derive $d_{t+1}\leq\kappa(0,d_*,\barD) d_t$. Note that $d_{t+1} \leq d_t$ when $t\notin\mathcal{M}$. 
Recall that by \cref{lem:margin-best}, we have $d_*\leq d_t$ for all $t \geq 1$, so we have $d_* \leq \kappa(0,d_*,\barD) d_t$ whenever $t \in\mathcal{M}$, and $d_* \leq  d_t$ otherwise. Applying the inequalities recursively leads to $d_* \leq d_1 \kappa(0,d_*,\barD)^{|\mathcal{M}|}$. Therefore, $|\mathcal{M}|\leq\frac{\log(d_1/d_*)}{\log(1/\kappa(0,d_*,\barD))}<\infty$ since $\kappa(0,d_*,\barD) < 1$ by \cref{prop:margin-decrease}. Finally, we have $d_1 \leq \tilD^{\pm}$ from \cref{lem:margin-best}. 
\end{proof}

\begin{proof}[Proof of \cref{thm:margin-best_manipulation-bound}]
Consider any iteration $t$ such that $r(A_t,y_t,b_t) \neq A_t$. Then, by \eqref{eq:manipulated} it must be the case that $0 \leq y_t^\top A_t + b_t < 2\|y_t\|_*/c$.

When $\lbl(A_t) = -1$, by \eqref{eq:proxy} and recalling $y^\top v(y)=\|y\|_*$ holds for any $y$, we deduce $y_t^\top s(A_t,y_t,b_t) + b_t = 0$.  Then, \cref{prop:margin-decrease} with $a=0$ gives $d_{t+1}\leq \kappa(0,d_*,\barD) d_t$. When $\lbl(A_t) = +1$, by \eqref{eq:proxy} we have $y_t^\top s(A_t,y_t,b_t) + b_t = 2\|y_t\|_*/c$. Then, \cref{prop:margin-decrease} with $a=2/c$ gives $d_{t+1}\leq \kappa(2/c,d_*,\barD) d_t$.

Also, $d_{t+1} \leq d_t$ for $t\notin\cN = \cN^+ \cup \cN^-$. By \cref{lem:margin-best} we have $d_t \geq d_*$ for all $t \geq 1$. Therefore, recursively applying these inequalities yields
\[ d_* \leq d_1 \cdot \kappa(0,d_*,\barD)^{|\cN^-|} \cdot \kappa(2/c,d_*,\barD)^{|\cN^+|}, \]
which implies
\[ |\cN^-| \cdot \log\left( 1/\kappa\left( 0,d_*,\barD \right) \right) + |\cN^+| \cdot \log\left( 1/\kappa\left( 2/c,d_*,\barD \right) \right) \leq \log(d_1/d_*) \leq \log\left( \tilD^{\pm}/d_* \right). \]
Note that $\kappa(0,d_*,\barD) < 1$ always, and when $d_* > 2/c$ we have $\kappa(2/c,d_*,\barD) < 1$, which verifies finiteness of the bounds.
\end{proof}

We now turn our attention to guarantees on convergence of iterates of \cref{alg:data-driven} to the maximum margin classifier $(y_*,b_*)$. In order to obtain this, we will make the following assumption, to ensure that $\cA$ is ``explored'' adequately. Otherwise, the sequence of $\{A_t\}_{t \in \bbN}$ could fail to cover $\cA$, and we would not recover $(y_*,b_*)$.

\begin{assumption}\label{assum:data-driven-stochastic}
	The set of possible feature vectors $\cA$ is closed, and there exists a probability distribution $\bbP$ over $\bbR^d$, for which $\cA$ is the support, in the sense that for any $A \in \cA$ and $\epsilon > 0$ we have $\bbP[\{x : \|A-x\| \leq \epsilon\}] > 0$, and for any $A \not\in \cA$, there exists $\epsilon > 0$ such that $\bbP[\{x : \|A-x\| \leq \epsilon\}] = 0$. Furthermore, the agents' feature vectors $\{A_t\}_{t \in \bbN}$ are drawn i.i.d. from the distribution $\bbP$.
\end{assumption}

\begin{theorem}\label{thm:data-driven-convergence}
	Suppose that \cref{assum:margin,assum:bounded,assum:strictly-convex-norm} hold with $d_* > 2/c$, and that $\{A_t\}_{t \in \bbN}$ are generated according to \cref{assum:data-driven-stochastic}. If $\{(y_t,b_t)\}_{t \in \bbN}$ are generated by \cref{alg:data-driven}, then $(y_t,b_t) \to (y_*/\|y_*\|_*,b_*/\|y_*\|_*)$ almost surely.
\end{theorem}
\begin{proof}[Proof of \cref{thm:data-driven-convergence}]
	Recall the notation $\widetilde{\cA}_t^+ = \{s(A_t,y_t,b_t) : t \in \bbN,\, \lbl(A_t) = +1\}$, $\cA^+ := \{A \in \cA : \lbl(A) = +1\}$ and $\widetilde{\cA}_t^- = \{s(A_t,y_t,b_t) : t \in \bbN,\, \lbl(A_t) = -1\}$, $\cA^- := \{A \in \cA : \lbl(A) = -1\}$. With this notation, by \cref{assum:margin}, we have
	\[ \left(\frac{y_*}{\|y_*\|_*}, \frac{b}{\|y_*\|_*} \right) \in {\argmax_{\|y\|_*\leq 1,b \in \bbR}} \obj{y,b}. \]
	Throughout this proof, we will assume that $\|y_*\|_* = 1$.
	
	Since each $\|s(A_t,y_t,b_t)\|_2 \leq \tilD$, by \cref{lem:uniform-convergence} the functions $\objt{\cdot}$ will uniformly converge to some $h_\infty(y,b)$. Since $d_*>\frac2c$, by \cref{thm:margin-best_manipulation-bound}, there exists $t_0\in\N$ such that $s(A_t,y_t,b_t)=A_t$ for all $t\geq t_0$, almost surely. Also, the set $\{A_t : t \geq t_0\}$ is dense in $\cA$ almost surely. To see this, suppose that it is not, so there exists some $A \in \cA$ and $\epsilon > 0$ such that $\|A_t - A\| > \epsilon$ for all $t \geq t_0$. Since the $A_t$ are i.i.d., this occurs with probability
	\[ \bbP\left[ \forall t \geq t_0, \|A_t - A\| > \epsilon \right] = \lim_{t \to \infty} \prod_{t \geq t_0} (1 - \bbP[\|A_t - A\| \leq \epsilon]) = 0. \]
	The final equality follows as $\bbP[\|A_t - A\| \leq \epsilon] > 0$ by \cref{assum:data-driven-stochastic}. Now, since $\objt{y,b} \leq h(y,b;\widetilde{\cA}_t^+\setminus\widetilde{\cA}_{t_0}^+,\widetilde{\cA}_t^-\setminus\widetilde{\cA}_{t_0}^-)$ for all $t \geq t_0$, hence
	\begin{align*}
		h_\infty(y,b) = \lim_{t\to\infty}\objt{y,b}
		&\leq \lim_{t\to\infty}h(y,b;\widetilde{\cA}_t^+\setminus\widetilde{\cA}_{t_0}^+,\widetilde{\cA}_t^-\setminus\widetilde{\cA}_{t_0}^-) = \obj{y,b}
	\end{align*}
	where the {last} equality follows as $\{\cA_t\}_{t \geq t_0} = \bigcup_{t \geq t_0} \left( \left(\widetilde{\cA}_t^+\setminus\widetilde{\cA}_{t_0}^+ \right) \cup \left(\widetilde{\cA}_t^-\setminus\widetilde{\cA}_{t_0}^-\right) \right)$ is dense in $\cA$.
	
	Next, observe that from \cref{lem:y_star-y-nonnegative,cor:proxy-inclusion} we know that $\lbl(A_t) (y_*^\top s(A_t,y_t,b_t) + b_*) \geq d_*$ for all $t\in\bbN$, thus $h_\infty(y_*,b_*) \geq d_*$. Since {$d_* \geq \obj{y,b} \geq h_\infty(y,b)$} almost surely, we deduce that $(y_*,b_*)$ maximizes $h_\infty$ over the constraints $\|y\| \leq 1$, $b \in \bbR$ almost surely. By \cref{lem:unique-sol}, it is the unique optimal solution. Finally, since $(y_{t+1},b_{t+1})$ is the maximizer of $\objt{y,b}$ over the same constraint set, and $\objt{\cdot}$ converges to $h_\infty(\cdot)$ uniformly, by \citet[Theorem 7.33]{rockafellar_variational_2009} we have that $(y_t,b_t) \to \argmax_{\|y\|_*\leq 1, b \in \bbR} h_\infty(y,b)$. As $(y_*,b_*)=\argmax_{\|y\|_*\leq 1, b \in \bbR} h_\infty(y,b)$ almost surely, we conclude $(y_t,b_t) \to (y_*/\|y_*\|_*,b_*/\|y_*\|_*)$ almost surely.
\end{proof}

\subsection{Guarantees for \texorpdfstring{\cref{alg:data-driven-subgradient-averaging}}{Algorithm 2}}\label{sec:guarantee-SMM-grad}

We now provide guarantees for \cref{alg:data-driven-subgradient-averaging} in the case when the agent manipulation cost is the $\ell_2$-norm. We first show the crucial property required to maintain separation of the proxy data throughout the course of the algorithm (see \cref{cor:proxy-inclusion}).
	
\begin{lemma}\label{lem:subgrad-averaging-positive-inner-product}
	Suppose \cref{assum:cost-norm} holds with $\|\cdot\|=\|\cdot\|_2$, and \cref{assum:margin} holds. For $\{y_t\}_{t \in \bbN}$ generated by \cref{alg:data-driven-subgradient-averaging}, we have $y_*^\top v(y_t) \geq 0$.
\end{lemma}
\begin{proof}
Since $\|\cdot\| = \|\cdot\|_2$, we have that $v(y) = y/\|y\|_2$ when $y \neq 0$ in \cref{assum:unique-direction}. Therefore the condition $y_*^\top v(y_t) \geq 0$ in \cref{cor:proxy-inclusion} is equivalent to $y_*^\top y_t \geq 0$.

We prove by induction that $y_*^\top z_t \geq 0$ and $y_*^\top y_t \geq 0$ for all $t \geq 1$. First, by \cref{lem:margin-prod}, the initialization procedure guarantees that we have $y_*^\top y_1 = y_*^\top z_1 \geq 0$, since we are separating points from $\cA$, which is guaranteed to have positive margin by \cref{assum:margin}.

Now assume that $y_*^\top y_\tau \geq 0$, $y_*^\top z_\tau \geq 0$ for all $\tau \in [t]$. The sets $\widetilde{\cA}_t^+ = \{s(A_\tau,y_\tau,b_\tau) : \tau \in [t], \lbl(A_\tau) = +1\}$ and $\widetilde{\cA}_t^- = \{s(A_\tau,y_\tau,b_\tau) : \tau \in [t], \lbl(A_\tau) = -1\}$ are separable by $(y_*,b_*)$ with margin at least $d_*$ by the assumption that $y_*^\top y_\tau \geq 0$ for $\tau \in [t]$. Therefore, as $s_t^+ \in \widetilde{\cA}_t^+$ and $s_t^- \in \widetilde{\cA}_t^-$ we have
\[ y_*^\top s_t^+ + b_* \geq d_* \|y_*\|_*, \quad y_*^\top s_t^- + b_* \leq -d_* \|y_*\|_*, \]
so 
\[ y_*^\top (z_t + \gamma_t {(s_t^+ - s_t^-)}) \geq y_t^\top z_t + 2\gamma_t d_* \|y_*\|_* \geq 0, \]
where in the last inequality we also used $y_*^\top z_t \ge0$ from the induction hypothesis. 
Since $z_{t+1} = \Proj_{B_{\|\cdot\|_2}}(z_t + \gamma_t (s_t^+ - s_t^-))$ is a non-negative scalar multiple of $z_t + \gamma_t {(s_t^+ - s_t^-)}$, we have $y_*^\top z_{t+1} \geq 0$ as well. Now, as $y_{t+1}$ is {simply} a convex combination of $z_1,\ldots,z_{t+1}$, we have $y_*^\top y_{t+1} \geq 0$ also. Thus, the result holds by induction. 
\end{proof}

We are now ready to state the performance guarantees for \cref{alg:data-driven-subgradient-averaging}. Like \cref{alg:data-driven}, \cref{alg:data-driven-subgradient-averaging} makes finitely many mistakes and manipulations, but unlike \cref{alg:data-driven}, we do not obtain explicit bounds. The classifiers  obtained from \cref{alg:data-driven-subgradient-averaging} also converge to the maximum margin classifier from \cref{assum:margin}.
\begin{theorem}\label{thm:averaging-convergence}
	Suppose \cref{assum:cost-norm} holds with $\|\cdot\|=\|\cdot\|_2$, and that \cref{assum:margin,assum:bounded} hold. Let $\{(y_t,b_t)\}_{t\in\N}$ be generated by \cref{alg:data-driven-subgradient-averaging} with $\gamma_t=1/\sqrt{t}$. Then, \cref{alg:data-driven-subgradient-averaging} makes finitely many mistakes. If we further assume $d_*>\frac2c$ in \cref{assum:margin}, then it also makes finitely many manipulations. When \cref{assum:data-driven-stochastic} and $d_* > \frac2c$ both hold in addition to the previous assumptions, then $(y_t,b_t) \to (y_*/\|y_*\|_2,b_*/\|y_*\|_2)$ almost surely. 
\end{theorem}
\begin{proof}[Proof of \cref{thm:averaging-convergence}]
	Let 
	\[ g_t(y) \coloneq \max_{b \in \bbR} h(y,b; \widetilde{\cA}_t^+, \widetilde{\cA}_t^-) = \frac{1}{2} \left( \min_{x\in\widetilde{\cA}_t^+} y^\top x - \max_{x\in\widetilde{\cA}_t^-} y^\top x \right). \]
	Observe that $g_t(y)$ is a concave function of $y$.
	Define sets
	\begin{align*}
		\widetilde{\cA}_\infty^+ &:= \widetilde{\cA}_0^+ \cup \left\{ s(A_t,y_t,b_t) : t \in \bbN,\  \lbl(A_t) = +1 \right\}\\ \widetilde{\cA}_\infty^- &:= \widetilde{\cA}_0^- \cup \left\{ s(A_t,y_t,b_t) : t \in \bbN,\  \lbl(A_t) = -1 \right\}.
	\end{align*}
	Then the pointwise limit of $g_t$ is
	\[ g_\infty(y) := \frac{1}{2} \left( \min_{x\in\widetilde{\cA}_\infty^+} y^\top x - \max_{x\in\widetilde{\cA}_\infty^-} y^\top x \right) \]
	Since $\|s(A_\tau,y_\tau,b_\tau)\|_2 \leq \tilD$, by arguments similar to \cref{lem:uniform-convergence} we can show that $g_t(\cdot) \to g_\infty(\cdot)$ uniformly over $\{y : \|y\|_2 \leq 1\}$.

	Note that $s_t^+ - s_t^-$ is simply a supergradient of $g_t(z_t)$, so the sequence $\{z_t\}_{t \geq 1}$ is obtained by online projected supergradient ascent on functions $g_t(\cdot)$, which are $\tilD$-Lipschitz continuous, with stepsizes $\gamma_t>0$. The standard analysis of online supergradient ascent gives, for any $z$ such that $\|z\|_2 \leq 1$,
	\begin{align}\label{eq:sgd-averaging}
		\frac{\sum_{t \in [T]} \gamma_t (g_t(z) - g_t(z_t))}{\sum_{t\in[T]} \gamma_t}
		&\leq \frac{\frac{1}{2} \|z_1 - z\|_2^2 + \frac{\tilD^2}{2} \sum_{t \in [T]} \gamma_t^2}{\sum_{t\in[T]} \gamma_t}. 
\end{align}
Since $\gamma_t=1/\sqrt{t}$, the right hand side converges to $0$ as $T\to\infty$, uniformly over $z$.

	By \cref{lem:subgrad-averaging-positive-inner-product} and \cref{cor:proxy-inclusion}, the sets $\widetilde{\cA}_t^+, \widetilde{\cA}_t^-$ are separated by $(y_*,b_*)$ with margin of at least $d_*$, so $g_t(y_*/\|y_*\|_2) \geq d_*$ for any $t \geq 1$. Thus, we have
	\begin{align*}
		\frac{\sum_{t \in [T]} \gamma_t (g_t(y_*/\|y_*\|_2) - g_t(z_t))}{\sum_{t\in[T]} \gamma_t}
		&= \frac{\sum_{t \in [T]} \gamma_t g_t(y_*/\|y_*\|_2)}{\sum_{t\in[T]} \gamma_t} - \frac{\sum_{t \in [T]} \gamma_t g_\infty(z_t)}{\sum_{t\in[T]} \gamma_t}\\
		&\quad - \frac{\sum_{t \in [T]} \gamma_t (g_t(z_t)-g_\infty(z_t))}{\sum_{t\in[T]} \gamma_t} \\
&\geq d_* - g_\infty(y_T) - \frac{\sum_{t \in [T]} \gamma_t \sup_{\|z\|_2 \leq 1} |g_t(z)-g_\infty(z)|}{\sum_{t\in[T]} \gamma_t}.
	\end{align*}
	Here, the inequality uses the concavity of $g_\infty$ which holds by \citet[Theorem 10.8]{rockafellar_convex_1970} as the functions $g_t$ are concave and converge to $g_\infty$ pointwise,
	and the definition of $y_T$. By Stolz-Ces\`{a}ro theorem and uniform convergence of $g_t$, the last term converges to $\lim_{T\to\infty} \sup_{\|z\|_* \leq 1} |g_T(z)-g_\infty(z)|=0$. Therefore, for any $\epsilon>0$, there exists $t_1(\epsilon)\in\N$ such that $d_*-g_\infty(y_t)\leq\epsilon$ for any $t\geq t_1(\epsilon)$.
	
	Using \eqref{eq:h-general-alternate}, we have
	\begin{align*}
		\objt{y_t,b_t}
&= g_t{(y_t)}- \left|  b_t + \frac{1}{2} \left( \min_{x\in\widetilde{\cA}_t^+} y_t^\top x + \max_{x\in\widetilde{\cA}_t^-} y_t^\top x \right) \right|,
	\end{align*}
	where the second term is
	\begin{align*}
		& \left| \frac{1}{2} \left( \min_{x\in\widetilde{\cA}_t^+} y_t^\top x + \max_{x\in\widetilde{\cA}_t^-} y_t^\top x \right) + b_t \right| \\
		&= \left| \frac{1}{2} \left( \min_{x\in\widetilde{\cA}_t^+} y_t^\top x + \max_{x\in\widetilde{\cA}_t^-} y_t^\top x \right) - \frac{1}{2} \left( \min_{x\in\widetilde{\cA}_{t-1}^+} y_t^\top x + \max_{x\in\widetilde{\cA}_{t-1}^-} y_t^\top x \right) \right| \\
		&\leq \frac{1}{2} \left| \min_{x\in\widetilde{\cA}_t^+} y_t^\top x - \min_{x\in\widetilde{\cA}_{t-1}^+} y_t^\top x \right| + \frac{1}{2} \left| \max_{x\in\widetilde{\cA}_t^-} y_t^\top x - \max_{x\in\widetilde{\cA}_{t-1}^-} y_t^\top x \right| \\
		&= \frac{1}{2} \left( \min_{x\in\widetilde{\cA}_{t-1}^+} y_t^\top x - \min_{x\in\widetilde{\cA}_t^+} y_t^\top x \right) + \frac{1}{2} \left( \max_{x\in\widetilde{\cA}_t^-} y_t^\top x - \max_{x\in\widetilde{\cA}_{t-1}^-} y_t^\top x \right) \\
		&= g_{t-1}(y_t) - g_t(y_t) \leq \sup_{\|z\|_* \leq 1} |g_{t-1}(z) - g_t(z)|. 
	\end{align*}
	Since $g_t$ converges uniformly to some $g_\infty$, this term converges to $0$, so there exists some $t_2(\epsilon)\in\N$ such that $\objt{y_t,b_t} \geq g_t(y_t) - \epsilon$ for any $t\geq t_2(\epsilon)$. Therefore, for $t\geq\max\{t_1(\epsilon),t_2(\epsilon)\}$, we have 
	\begin{align*}
		\lbl(A_t) (y_t^\top s(A_t,y_t,b_t) + b_t) \geq \objt{y_t,b_t} 
		\geq g_t(y_t) - \epsilon
		\geq g_\infty(y_t) - \epsilon
		\geq d_* - 2\epsilon. 
	\end{align*}
	By definition of $\widetilde{\cA}_t^+, \widetilde{\cA}_t^-$, this shows that
	for $0 < \epsilon < d_*/2$, $\lbl(A_t) (y_t^\top s(A_t,y_t,b_t) + b_t) > 0$. Then, by \cref{lem:classifier-proxy-inner-product},we must have $\plbl(r(A_t,y_t,b_t),y_t,b_t)=\lbl(A_t)$, i.e., the classifier $(y_t,b_t)$ predicts correctly at the $t$\textsuperscript{th} iteration. This shows \cref{alg:data-driven-subgradient-averaging} makes no mistakes after time $\max\{t_1(\epsilon),t_2(\epsilon)\}$, i.e., it makes finitely many mistakes.
	
	If we further assume $d_*>\frac2c$, then for $0 < \epsilon < (d_* - 2/c)/2$ sufficiently small, $\lbl(A_t) (y_t^\top s(A_t,y_t,b_t) + b_t) > \frac2c$, thus according to \eqref{eq:proxy} we must have $ s(A_t,y_t,b_t) = A_t$ and from \eqref{eq:manipulated} we must have $r(A_t,y_t,b_t) = A_t$ as well. Hence, there are finitely many manipulations when $d_* \geq \frac{2}{c}$.

We now show that, under \cref{assum:data-driven-stochastic}, we have $(y_t,b_t) \to (y_*/\|y_*\|_2,b_*/\|y_*\|_2)$. Under \cref{assum:bounded,assum:data-driven-stochastic}, $\cA^+,\cA^-$ are compact, and under \cref{assum:margin}, $(y_*/\|y_*\|_2, b_*/\|y_*\|_2)$ maximize \eqref{eq:arbitrary-margin} with $\widetilde{\cA}^+ = \cA^+$ and $\widetilde{\cA}^- = \cA^-$. Then, by \cref{lem:margin-prod} there exists two points $x^+ \in \conv(\cA^+), x^- \in \conv(\cA^-)$ such that
	\begin{align}\label{eq:witness2} 
	\frac{y_*}{\|y_*\|_2} = \frac{x^+-x^-}{\|x^+-x^-\|_2}, \quad  \frac{b_*}{\|y_*\|_2}= -\frac{1}{2\|y_*\|_2}\, y_*^\top \left( x^+ + x^- \right), \quad d_* = \frac{\|x^+-x^-\|_2}{2} . 
	\end{align}
	For $x^+$, by Carath\'{e}odory's theorem there exists $A_1^+,\ldots,A_k^+ \in \cA^+$ such that $x^+$ is a convex combination of these. Since $\{A_t\}_{t \in \bbN}$ is generated i.i.d. from $\bbP$, and the support of $\bbP$ is $\cA$, by the same argument as in the proof of \cref{thm:data-driven-convergence}, $\{A_t\}_{t \in \bbN}$ is dense in $\cA$ almost surely. Therefore, almost surely, there exists subsequences in $\{A_{t,i}\}_{t \in \bbN} \subset \{A_t\}_{t \in \bbN}$ such that $A_{t,i} \to A_i^+$ for all $i \in [k]$. 
	Now, since $s(A_t,y_t,b_t) = A_t$ for $t > t(\epsilon) := \max\{ t_1(\epsilon), t_2(\epsilon)\}$, we have that $\{A_{t,i}\}_{t > t(\epsilon)} \subset \widetilde{\cA}_\infty^+$, hence we have for any $y$ and $t > t(\epsilon)$,
	\[ \min_{x \in \widetilde{\cA}_\infty^+} y^\top x \leq y^\top A_{t,i}. \]
	Taking the limit as $t \to \infty$, this means we have for all $i \in [k]$
	\[ \min_{x \in \widetilde{\cA}_\infty^+} y^\top x \leq y^\top A_i^+ \implies \min_{x \in \widetilde{\cA}_\infty^+} y^\top x \leq y^\top x^+. \]
	By a similar argument, we have
	\[ \max_{x \in \widetilde{\cA}_\infty^-} y^\top x \geq y^\top x^-. \]
	However, this shows that for any $y$ such that $\|y\|_2 \leq 1$, we have
	\[ g_\infty(y) = \frac{1}{2} \left( \min_{x \in \widetilde{\cA}_\infty^+} y^\top x - \max_{x \in \widetilde{\cA}_\infty^-} y^\top x \right) \leq \frac{1}{2} y^\top (x^+ - x^-) \leq \frac{\|x^+ - x^-\|_2}{2} = d_*. \]
But now from \cref{lem:subgrad-averaging-positive-inner-product} we deduce $y_t^\top y_* \geq 0$, and so by \cref{lem:proxy-inclusion} we have that $\lbl(A_t)\cdot (y_*^\top s(A_t,y_t,b_t)+b_*) \geq \lbl(A_t)\cdot (y_*^\top A_t+b_*) \geq d_* \|y_*\|_2$ for all $t \in \bbN$.
	This implies that for all $t \in \bbN$ such that $\lbl(A_t) = +1$, we have
	\[ \left(y_*/\|y_*\|_2\right)^\top (s(A_t,y_t,b_t)) \geq d_* - b_*/\|y_*\|_2, \]
	and for all $t \in \bbN$ such that $\lbl(A_t) = -1$, we have
	\[ \left(y_*/\|y_*\|_2\right)^\top (s(A_t,y_t,b_t)) \leq -d_* - b_*/\|y_*\|_2. \]
	Since $\widetilde{A}_0^+ \subseteq \cA^+$ and $\widetilde{A}_0^- \subseteq \cA^-$, we also have $\left(y_*/\|y_*\|_2\right)^\top x \geq d_* - b_*/\|y_*\|_2$ for $x \in \widetilde{\cA}_0^+$ and $\left(y_*/\|y_*\|_2\right)^\top x \leq -d_* - b_*/\|y_*\|_2$ for $x \in \widetilde{\cA}_0^-$. Together, this means that
	\[ \min_{x \in \widetilde{\cA}_\infty^+} \left(y_*/\|y_*\|_2\right)^\top x \geq d_* - b_*/\|y_*\|_2 \quad \text{and} \quad  \min_{x \in \widetilde{\cA}_\infty^-} \left(y_*/\|y_*\|_2\right)^\top x \leq -d_* - b_*/\|y_*\|_2, \]
	thus 
	\[ g_{\infty}(y_*/\|y_*\|_2) \geq \frac{d_*-b*/\|y_*\|_2+d_*+b_*/\|y_*\|_2}{2} = d_*,\]
	which implies $y_*/\|y_*\|_2$ is an optimal solution to $\max_{y : \|y\|_2 \leq 1} g_\infty(y)$.
	By a similar argument to the proof of \cref{lem:unique-sol}, using the fact that $g_\infty(y)$ is positive homogeneous, it is in fact the unique optimal solution. 
	
	We showed above that $g_\infty(y_t) \geq d_* - \epsilon$ when $t > t(\epsilon)$, thus $\liminf_{t \to \infty} g_\infty(y_t) \geq d_*$. However, this together with $g_\infty(y)\leq d_*$ for all $y$ satisfying $\|y\|_2\leq 1$ implies $g_\infty(y_t) \to d_*$. Then, we must have that limit points of $\{y_t\}_{t \in \bbN}$ are optimal for $\max_{y : \|y\|_2 \leq 1} g_\infty(y)$. But since $y_*/\|y_*\|_2$ is the unique optimal solution, $y_t \to y_*/\|y_*\|_2$. This holds whenever $\{A_t\}_{t \in \bbN}$ is dense in $\cA$, which happens almost surely.
	
	To complete the proof, we will show that $b_t \to \frac{b_*}{\|y_*\|_2}$. 
	Recall that \cref{alg:data-driven-subgradient-averaging} computes
	\[ b_{t+1} = -\frac{1}{2} \left( \min_{x \in \widetilde{\cA}_t^+} y_{t+1}^\top x + \max_{x \in \widetilde{\cA}_t^-} y_{t+1}^\top x \right). \]
	When $y_t \to y_*/\|y_*\|_2$, taking the limit as $t \to \infty$ results in
	\[ \lim_{t \to \infty} b_t =  -\frac{1}{2 \|y_*\|_2} \left( \min_{x \in \widetilde{\cA}_\infty^+} y_*^\top x + \max_{x \in \widetilde{\cA}_\infty^-} y_*^\top x \right). \]
	Also, using the form of the function $g_\infty(\cdot)$ and \eqref{eq:witness2}  we deduce
	\begin{align*}
		0 &= g_\infty(y_*/\|y_*\|) - d_*\\
		&= \frac{1}{2} \left( \min_{x \in \widetilde{\cA}_\infty^+} (y_*/\|y_*\|_2)^\top x - (y_*/\|y_*\|_2)^\top x^+ \right) + \frac{1}{2} \left( (y_*/\|y_*\|_2)^\top x^- - \max_{x \in \widetilde{\cA}_\infty^-} (y_*/\|y_*\|_2)^\top x \right).
	\end{align*}
	Recall that in the preceding arguments we have shown $\min_{x \in \widetilde{\cA}_\infty^+} y^\top x \leq y^\top x^+$ and $\max_{x \in \widetilde{\cA}_\infty^-} y^\top x \geq y^\top x^-$, thus
both terms in the parentheses of the above expression must be $\leq0$
	which means that they must be $0$, hence
	\[ \lim_{t \to \infty} b_t =  -\frac{1}{2 \|y_*\|_2} \left( \min_{x \in \widetilde{\cA}_\infty^+} y_*^\top x + \max_{x \in \widetilde{\cA}_\infty^-} y_*^\top x \right) =  -\frac{1}{2} (y_*/\|y_*\|_2)^\top \left( x^+ + x^- \right) = \frac{b_*}{\|y_*\|_2}, \]
	as required.
\end{proof}

\subsection{Guarantees for \texorpdfstring{\cref{alg:projected-perceptron}}{Algorithm 3}}\label{sec:guarantee-perceptron}

Next, we will examine \cref{alg:projected-perceptron} with $\L=\R^d \times \R$, $\L=\R^d\times \{0\}$ and {$\L=\R_+^d\times \R$} under different structural assumptions and provide finite mistake bounds. To this end, given a sequence of classifiers $(y_t,b_t)$ and prediction labels $\widetilde{\ell}_t$ generated by \cref{alg:projected-perceptron}, we define
\[
\mathcal{M}_T \coloneqq \set{ t\in[T]:~ \widetilde{\ell}_t \neq \lbl(A_t) },
\]
which consists of all iterations where the algorithm makes a mistake, which in the case of \cref{alg:projected-perceptron} corresponds to all the iterations in which the classifier is updated.

The following lemma is an important ingredient for our results on bounding $|\mathcal{M}_T|$.
The proof of this result is based on a standard subgradient descent analysis using the hinge loss. To this end, recall that given a data point $\xi\in\R^{d+1}$ with label $\ell\in\{-1,+1\}$ and classifier $\xi \mapsto \sign\left(q^\top \xi \right)$,
the hinge loss function is defined as {$L_{\hinge}(q;\xi,\ell)\coloneqq\max\{0,1-\ell [q^\top \xi]\}$}. Note that the function $L_{\hinge}(q; \xi,\ell) $ is convex in $q$ for any $(\xi,\ell)$. Note also that the proof of \cref{lem:perceptron-mistake-bound} below requires an auxiliary technical result, \cref{lem:perceptron-stepsize_invariance}, which is given in the appendix. \begin{lemma}\label{lem:mistake-bound}\label{lem:perceptron-mistake-bound}
	Let $\L$ be any closed convex cone. Suppose we are given a sequence of points $A_t\in\cA$ for $t\in[T]$.
	Under \cref{assum:cost-norm,assum:unique-direction,assum:bounded}, for any $(y,b)\in\L\setminus\{0\}$, \cref{alg:projected-perceptron} with $\L$ and stepsizes $\gamma=1$ leads to the following inequality: 
	\begin{align*}
		|\mathcal{M}_T| - \sum_{t\in\mathcal{M}_T} L_{\hinge}(y,b; (s(A_t, y_t, b_t),1),\lbl(A_t)) 
		&\leq \sqrt{\|y\|_2^2 +b^2} \,\sqrt{\tilD^2+1} \sqrt{|\mathcal{M}_T|}, 
	\end{align*}
	where $\tilD$ is defined by \eqref{eq:proxy-bounded}. 
\end{lemma}
\begin{proof}[Proof of \cref{lem:mistake-bound}]
	Given a sequence of points $A_t$, let $q_t:=(y_t,b_t)$ be the sequence of classifiers obtained from \cref{alg:projected-perceptron} when we set $\gamma=1$. For the same sequence of points, let $q_t':=(y_t', b_t')$ be the sequence of classifiers obtained from \cref{alg:projected-perceptron} for a fixed $\gamma>0$. 
	To ease our notation, we let $q:=(y,b)$, $q_t:=(y_t,b_t)$, $q_t':=(y_t',b_t')$, $\ell_t:=\lbl(A_t)$, $r_t:=r(A_t,y_t,b_t)$, $r_t':=r(A_t,y_t',b_t')$, $s_t:=s(A_t,y_t,b_t)$, $\xi_t:=(s(A_t,y_t,b_t),1)$ and $\xi_t':=(s(A_t,y_t',b_t'),1)$. As \cref{assum:unique-direction} holds, by \cref{lem:perceptron-stepsize_invariance}, we conclude that $q_t' = \gamma q_t$, $r_t'=r_t$ and $\xi_t'=\xi_t$ for all $t$.
	
	Now, suppose $\gamma>0$ is fixed (we will determine its value at the end of the proof). Consider any $t\in\mathcal{M}_T$. Then, by definition of $\mathcal{M}_T$, we must have $\plbl(r_t,y_t,b_t) \neq \ell_t$. Then, by \cref{lem:classifier-proxy-inner-product} this implies that $0 \geq \ell_t \left(y_t^\top s_t +b_t \right) = \ell_t \cdot q_t^\top\xi_t$. Hence, from $q_t'=\gamma q_t$ with $\gamma>0$, we deduce $0\geq \ell_t \cdot q_t'^\top\xi_t$, and so we arrive at 
	\begin{align}\label{eq:hinge-vs-0-1}
		\sum_{t\in\mathcal{M}_T} L_{\hinge}(q_t'; \xi_t,\ell_t)
		&= \sum_{t\in\mathcal{M}_T} \max\left\{0, 1 - \ell_t \cdot q_t'^\top \xi_t\right\} 
		\geq \sum_{t\in\mathcal{M}_T} 1
		= |\mathcal{M}_T|. 
	\end{align}
	That is, the hinge loss is an upper bound on the $0$-$1$ loss, i.e., the number of mistakes made.
	In addition,  once again using $0\geq \ell_t \cdot q_t'^\top\xi_t$, we conclude that $ L_{\hinge}(q_t';  \xi_t,\ell_t)=1- \ell_t \cdot q_t'^\top\xi_t$ and hence 
	\begin{align}\label{eq:hinge-subgradient}
		\nabla L_{\hinge}(q_t';  \xi_t,\ell_t) = -\ell_t \cdot  \xi_t. 
	\end{align}
	Then, using the subgradient inequality on $L_{\hinge}$, we have  for any $\gamma>0$ and $q\in\L\setminus\{0\}$,
	\begin{align*}
		&\sum_{t\in\mathcal{M}_T} L_{\hinge}(q_t';  \xi_t,\ell_t) - \sum_{t\in\mathcal{M}_T} L_{\hinge}(q;  \xi_t,\ell_t) \\
		&\leq \sum_{t\in\mathcal{M}_T} \grad L_{\hinge}(q_t'; \ \xi_t,\ell_t)^\top (q_t' - q) &&\text{(subgradient ineq.)} \\
		&=   \sum_{t\in\mathcal{M}_T}  \left(-\ell_t [ \xi_t^\top  (q_t' - q) ] \right) &&\text{(by \eqref{eq:hinge-subgradient})} \\
		&=   \sum_{t\in\mathcal{M}_T}  \left(  
		-\frac{1}{2\gamma}\|q_t' - q - \gamma  \left(-\ell_t  \xi_t \right) \|_2^2 
		+ \frac{1}{2\gamma}\|q_t'-q\|_2^2 + 
		\frac{\gamma}{2} \| (-\ell_t  \xi_t )  \|_2^2
		\right) &&\text{(by perfect squares)} \\
		&= \sum_{t\in\mathcal{M}_T} \left(-\frac{1}{2\gamma}\|z_{t+1}' - q\|_2^2 + \frac{1}{2\gamma}\|q_t'-q\|_2^2 + \frac{\gamma}{2} \|\xi_t\|_2^2 \right) &&\text{(by \eqref{eq:projected-gd-stepsize})} \\
		&\leq \sum_{t\in\mathcal{M}_T} \left(-\frac{1}{2\gamma}\|q_{t+1}' - q\|_2^2 + \frac{1}{2\gamma}\|q_t'-q\|_2^2 + \frac{\gamma}{2} \|\xi_t\|_2^2 \right) &&\text{(projection)} \\
		&= -\frac{1}{2\gamma}\|q_{T+1}' - q\|_2^2 + \frac{1}{2\gamma}\|q_{0}'-q\|_2^2 + \sum_{t\in\mathcal{M}_T}\frac{\gamma}{2} \|\xi_t\|_2^2 &&\text{(telescoping)} \\
		&\leq \frac{1}{2\gamma}\|q\|_2^2 + \frac{\gamma}{2} |\mathcal{M}_T| \max_{t\in \mathcal{M}_T} \set{\|\xi_t\|_2^2} &&\text{(by $q_0=0$)}.
	\end{align*}
	Recall that $\xi_t=(s(A_t,y_t,b_t),1)$ and by \cref{assum:bounded} we have $\sup_{A\in\cA}\|s(A,y,b)\|_2 \leq \tilD$. Therefore, $\|\xi_t\|_2^2\leq {\tilD^2+1}$ for all $t$.   In summary, taking also \eqref{eq:hinge-vs-0-1} into account, we have arrived at for any $\gamma>0$ and $q\in\L\setminus\{0\}$
	\begin{align*}
		|\mathcal{M}_T| - \sum_{t\in\mathcal{M}_T} L_{\hinge}(q; \xi_t,\ell_t) 
		&\leq  \sum_{t\in\mathcal{M}_T} L_{\hinge}(q_t'; \xi_t,\ell_t) - \sum_{t\in\mathcal{M}_T} L_{\hinge}(q; \xi_t,\ell_t) \\
		&\leq \frac{1}{2\gamma}\|q\|_2^2 + \frac{\gamma}{2}(\tilD^2+1) |\mathcal{M}_T| \\
		&= \|q\|_2 \sqrt{\tilD^2+1} \sqrt{|\mathcal{M}_T|},
	\end{align*}
	where the last step follows by setting $\gamma = \frac{\|q\|_2}{\sqrt{\tilD^2+1}\sqrt{|\mathcal{M}_T|}}$. Note that as $q\neq0$, we have $\gamma>0$. Thus, we reach the desired conclusion by plugging the expressions $q=(y,b)$, $\xi_t$ and $\ell_t$.
\end{proof}

For the case $\L=\R^d \times \R$, i.e., \cref{alg:projected-perceptron} without the projection step, we establish the following upper bound on $|\mathcal{M}_T|$.

\begin{theorem}\label{prop:mistake-bound}\label{thm:perceptron-mistake-bound}
Suppose \cref{assum:cost-norm,assum:unique-direction,assum:margin,assum:bounded} hold. Given a sequence of points $A_t\in\cA$, let $(y_t,b_t)$ be the classifiers generated by \cref{alg:projected-perceptron} with $\L=\R^d \times \R$ and $\gamma=1$. Then, 
we have 
	\[
	\sum_{t\in \mathcal{M}_T} L_{\hinge}\left({y_*\over d_*\|y_*\|_*},{b_*\over d_*\|y_*\|_*}; (s(A_t, y_t, b_t),1),\lbl(A_t)\right) \leq\frac{2}{cd_*}|\mathcal{M}_T|.
	\]
	As a result, \cref{alg:projected-perceptron} with $\L=\R^d \times \R$ guarantees 
	\begin{align*}
		|\mathcal{M}_T| 
		\leq \frac{(\|y_*\|_2^2 + b_*^2)}{\|y_*\|_*^2} \cdot
		\frac{\left(\tilD^2+ 1\right)}{ \max\left\{0, {d_*-\frac{2}{c}}\right\}^2},
	\end{align*}
	where $\tilD$ is defined by \eqref{eq:proxy-bounded}.
\end{theorem}
\begin{proof}[Proof of \cref{thm:perceptron-mistake-bound}]
We define $\beta_*:=d_*\|y_*\|_*$. 
We will plug in $(y,b)={1\over \beta_*}(y_*,b_*)$ in \cref{lem:mistake-bound} and analyze the resulting bound from \cref{lem:mistake-bound}. To this end, we have 
\begin{align}
	&\sum_{t\in\mathcal{M}_T} L_{\hinge}\left({y_*\over \beta_*},{b_*\over \beta_*}; (s(A_t, y_t, b_t),1),\lbl(A_t)\right) \notag \\
	&= \sum_{t\in\mathcal{M}_T} \max\set{0, 1 - \lbl(A_t) \left[{1\over \beta_*}y_*^\top s(A_t, y_t, b_t)+{1\over \beta_*}b_*\right]}  \notag \\
	&= \sum_{t\in\mathcal{M}_T} \max\set{0, 1 - \lbl(A_t) {1\over \beta_*}[y_*^\top A_t  + b_* ] -  {1\over \beta_*}\alpha(A_t, y_t, b_t)\, y_*^\top v(y_t) }  \notag \\
	&\leq \sum_{t\in\mathcal{M}_T} \max\set{0,  - {1\over \beta_*}\alpha(A_t, y_t, b_t)\, y_*^\top v(y_t) } \label{eq:perceptron:hinge_bound}\\
	&\leq \sum_{t\in\mathcal{M}_T} {1\over \beta_*} \alpha(A_t, y_t, b_t)\, \|y_*\|_* \|v(y_t)\|  \notag \\
	&\leq \sum_{t\in\mathcal{M}_T} {2\over c d_* }   \|v(y_t)\|  \notag
\end{align}
Here, the second equation follows from the fact that $s(A,y,b)=A+\lbl(A)\cdot\alpha(A,y,b)\cdot v(y)$ for some $\alpha(A,y,b)\in[0,2/c]$, see e.g., \eqref{eq:proxy}, and the subsequent remark, and $\lbl(A)\in\{{-1},1\}$. The first inequality follows from $A_t\in\cA$ and \cref{assum:separable}, and the second inequality holds from Cauchy-Schwarz inequality, and the last one is due to $\beta_*:=d_*\|y_*\|_*$ and $\alpha(A_t,y_t,b_t)\in[0,2/c]$. As $\|v(y_t)\|\leq 1$ by definition, we conclude that
\[
\sum_{t\in\mathcal{M}_T} L_{\hinge}\left({y_*\over \beta_*},{b_*\over \beta_*}; (s(A_t, y_t, b_t),1),\lbl(A_t)\right) \leq  {2\over c d_* } |\mathcal{M}_T|.
\]
Plugging this relation in \cref{lem:mistake-bound} leads to
\[
|\mathcal{M}_T| - \frac{2}{cd_*}|\mathcal{M}_T|  \leq \frac{\sqrt{\|y_*\|_2^2 + b_*^2}\, 
\sqrt{\tilD^2+1}}{d_* \|y_*\|_*}\sqrt{|\mathcal{M}_T|}, \]
from which the mistake bound follows.
\end{proof}

\cref{prop:mistake-bound} shows that, if the margin $d_*>\frac{2}{c}$, then \cref{alg:projected-perceptron} with $\L=\R^d \times \R$ will make only a finite number of mistakes. That is, after sufficiently many iterations, the algorithm will stop updating and correctly classify every data point that follows \cref{assum:margin} thereon. As we will see in \cref{lem:mistake-bound-counterexample}, whenever the margin $0<d_*\leq\frac{2}{c}$, there are examples such that \cref{alg:projected-perceptron} makes an infinite number of mistakes and never stops at a correct classifier nor converges to $(y_*,b_*)$ unless we use $\ell_2$-norm. 

Note also that \cref{alg:projected-perceptron} with $\L=\R^d \times \R$ does not make an assumption of $b_*=0$, and thus it does not require a linear search procedure used by \citet[Section 7, Algorithm 5]{ahmadi_strategic_2021}. Moreover, whenever $d_*>\frac2c$ holds, \cref{prop:mistake-bound} gives a finite mistake bound for \cref{alg:projected-perceptron} for general norms $\|\cdot\|$ in \cref{assum:cost-norm}, and does so without relying on a linear search to find an intercept $b_*$. To the best of our knowledge, this is the first mistake bound for \cref{alg:projected-perceptron} for general norms $\|\cdot\|$ with or without a linear search for $b_*$.

Next, we analyze \cref{alg:projected-perceptron} with $\L=\R^d\times\{0\}$ under the assumption that $b_*=0$, establish its equivalence to \citet[Algorithm 2]{ahmadi_strategic_2021}. In this particular setting, when we have $\|\cdot\|=\|\cdot\|_2$ in \cref{assum:cost-norm}, we also establish that it achieves a finite mistake bound even when $d_*\leq \frac2c$. 
{\cref{thm:perceptron-mistake-bound2} below almost precisely matches the bound of \citet[Theorem 1]{ahmadi_strategic_2021}, except that we have an extra additive $1/d_*^2$ term. 
In fact, a more careful analysis of  \cref{lem:mistake-bound} that utilizes the information of $b_t=0$ for all $t$ would eliminate this additional term. To keep our exposition simple, we leave this to the reader.
} \begin{theorem}\label{prop:mistake-bound2}\label{thm:perceptron-mistake-bound2}
	Suppose \cref{assum:cost-norm} holds with $\|\cdot\|=\|\cdot\|_2$, and \cref{assum:unique-direction,assum:margin-zero-b,assum:bounded} hold. Given a sequence of points $A_t\in\cA$, let $(y_t,b_t)$ be the classifiers generated by \cref{alg:projected-perceptron} with $\L=\R^d\times\{0\}$ and $\gamma=1$. Then, the update rule in \cref{alg:projected-perceptron} is precisely the same as \citet[Algorithm 2]{ahmadi_strategic_2021} and also
	we have 
	\[
	\sum_{t\in\mathcal{M}_T} L_{\hinge}\left({y_*\over d_*\|y_*\|_2}, 0; (s(A_t, y_t, 0),1),\lbl(A_t)\right) \leq 0.
	\]
	As a result, \cref{alg:projected-perceptron} with $\L=\R^d\times\{0\}$ and $\gamma=1$ guarantees 
	\begin{align*}
		|\mathcal{M}_T| 
		\leq {\frac{\tilD^2+ 1 }{d_*^2}= \frac{(D+2/c)^2+ 1}{d_*^2}}.
	\end{align*}
\end{theorem}
\begin{proof}[Proof of \cref{thm:perceptron-mistake-bound2}]
Due to $\L=\R^d\times\{0\}$ and the projection step, we deduce that \cref{alg:projected-perceptron} with $\L=\R^d\times\{0\}$ will result in $b_t=0$ for all $t$. Therefore, when  $\|\cdot\|=\|\cdot\|_2$ holds in \cref{assum:cost-norm}, \cref{alg:projected-perceptron} with $\L=\R^d\times\{0\}$ and $\gamma=1$ becomes precisely \citet[Algorithm 2]{ahmadi_strategic_2021}.

To prove the other claims, we will first show by induction that $y_*^\top y_t\geq0$.  The base case holds because $y_0=0$. For induction hypothesis, suppose $y_*^\top y_t\geq0$ holds for some $t$. Then, based on $\plbl(r(A_t,y_t,b_t),y_t,b_t) = \lbl(A_t)$ or not, we have either $y_{t+1}=y_t$ (if we predicted the label correctly) and so $y_*^\top y_{t+1}=y_*^\top y_t\geq0$ holds immediately, or $y_{t+1}= y_t + \lbl(A_t)s(A_t,y_t,b_t)$ (when a mistake is made). In the latter case, using $s(A_t,y_t,b_t)=A_t+\lbl(A_t)\cdot\alpha_t\cdot v(y_t)$ for some $\alpha_t\in[0,2/c]$ we have
\begin{align*}
	y_*^\top y_{t+1}
	&= y_*^\top[y_t+\lbl(A_t)\cdot(A_t+\lbl(A_t)\alpha_tv(y_t))] \\
	&= y_*^\top y_t + \lbl(A_t)\cdot y_*^\top A_t + \alpha_t y_*^\top v(y_t) \\
	&\geq \alpha_t y_*^\top v(y_t) \\
	&= \alpha_t y_*^\top y_t \\
	&\geq 0.
\end{align*}
Here, the first inequality follows from the induction hypothesis $y_*^\top y_t\geq0$ and \cref{assum:margin-zero-b} which ensures $\lbl(A_t) (y_*^\top A_t)\ge0$. The last equation follows from the key fact that $v(y_t)=y_t$ for $\ell_2$ norm. And, finally the last inequality follows from the induction hypothesis.
This then concludes the induction. 

Following the same proof for \cref{prop:mistake-bound} and plugging in $(y,b)={1\over \beta_*}(y_*,0)$ with $\beta_*:=d_*\|y_*\|_2$ in \cref{lem:mistake-bound} and using \cref{assum:separable} give us \eqref{eq:perceptron:hinge_bound}, i.e., 
\begin{align*}
	&\sum_{t\in\mathcal{M}_T} L_{\hinge}\left({y_*\over d_*\|y_*\|_2}, 0; (s(A_t, y_t, 0),1),\lbl(A_t)\right)  
	\leq \sum_{t\in\mathcal{M}_T} \max\set{0,  - {1\over d_*\|y_*\|_2} \alpha(A_t, y_t, 0)\, y_*^\top v(y_t) },
\end{align*}
where $\alpha(A_t, y_t, 0)\in[0,2/c]$. 
Recall that $\|\cdot\|=\|\cdot\|_2$, and thus $y_t=v(y_t)$.  
Then, as $0\le y_*^\top y_t=y_*^\top v(y_t)$ holds for all $t$, we conclude that  the above summation expression must be nonpositive. 
Therefore, plugging in $(y,b)={1\over \beta_*}(y_*,0)={1\over d_* \|y_*\|_2}(y_*,0)$ in \cref{lem:mistake-bound}, results in 
\begin{align*}
	|\mathcal{M}_T| \leq {1\over d_*}\sqrt{\tilD^2+1}\, \sqrt{|\mathcal{M}_T|}, 
\end{align*}
which leads to the desired relation by observing that in this case we also have $C_{\|\cdot\|}=C_{\|\cdot\|_2}=1$ in the definition of $\tilD$ in \eqref{eq:proxy-bounded}.
\end{proof}

\begin{remark}\label{rem:compare-mistake-bounds}
When the cost is based on the $\ell_2$-norm, \cref{lem:kappa-L2,thm:margin-best_mistake-bound} provide the mistake bound of \cref{alg:data-driven} as
\begin{align}\label{eq:compare-data-driven-bound}
	\frac{2\log\left( \tilD^{\pm}/d_* \right)}{\log\left( \min\left\{ \frac{4\barD^2}{4\barD^2 - d_*^2}, 2 \right\} \right)},
\end{align}
where $\tilD^{\pm}$ and $\barD$ are defined in \cref{eq:proxy-bounded}. 
In contrast, \cref{thm:perceptron-mistake-bound2} gives a bound of
\begin{align}\label{eq:compare-perceptron-bound}
	\frac{\tilD^2 + 1}{d_*^2}
\end{align}
when $b_* = 0$ (the bound in \cref{thm:perceptron-mistake-bound} is worse for $b_* \neq 0$ so we will just examine this one). 
Note that when $d_* \geq 2 \barD$, the denominator of \eqref{eq:compare-data-driven-bound} becomes $2$, and since $\tilD^{\pm} \leq \tilD$, this means that \eqref{eq:compare-data-driven-bound} is an exponential improvement over \eqref{eq:compare-perceptron-bound}. Since $\barD$ measures the diameters of $\cA^+$ and $\cA^-$, while $d_*$ is the distance between the sets, it is possible for $d_* \geq 2 \barD$ to be satisfied.

We now consider the case when $d_* < 2\barD$, for which \eqref{eq:compare-data-driven-bound} becomes
\[\frac{2\log\left( \tilD^{\pm}/d_* \right)}{\log\left( \frac{4\barD^2}{4\barD^2 - d_*^2} \right)} = \frac{ \log\left( \left( \tilD^{\pm}/d_* \right)^2 \right)}{\log\left( 1 + \frac{d_*^2}{4\barD^2 - d_*^2} \right)}. \]
Since for any $x \geq -1$ we have $\frac{x}{1+x} \leq \log(1+x) \leq x$, we thus have
\[ \left( \frac{4 \barD^2}{d_*^2} - 1 \right) \log\left( \left(\tilD^{\pm}/d_*\right)^2 \right) \leq \eqref{eq:compare-data-driven-bound} \leq \frac{4 \barD^2}{d_*^2} \log\left( \left( \tilD^{\pm}/d_* \right)^2 \right). \]
Comparing the upper and lower bounds with \eqref{eq:compare-perceptron-bound}, we have the following:
\begin{align*}
&&&&\tilD^2 + 1 &< (4 \barD^2 - d_*^2) \log\left( \left( \tilD^{\pm}/d_* \right)^2 \right) &\implies \eqref{eq:compare-data-driven-bound} > \eqref{eq:compare-perceptron-bound}&&&&&\\
&&&&\tilD^2 + 1 &> 4 \barD^2 \log\left( \left( \tilD^{\pm}/d_* \right)^2 \right) &\implies \eqref{eq:compare-data-driven-bound} < \eqref{eq:compare-perceptron-bound}&&&&&
\end{align*}
and
\begin{align*}
& (4 \barD^2 - d_*^2) \log\left( \left( \tilD^{\pm}/d_* \right)^2 \right) \leq \tilD^2 + 1 \leq  4 \barD^2 \log\left( \left( \tilD^{\pm}/d_* \right)^2 \right)\\
&\implies \left| \eqref{eq:compare-data-driven-bound} - \eqref{eq:compare-perceptron-bound} \right| \leq \log\left( \left( \tilD^{\pm}/d_* \right)^2 \right).
\end{align*}
Since $\max\left\{ \barD, \tilD^{\pm} \right\} \leq \tilD$, the upper bound also implies that
\[ \eqref{eq:compare-data-driven-bound} \leq \frac{4 \barD^2}{d_*^2} \log\left( \left( \tilD^{\pm}/d_* \right)^2 \right) \leq 4\left( \frac{\tilD^2 + 1}{d_*^2} \right) \log\left( \frac{\tilD^2 + 1}{d_*^2} \right), \]
i.e., \eqref{eq:compare-data-driven-bound} is within a multiplicative logarithmic factor of \eqref{eq:compare-perceptron-bound}.

In practice, the constants $\barD, \tilD^{\pm}$ may be much less than $\tilD$, thus it is often the case that $\eqref{eq:compare-data-driven-bound} < \eqref{eq:compare-perceptron-bound}$. Even if this does not hold, \eqref{eq:compare-data-driven-bound} is no more than a logarithmic factor worse than \eqref{eq:compare-perceptron-bound}.
\end{remark}

We next analyze the case when $y_* \in \bbR_+^d$. This assumption enables us to use the version of \cref{alg:projected-perceptron} with a projection step onto $\L=\R_+^d\times\R$, and moreover we will show that under this assumption {we can provide finite mistake bounds for any $d_* > 0$, even when the norm $\|\cdot\|$ in \cref{assum:cost-norm} satisfy \cref{assum:norm-non-negative} below but it is not necessarily $\|\cdot\|_2$. As \cref{lem:norm-non-negative} shows, this still allows for a large class of norms such as all $\ell_p$-norms for $p \in [1,\infty]$.}

\begin{assumption}\label{assum:non-negative}
	\cref{assum:margin} holds with $y_* \in \R_+^d$. 
\end{assumption}

\begin{assumption}\label{assum:norm-non-negative}
	If $y\in\R_+^d$, then the selection $v(y)$ from \cref{assum:unique-direction} satisfies $v(y)\in\R_+^d$. 
\end{assumption}

\begin{lemma}\label{lem:norm-non-negative}
	If a norm $\|\cdot\|:\R^d\to\R_+$ satisfies that, for any $(x^{(1)},\dots,x^{(d)})\in\R^d$, 
	\begin{align*}
		\|(x^{(1)}, \dots, x^{(i)}, \dots, x^{(d)})\| \geq \|(x^{(1)}, \dots, |x^{(i)}|, \dots, x^{(d)})\|, 
	\end{align*}
	then $v(y)$ can be set so that \cref{assum:norm-non-negative} holds. 
\end{lemma}
\begin{proof}[Proof of \cref{lem:norm-non-negative}]
  Consider any $y\in\R_+^d$. Recall that
  \begin{align*} 
    v(y)
    &\in \argmax_{\|x\| \leq 1} y^\top x.
  \end{align*}
  For $x=(x^{(1)},\dots,x^{(d)})\in\partial\|y\|_*$, let $x_+:=(|x^{(1)}|,\dots,|x^{(d)}|)\in\R_+^d$, then using $y\in\R^d_+$ we deduce $x^\top y \leq x_+^\top y$. Note also that by the definition of $x_+$, using the premise of the lemma we arrive at $\|x\| \geq \|x_+\|$. Then, combining these two, we conclude
  \begin{align*}
    y^\top x \leq \frac{y^\top x_+}{\|x_+\|},
  \end{align*}
 Hence $x_+\in\argmax_{\|x\| \leq 1} y^\top x$ as well. Therefore, we can always choose $v(y) \in \bbR_+^d$.
\end{proof}

Next, for any  $d_*>0$, under \cref{assum:non-negative,assum:norm-non-negative}, we present a finite mistake bound for \cref{alg:projected-perceptron} with $\L=\R_+^d\times\R$. 

\begin{theorem}\label{prop:projection-mistake-bound}\label{thm:nonnegative-perceptron-mistake-bound}
	Suppose \cref{assum:cost-norm,assum:unique-direction,assum:margin,assum:bounded,assum:non-negative,assum:norm-non-negative} hold.  
	Given a sequence of points $\{A_t\}_{t \in \bbN} \subseteq \cA$, let $(y_t,b_t)$ be the classifiers generated by \cref{alg:projected-perceptron} with $\L=\R^d_+\times\R$ and $\gamma=1$. Then we have 
	\[
	\sum_{t\in \mathcal{M}_T} L_{\hinge}\left({y_*\over d_*\|y_*\|_*},{b_*\over d_*\|y_*\|_*}; (s(A_t, y_t, b_t),1),\lbl(A_t)\right) 
	\leq 0.
	\]
	As a result, \cref{alg:projected-perceptron} with $\L=\R^d_+\times\R$ guarantees 
	\begin{align*}
		|\mathcal{M}_T| 
		\leq \frac{ (\|y_*\|_2^2 + b_*^2) (\tilD^2+1)}{\|y_*\|_*^2\, d_*^2 } 
	\end{align*}
	where $\tilD$ is defined by \eqref{eq:proxy-bounded}.
\end{theorem}
\begin{proof}[Proof of \cref{thm:nonnegative-perceptron-mistake-bound}]
We define $\beta_*:=d_*\|y_*\|_*$. 
Following the same proof for \cref{prop:mistake-bound} and plugging in $(y,b)={1\over \beta_*}(y_*,b_*)$ in \cref{lem:mistake-bound} and using \cref{assum:separable} give us \eqref{eq:perceptron:hinge_bound}, i.e., 
\begin{align*}
	&\sum_{t\in\mathcal{M}_T} \ell_{\hinge}\left({y_*\over \beta_*}, {b_*\over \beta_*}
	; (s(A_t, y_t, b_t),1),\lbl(A_t)\right)  
	\leq \sum_{t\in\mathcal{M}_T} \max\set{0,  - {1\over \beta_*}\alpha(A_t, y_t, b_t)\, y_*^\top v(y_t) }.
\end{align*}
We claim that this expression is less than or equal to zero. This is 
because $y_*\in\R_+^d$ by \cref{assum:non-negative}, $y_t\in\R_+^d$ holds due to the projection step in \cref{alg:projected-perceptron} with $\L=\R^d_+\times\R$, and $v(y_t)\in\R_+^d$ by \cref{assum:norm-non-negative}.
Then, by plugging in $(y,b)={1\over \beta_*}(y_*,b_*)$ in \cref{lem:mistake-bound}, we deduce 
\begin{align*}
|\mathcal{M}_T| \leq \frac{ \sqrt{\|y_*\|_2^2 + b_*^2} \, \sqrt{\tilD^2+1}}{d_* \|y_*\|_*} \sqrt{|\mathcal{M}_T|}, 
\end{align*}
which leads to the desired relation.
\end{proof}

\begin{remark}\label{rem:perceptron-comparison}
\cref{prop:projection-mistake-bound} gives a mistake bound for \cref{alg:projected-perceptron} with $\L=\R^d_+\times\R$ for a general class of norms $\|\cdot\|$, including $\ell_p$-norms for any $p\geq1$. Note that even when we assume $b_*=0$ and use this information in \cref{alg:projected-perceptron} by setting $\L=\R^d_+\times\{0\}$ and select $\|\cdot\|$ in \cref{assum:cost-norm} to be the weighted $\ell_1$-norm, we observe that the resulting version of \cref{alg:projected-perceptron} and \citet[Algorithm 3]{ahmadi_strategic_2021} which was desiged specifically for the weighted $\ell_1$-norm are indeed distinct despite the fact that they have some similarities.
More precisely, in \cref{alg:projected-perceptron} we project onto $\R^d_+\times \{0\}$ at each iteration to ensure $y_t\in\R^d_+$, but on the other hand this is guaranteed in a different ad hoc manner using the properties of $\ell_1$-norm in \citet[Algorithm 3]{ahmadi_strategic_2021}.
In addition, we bypass potential non-uniqueness of agent responses by making \cref{assum:unique-direction} instead of adding a tie-breaking step utilized by  \citet[Algorithm 3]{ahmadi_strategic_2021}.

When restricted to the case of weighted $\ell_1$-norm and under the assumption that $b_*=0$, the mistake bound of \citet[Algorithm 3, Theorem 2]{ahmadi_strategic_2021} is $((d+1)(D+2/c)^2 + 1)/d_*^2$.
In contrast, as the $C_{\|\cdot\|}$ term in \eqref{eq:proxy-bounded} is $1$ for the $\ell_1$-norm, we deduce $\tilD = D + 2C_{\|\cdot\|}/c = D + 2/c$, and the mistake bound of \cref{thm:nonnegative-perceptron-mistake-bound} is $(\|y_*\|_2/\|y_*\|_\infty)^2 ((D+2/c)^2 + 1)/d_*^2 \leq d ((D+2/c)^2 + 1)/d_*^2$, which has a {similar} dependence on the dimension $d$.
Furthermore, when $b_*$ is unknown, their approach requires a linear search to discover $b_*$ (see  \citet[Algorithm 5]{ahmadi_strategic_2021}) which incurs an extra $O(D/d^*)$ multiplicative term in the mistake bound. In contrast, our approach directly updates $b_t$ values, and as a result our mistake bound incurs an extra additive $(b_*/\|y_*\|_\infty^2)((D+2/c)^2+1)/d_*^2$ term.

Finally, \citet[Algorithm 3]{ahmadi_strategic_2021} is specifically designed to work for only the weighted $\ell_1$-norm; in contrast \cref{assum:norm-non-negative} covers a wide variety of norms, including all $\ell_p$-norms for $p \geq 1$, and so \cref{alg:projected-perceptron} and its mistake bound guarantee given in \cref{prop:projection-mistake-bound} remain applicable in a much broader range.
\end{remark}

\subsection{Necessity of the assumptions for theoretical guarantees
}\label{sec:guarantee-examples}

In this section, through various examples, we explore the necessity of assumptions on $d_*$ that are prevalent in our guarantees on the finite mistake/manipulation bounds and convergence to $(y_*,b_*)$ for \cref{alg:data-driven,alg:projected-perceptron}.

While we focus on \cref{alg:data-driven,alg:projected-perceptron}, we conjecture that when \cref{alg:data-driven} fails, \cref{alg:data-driven-subgradient-averaging} will also fail since \cref{alg:data-driven-subgradient-averaging} is an approximation of \cref{alg:data-driven}. However, we defer rigorous investigation of this to future work.

We first show that when $d_* < 2/c$, \cref{alg:data-driven} may still have infinitely many manipulations under \cref{assum:data-driven-stochastic} \emph{with positive probability}, and furthermore that $(y_t,b_t)$ may not converge to $(y_*,b_*)$. This constrasts with \cref{thm:margin-best_manipulation-bound,thm:data-driven-convergence} which states that when $d_* > 2/c$, we guarantee finite manipulation and convergence to $(y_*,b_*)$.

\begin{example}\label{ex:data-driven-not-converge}
{Suppose that $\cA = \cA^+ \cup \cA^-$ where $\cA^+ = \{(0,1),(-2,1)\}$, $\cA^-=\{(-2,-1)\}$. With $\|\cdot\| = \|\cdot\|_2$ in \cref{assum:cost-norm}, we have that \cref{assum:margin} is satisfied with $d_* = 1$ and $y_*=(0,1)$, $b_*=0$. We set $c = \sqrt{2}$, thus $1/c < d_* < 2/c$. Let $\bbP$ be the distribution which draws any point from $\cA$ with uniform probability $1/3$. We let $A_t$ be drawn i.i.d. from this distribution, so \cref{assum:data-driven-stochastic} is satisfied. Consider the event $\cE$ that both $(0,1)$ and $(-2,-1)$ are drawn at least once before ever seeing $(-2,1)$; then it is easy to check that $\bbP[\cE] = 1/3$. On the event $\cE$, we have that \cref{alg:data-driven} computes $y_t = (1,1)/\sqrt{2}$, $b_t = 1/\sqrt{2}$ for all $t \geq 1$, and infinitely many manipulations occur.}
\end{example}
\begin{proof}[Proof of \cref{ex:data-driven-not-converge}]
{On the event $\cE$ where both $(0,1)$ and $(-2,-1)$ are drawn at least once before ever seeing $(-2,1)$, the initialization \cref{alg:initialization} sets $\widetilde{\cA}_0^+ = \{(0,1)\}$, $\widetilde{\cA}_0^- = \{(-2,-1)\}$, $y_1 = (1,1)/\sqrt{2}$, $b_1 = 1/\sqrt{2}$ and $d_1 = \sqrt{2} = 2/c$. We show that in fact $(y_t,b_t) = (y_1,b_1)$ for all $t \geq 1$, since $\widetilde{A}_t^+ = \{(0,1),(-1,2)\}$, $\widetilde{\cA}_t^- = \{(-2,-1)\}$.

For convenience, we will use the notation $(y,b) := (y_1,b_1)$,  $A_1 = (0,1)$, $A_2 = (-2,1)$ and $A_3 = (-2,-1)$ (noting that these are \emph{not} time indices). Suppose it holds that $(y_t,b_t) = (y,b)$ up to some time $t \geq 1$. First, recognize that $y^\top A_1 + b = \sqrt{2} = 2/c$, $y^\top A_2 + b = 0$ and $y^\top A_3 + b = -\sqrt{2} = -2/c$. Therefore, according to \eqref{eq:notation} the only point that is manipulated is $A_2$, for which we have $r(A_2,y,b) = A_2 + (1,1) = (-1,2)$ and as $\lbl(A_2) = +1$ we also have $s(A_2,y,b)=r(A_2,y,b)$ as well. We thus have $\widetilde{A}_t^+ = \{(0,1),(-1,2)\}$, $\widetilde{\cA}_t^- = \{(-2,-1)\}$. However, we can check that $y^\top (-1,2) + b = \sqrt{2}$, so $(y,b)$ also maximizes the  margin between the sets $\widetilde{\cA}_t^+$, $\widetilde{\cA}_t^-$. This means that $(y_{t+1},b_{t+1}) = (y,b)$. Therefore, by induction it holds for all $t \geq 0$.

Furthermore, $A_2$ will be drawn from $\bbP$ infinitely often almost surely, so infinitely many manipulations will occur on $\cE$.
}
\end{proof}

For the strategic perceptron \cref{alg:projected-perceptron}, the situation is slightly more nuanced. Recall that if $\bbL = \bbR^d \times \bbR$, then  \cref{thm:perceptron-mistake-bound} guarantees that \cref{alg:projected-perceptron} makes finitely many mistakes when $d_* > 2/c$, regardless of the norm. On the other hand, when $\bbL = \bbR^d \times \{0\}$ and $\|\cdot\|=\|\cdot\|_2$ holds in \cref{assum:cost-norm}, \cref{thm:perceptron-mistake-bound2} guarantees that finitely many mistakes are made for any $d_* > 0$.
The following example shows that if $d_* \leq 2/c$ and the norm $\|\cdot\|$ in \cref{assum:cost-norm} is selected to be anything other than $\ell_2$-norm, there exists $\cA$ and distribution $\bbP$ over it satisfying \cref{assum:data-driven-stochastic} such that \cref{alg:projected-perceptron} makes infinitely many mistakes with probability $\geq 1/3$.
\begin{example}\label{lem:mistake-bound-counterexample}
	Suppose that $\|\cdot\|$ in \cref{assum:cost-norm} is any norm \emph{other} than the $\ell_2$-norm. There exists $\cA$ that satisfies \cref{assum:margin-zero-b} with $d_* \leq 2/c$, and a distribution $\bbP$ over $\cA$ satisfying \cref{assum:data-driven-stochastic}, such that \cref{alg:projected-perceptron} with $\L=\R^d\times\{0\}$ and $\gamma=1$ makes infinitely many mistakes on this sequence with probability $\geq 1/3$.
\end{example}
\begin{proof}[Proof of \cref{lem:mistake-bound-counterexample}]
	By \cref{lem:l2-norm-symmetry}, for any norm $\|\cdot\|$ other than the $\ell_2$-norm, there exist $z,w\in\R^d \setminus\{0\}$ such that $w = v(z)$ is not parallel to $z$. Furthermore, we will choose $z$ such that $\|z\| = 2/c$.
{As $w=v(z)$, by \eqref{eq:manipulation-direction}, we have $\|w\| = 1$. We let $\cA := \{z,-z,2w/c\}$, where $\lbl(z) = +1$, $\lbl(-z)=-1$ and $\lbl(2w/c) = -1$.}
	
	We first claim that $\{z\}$ and $\{-z,2w/c\}$ are linearly separable by a hyperplane passing through the origin. To see this, simply take $\bar{y} = z/\|z\|_2 - w/\|w\|_2$. Then, $\bar{y}^\top z = \|z\|_2 - z^\top w/\|w\|_2 > 0$ since $w$ and $z$ are not parallel. Similarly, $\bar{y}^\top w < 0$ and $\bar{y}^\top (-z) < 0$. Furthermore, we have
	\[ d_* = \max_{y \neq 0} \min_{A \in \cA} \lbl(A) \cdot \frac{y^\top A}{\|y\|_*} = \max_{y \neq 0} \min\left\{ \frac{y^\top z}{\|y\|_*}, -\frac{y^\top (-z)}{\|y\|_*}, -\frac{y^\top (2w/c)}{\|y\|_*} \right\} \leq \frac{2}{c}. \]
	where the inequality follows from $\|z\| = 2/c$ and $\|w\| = 1$. We set the distribution $\bbP$ to pick each point in $\cA$ with equal probability.
	
	We consider the event when $A_0 = -z$, which occurs with probability $1/3$. On this event, $\lbl(A_0) = -1$, 
{and since \cref{alg:projected-perceptron} initializes $y_0 = 0$, $b_0=0$, it predicts the label to be $\plbl(r(A_0,y_0,b_0),y_0,b_0) = \sign(0)=+1$. The proxy is $s(A_0,y_0,b_0) = A_0 = -z$ hence the update is $y_1 = y_0 - s(A_0,y_0,0) = z$, $b_1 = 0$.}

	Once $y_1 = z$, notice that $\frac{y_1^\top z}{\|y_1\|_*} > 0$ and $\frac{y_1^\top (-z)}{\|y_1\|_*} < 0$. According to \eqref{eq:manipulated} we have $r(-z,y_1,0) = -z$. 
	Then, we have 
	\[y_1^\top r(-z,y_1,0) - 2\|y_1\|_*/c =-\|z\|_2^2 - 2\|z\|_*/c = -\|z\|_2^2 - \|z\| \|z\|_*  <0\] (as $z\neq0$) and so $\plbl(r(-z,y_1,0),y_1,0)=-1$ and $r(-z,y_1,0)$ is correctly classified.
	Moreover, using \eqref{eq:manipulated} and the fact that $b_1=0$, we deduce
	\begin{align*}
		y_1^\top r(z,y_1,0) - 2\|y_1\|_*/c 
		&= \begin{cases}
			y_1^\top z + \left( \frac{2}{c} - \frac{y_1^\top z}{\|y_1\|_*} \right) y_1^\top v(y_1) - 2\|y_1\|_*/c , &\text{if } 0 \leq \frac{y_1^\top z}{\|y_1\|_*} < \frac{2}{c} \\
			y_1^\top z - 2\|y_1\|_*/c, & \text{otherwise}\\
		\end{cases} \\
		&= \begin{cases}
			z^\top z + \left( \frac{2}{c} - \frac{z^\top z}{\|z\|_*} \right) z^\top v(z) - 2\|z\|_*/c , &\text{if } 0 \leq \frac{z^\top z}{\|z\|_*} < \frac{2}{c} \\
			z^\top z - 2\|z\|_*/c, & \text{otherwise}\\
		\end{cases} \\
		&= \begin{cases}
			z^\top z + \left( \frac{2}{c} \|z\|_* - z^\top z \right)  - 2\|z\|_*/c , &\text{if } 0 \leq \frac{z^\top z}{\|z\|_*} < \frac{2}{c} \\
			z^\top z - 2\|z\|_*/c, & \text{otherwise}\\
		\end{cases} \\
		&= \begin{cases}
			0 , &\text{if } 0 \leq \frac{z^\top z}{\|z\|_*} < \frac{2}{c} \\
			z^\top z - 2\|z\|_*/c, & \text{otherwise}\\
		\end{cases} \\
		& \geq 0.
	\end{align*} 
	Here, the second equality follows from $y_1=z$, and the third one from \eqref{eq:manipulation-direction}, and the last inequality holds as $z^\top z>0$ since $z\neq0$.
	Therefore, $y_1$ correctly classifies $r(z,y_1,0)$, as well.
On the other hand, for $2w/c$, we have $\frac{y_1^\top (2w/c)}{\|y_1\|_*} = \frac{2}{c} \frac{z^\top v(z)}{\|z\|_*} = \frac{2}{c}$. Thus, $r(2w/c,y_1,0) = 2w/c$, and $\plbl(r(2w/c,y_1,0),y_1,0) = +1$, which is a mistake. On the other hand, $s(2w/c,y_1,0) = r(2w/c,y_1,0) - \frac{2}{c} v(y_1) = \frac{2}{c} w - \frac{2}{c} v(z) = 0$. Thus, even though a mistake is made, the step in \cref{alg:projected-perceptron} to update the classifier will not result in any change in the classifier. 

	Hence, on the event that $A_0 = -z$, we have that \cref{alg:projected-perceptron} sets $y_t = z$, $b_t = 0$ for all $t \geq 1$. Thus, whenever $A_t = 2w/c$, we encounter a mistake, and $\bbP$ will ensure that $A_t = 2w/c$ infinitely often.
\end{proof}

Recall that when $b_*=0$ is known a priori, \cref{assum:margin-zero-b} holds, and also 
if $\|\cdot\|=\|\cdot\|_2$, then \cref{thm:perceptron-mistake-bound2} guarantees that \cref{alg:projected-perceptron} with $\bbL = \bbR^d \times \{0\}$ and $\gamma=1$ makes finitely many mistakes for any $d_*>0$. In the following example, we will show that in this setting even when $d_*>2/c$, \cref{alg:projected-perceptron} with $\L=\R^d\times\R$,  $\L=\R^d\times\{0\}$ or $\L = \R_+^d \times \R$ may result in infinitely many manipulations, and that $(y_t,b_t)$ may not converge to the maximum margin classifier $(y_*,b_*)$. This is in contrast to \cref{alg:data-driven}, for which \cref{thm:data-driven-convergence,thm:margin-best_manipulation-bound} ensure that finitely many manipulations occur whenever $d_*>2/c$ and $(y_t,b_t) \to (y_*/\|y_*\|_*, b_*/\|y_*\|_*)$ almost surely under suitable assumptions.
\begin{example}\label{ex:perceptron-margin}
Consider the same setting as \cref{ex:truthful-max-margin}. Since \cref{assum:margin} is satisfied with $b_*=0$, \cref{assum:margin-zero-b} is also satisfied with the same $y_*=(0,1)$ and $d_*=1$. Therefore, \cref{alg:projected-perceptron} with $\gamma=1$ and $\L=\R^d\times\R$ (recall that $2/c=1/2<1=d_*$),  $\L=\R^d\times\{0\}$ or $\L = \R_+^d \times \R$ may be applied to find a classifier $(\bar{y},\bar{b}) \neq (y_*,b_*)$ that correctly classifies all agents, and thus \cref{alg:projected-perceptron} never updates again. However, this classifier will always cause some point $A \in \cA^+$ to manipulate.
	\end{example}

	\begin{proof}[Proof of \cref{ex:perceptron-margin}]
	Starting from $y_0 = (0,0)$, we will compute the updates for $A_0 = (1,-1) \in \cA^-$ and $A_1 = (2,1) \in \cA^+$ and $\L = \bbR^d \times \bbR$.
	
	First, when $y_0 = (0,0)$, $\plbl(A_0,y_0,0) = \sign(0) = 1 \neq \lbl(A_0) = -1$,
	so a mistake is made. Then, the update is $y_1 = y_0 + \lbl(A_0)\, s(A_0,y_0,0) = -s(A_0,y_0,0) = -A_0 = (-1,1)$ and $b_1= \lbl(A_0) = -1$. Second, when $y_1 = (-1,1)$, we have $y_1^\top A_1+b_1=-2$ so $r(A_1,y_1,b_1)=s(A_1,y_1,b_1)=A_1$ and $\plbl(A_1,y_1,b_1) = \sign(-2 - 1/\sqrt{2}) = -1 \neq \lbl(A_1) = 1$, so the update is $y_2 = y_1 + \lbl(A_1)\, s(A_1,y_1, b_1) = y_1 + A_1 = (1,2)$ and $b_2 = b_1 + \lbl(A_1) = 0$.
	
	From \cref{ex:truthful-max-margin}, we know that the classifier given by $\plbl(A,y_2,b_2)$ where  $y_2 = (1,2)$ and $b_2=0$
	makes no mistakes on any $A \in \cA$, thus \cref{alg:projected-perceptron} will no longer update. However, $y_2 \neq y_*$, and the agent $A=(-1,1)$ will always manipulate.
	
	When $\bbL = \bbR^d \times 0$, we can easily check that $y_1 = (-1,1)$, $b_1 = 0$ and $y_2 = (1,2)$, $b_2 = 0$. Thus, the same conclusion holds.
	
	When $\bbL = \bbR_+^d \times \bbR$, we compute the update for $A_0 = (1,-1)$. Since $\lbl(A_0) = -1$, when $(y_0,b_0) = (0,0)$ we have $r(A_0,y_0,b_0) = A_0$ and $\plbl(A_0,y_0,b_0) = +1 \neq \lbl(A_0)$, so a mistake is made and we compute $y_1 = \Proj_{\bbR_+^d}(y_0 - A_0) = \Proj_{\bbR_+^d}((0,0)-(1,-1))= \Proj_{\bbR_+^d}((-1,1))=(0,1)$ and $b_1 = \lbl(A_0) = -1$. We can check that $y_1^\top A + b_1 = 0$ for all $A \in \cA^+$, and $y_1^\top A + b_1 = -2$ for all $A \in \cA^-$. Hence,  $r(A,y_1,b_1) = s(A,y_1,b_1) = A + \frac{2}{c} y_1$ for all $A \in \cA^+$, and $r(A,y_1,b_1) = A$ for all $A \in \cA^-$. However, we have that $\plbl(r(A,y_1,b_1),y_1,b_1) = \lbl(A)$ for all $A \in \cA$. Therefore, \cref{alg:projected-perceptron} never updates again, but each point $A \in \cA^+$ is manipulated since $r(A,y_1,b_1) = A + \frac{2}{c} y_1 \neq A$ for all $A\in\cA^+$.
\end{proof}

\section{Numerical Study}\label{sec:numerical}

In this section, we compare the performance of the three algorithms detailed in \cref{sec:algorithms}. We set the norm $\|\cdot\|=\|\cdot\|_2$ in \cref{assum:cost-norm} as our theoretical analysis for \cref{alg:data-driven-subgradient-averaging} is based on this cost structure. Since in general it is not clear whether the data may satisfy \cref{assum:non-negative,assum:margin-zero-b},
we implement \cref{alg:projected-perceptron} with $\L=\R^d\times\R$. 
The algorithms are implemented in Python 3.8.18 and tested on a server with a 12-core 12-thread 2.8-GHz processor and 64 GB memory. The best margin classifier and the problem \eqref{eq:data-driven} in \cref{alg:data-driven} are solved using MOSEK version 10.1.27 (via the Python Fusion API), where the numerical tolerance parameters \texttt{intpntCoTolPfeas}, \texttt{intpntCoTolRelGap}, \texttt{intpntTolRelGap}, \texttt{intpntTolStepSize}, and \texttt{intpntCoTolInfeas} for the conic problem solver are set to $10^{-12}$, and other configurations are set as default.

\paragraph{Data and preprocessing.} We tested the numerical performance of algorithms on both real and synthetic data. 
Here we present and discuss the results from the real data. See \cref{sec:experiment-synthetic} for synthetic data; the results largely align for both types of data. Our real data originates from the loan application records collected by an online platform named Prosper. Detailed information of applicants such as credit history length and bank card utilization are encoded as features, and the loan status serves as the label for binary classification. The dimension of this dataset is $d=6$. The subset of loan applicants' features used for classification follows from \citep{ghalme_strategic_2021}. This data set has $20,222$ data points and $41.70\%$ of the data points have $+1$ labels.

In order to ensure that our assumptions are met, we preprocess this data set as follows. To satisfy \cref{assum:separable} of separability and guarantee a positive margin of at least $\rho>0$ as in \cref{assum:margin}, we first obtain a support vector classifier on the original data, and then remove data points that are misclassified or within distance $\rho$ of the decision boundary, where $\rho\in\set{0.01, 0.02, 0.04}$. 
(The resulting data sets have respectively $16,580$, $16,271$, and $15,560$ points. Moreover, the proportion of $+1$ labels in the resulting datasets are $39.33\%$, $39.24\%$, $39.06\%$, respectively). For each preprocessed data set, we also compute the maximum margin classifier $(y_*,b_*)$ on the non-manipulated data, achieving a maximum margin of $d_*\geq\rho$.
This classifier serves as a benchmark against which we compare the iterates of our algorithms. We set the constant $c$ for unit cost of manipulation such that $2/c \in \set{0.8\rho, \rho, 1.2\rho}$. As a result, the condition $d_*>2/c$ is satisfied by some combinations of $\rho$ and $2/c$ while not satisfied by others.

\paragraph{Performance metrics.} 
We compare the performance of algorithms as the distance between $(y_*,b_*)$ and $(y_t,b_t)$ normalized by $y_*$ and $y_t$ repsectively, i.e., $\left\|\frac{(y_t,b_t)}{\|y_t\|_2} - \frac{(y_*,b_*)}{\|y_*\|_2}\right\|_2$ (\cref{fig:distance-real-data,fig:2rounds}), as well as the number of mistakes (\cref{tab:mistake-real-data-noiseless}) {and the number of manipulations (\cref{tab:manipulation-real-data-noiseless})}. In addition, we also consider the margin of $(y_t,b_t)$ on the entire dataset, i.e., $h(y_t,b_t;\cA^+,\cA^-)$ (\cref{fig:2rounds}). Since its trend is similar to the distance metric (see \cref{fig:2rounds}), we present only the case of $\rho=0.01$, $2/c=0.009$ for conciseness. 
We report the computation time of each algorithm in \cref{tab:time-real-data-noiseless}. {We also include additional figures that display how the number of mistakes and manipulations grow with the iterations in \cref{fig:mistake-real-data,fig:manipulation-real-data} in \cref{sec:experiment-fig}.}

\paragraph{Distance to best margin classifier.}
Examining \cref{fig:distance-real-data}, we see that in the cases where $\rho > 2/c$, \cref{alg:data-driven} quickly finds a $(y_t, b_t)$ close to $(y_*, b_*)$. 
This is in line with \cref{thm:data-driven-convergence} which asserts convergence of \cref{alg:data-driven} to $(y_*, b_*)$. 
\cref{alg:data-driven-subgradient-averaging} is designed to be a simplified and cost-efficient version of \cref{alg:data-driven}. Given this, numerically we also observe that \cref{alg:data-driven-subgradient-averaging} converges slower compared to \cref{alg:data-driven}.
On the other hand, the trend of \cref{alg:projected-perceptron} is much noisier and it is not clear if \cref{alg:projected-perceptron} will convergence to the best margin classifier even after $\approx15,000$ iterations. Note that non-convergence behavior of \cref{alg:projected-perceptron}
is in line with \cref{ex:perceptron-margin}.
Even though our theory does not cover the case of $\rho \leq 2/c$, we notice from \cref{fig:distance-real-data} that \cref{alg:data-driven} performs the best in this setting as well.
Notice that \cref{alg:data-driven} converges to $(y_*,b_*)$ to a high accuracy when $\rho=0.02$, whereas it is stuck at around $10^{-3}$ when $\rho=\set{0.01,0.04}$. This is because a critical data point that defines the best margin classifier arrives at an early stage when $(y_t,b_t)$ is yet far from $(y_*,b_*)$ and results in a manipulated reporting of the data point. As a result, the learner never gets to see this particular critical point, but only a rough estimate of it. Under our probabilistic model \cref{assum:data-driven-stochastic}, there will be future data points that are arbitrarily close to this critical point almost surely. However, as the time horizon in this experiment is fixed, {and the data points are visited only once,} the neighborhood of the critical point happens to be absent in the dataset.
To illustrate this, we input the same dataset for a second round for $\rho=0.01$, {$2/c=0.8 \rho$}.
\cref{fig:2rounds} shows that \cref{alg:data-driven} indeed converges to a high accuracy when the critical points are revisited {as we continue running the algorithms for a 2nd round on the same dataset}.

\paragraph{Number of mistakes.}
From \cref{tab:mistake-real-data-noiseless} we observe that, in any setting of $\rho$ and $2/c$, \cref{alg:data-driven} makes the fewest mistakes (9 over more than 15,000 data points) and \cref{alg:projected-perceptron} makes the most mistakes (up to {1017}). In general, as $\rho$ decreases,
\cref{alg:projected-perceptron} makes more mistakes and
its difference to \cref{alg:data-driven,alg:data-driven-subgradient-averaging} becomes more pronounced. By contrast, the impact of $\rho$ on \cref{alg:data-driven-subgradient-averaging} is not as clear. Finally, \cref{alg:data-driven} has a tendency of making all of its mistakes early on in the first 250 iterations and not making \emph{any}
mistakes afterwards. \cref{fig:mistake-real-data} in \cref{sec:experiment-fig} shows that, in contrast, the number of mistakes made by \cref{alg:projected-perceptron,alg:data-driven-subgradient-averaging} both grow with the number of iterations, with a noticeably faster growth rate for \cref{alg:projected-perceptron}.

\paragraph{Number of manipulations.}
Acoording to \cref{tab:manipulation-real-data-noiseless}, \cref{alg:data-driven} is especially good at encouraging truthfulness, leading to the fewest manipulation in all the cases, followed by \cref{alg:data-driven-subgradient-averaging} then \cref{alg:projected-perceptron}. \cref{fig:manipulation-real-data} in \cref{sec:experiment-fig} shows that more manipulations occur early on in \cref{alg:data-driven-subgradient-averaging} but eventually \cref{alg:projected-perceptron} overtakes it. As $2/c$ decreases, generally all three algorithms incur less manipulations, since the manipulation cost increases. Although our theoretical guarantees on finite manipulation does not cover $\rho<2/c$, the same pattern can be observed: \cref{alg:data-driven} incurs much less manipulation compared to \cref{alg:data-driven-subgradient-averaging,alg:projected-perceptron}.

\paragraph{CPU time.} \cref{tab:time} indicates that \cref{alg:projected-perceptron} is the fastest, \cref{alg:data-driven-subgradient-averaging} require more time, and \cref{alg:data-driven} is the most expensive, taking about $5$ times more computation time than \cref{alg:data-driven-subgradient-averaging}. This is expected as the updating rule of \cref{alg:projected-perceptron} involves only very simple calculation of constant complexity, whereas \cref{alg:data-driven-subgradient-averaging} requires finding a minimizer or maximizer over the growing sets $\widetilde{\cA}_t^+$ and $\widetilde{\cA}_t^-$. \cref{alg:data-driven} is even more expensive as it requires at each iteration solving an optimization problem with an increasing number of constraints. In fact, our implementation of \cref{alg:data-driven} has partly eased the computation difficulty by first checking if the new proxy point $s(A_t,y_t,b_t)$ decreases the objective $h(y_t,b_t;\widetilde{\cA}_t^+,\widetilde{\cA}_t^-)$ of \eqref{eq:data-driven} and accordingly deciding whether it is necessary to solve \eqref{eq:data-driven}. Without this efficient implementation, the running time of \cref{alg:data-driven} can be up to $2\times10^3$ seconds, which is orders of magnitude larger compared to \cref{alg:projected-perceptron,alg:data-driven-subgradient-averaging}. Note that even with this efficient implementation, in the worst case, there can be adversarial order of data point arrivals such that \eqref{eq:data-driven} has to be solved at almost every iteration, making it possible for \cref{alg:data-driven} to be prohibitively expensive.

\paragraph{Additional results.} \cref{sec:experiment-supp} contains supplementary numerical results. We tested the effect of adding noise to agent responses (\cref{sec:experiment-noise}) and ran experiments on synthetic data (\cref{sec:experiment-synthetic}). While our theory does not apply to noisy agent responses, this reflects real-world settings where agents may not be fully rational or the learner observations are inaccurate. We found that adding noise affects the performance of \cref{alg:data-driven}, while \cref{alg:data-driven-subgradient-averaging,alg:projected-perceptron} are relatively robust to noise. This is worthy of investigation in future work. In our synthetic data experiments, we generated data such that \cref{assum:margin-zero-b} holds and provided this information to \cref{alg:projected-perceptron} by selecting $\L=\R^d\times\{0\}$; note that this information is not provided to \cref{alg:data-driven,alg:data-driven-subgradient-averaging} in these experiments. Even then, our synthetic data experiment conclusions are consistent with the ones we presented here.

\begin{landscape}
\begin{figure}
  \centering
  \includegraphics[scale=.9]{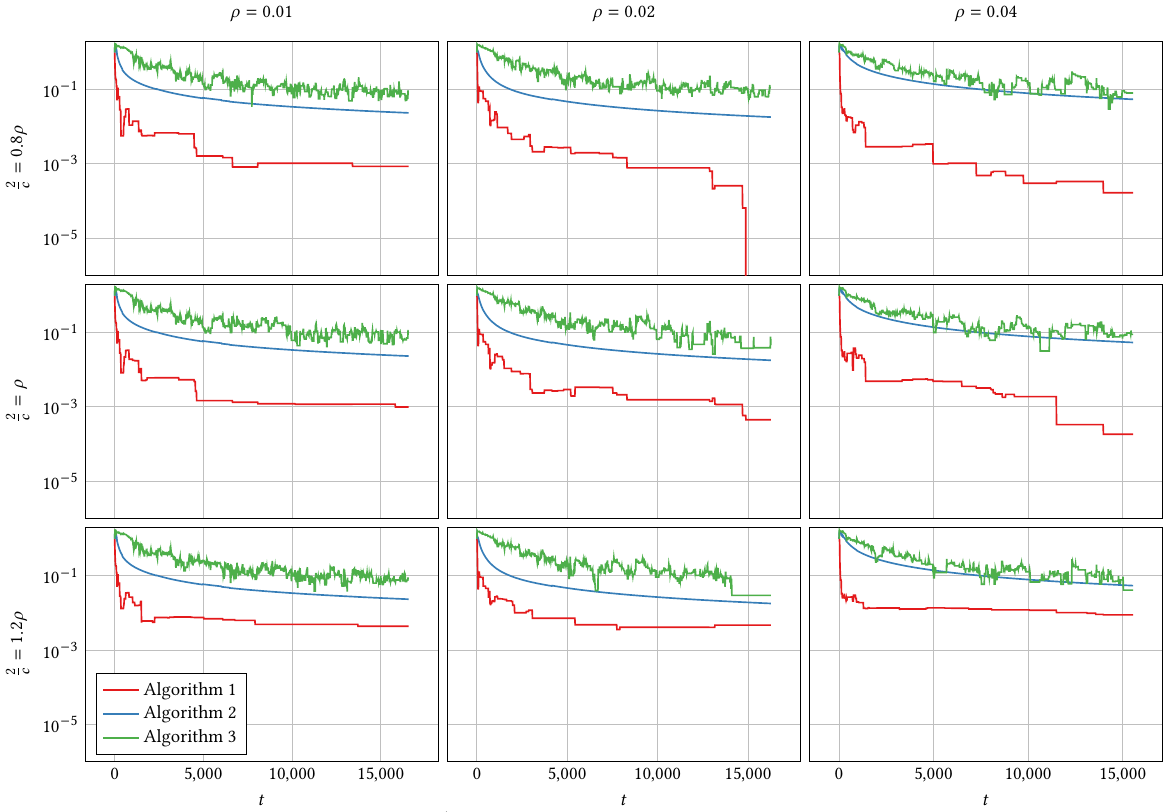}
  \caption{Distance $\left\|\frac{(y_t,b_t)}{\|y_t\|_2} - \frac{(y_*,b_*)}{\|y_*\|_2}\right\|_2$ between $(y_*,b_*)$ and $(y_t,b_t)$ normalized by $y_*$ and $y_t$ respectively for \cref{alg:projected-perceptron,alg:data-driven,alg:data-driven-subgradient-averaging} on loan data, with different margins $\rho\in\{0.01, 0.02, 0.04\}$, $2/c\in\{0.8\rho, \rho, 1.2\rho\}$, and $\sigma=0$.}
  \label{fig:distance-real-data}
  
\end{figure}
\end{landscape}

\begin{figure}
	\centering
  \includegraphics[scale=.9]{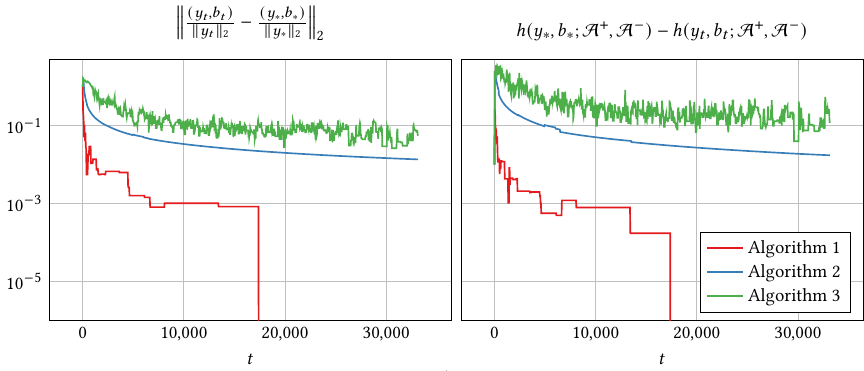}
  \caption{Distance $\left\|\frac{(y_t,b_t)}{\|y_t\|_2} - \frac{(y_*,b_*)}{\|y_*\|_2}\right\|_2$ and margin $h(y_*,b_*;\cA^+,\cA^-) - h(y_t,b_t;\cA^+,\cA^-)$ for \cref{alg:projected-perceptron,alg:data-driven,alg:data-driven-subgradient-averaging} on loan data for two rounds, with $\rho=0.01$, $2/c=0.8\rho$, and $\sigma=0$.}
\label{fig:2rounds}
  
\end{figure}

\begin{table}[htbp]
  \centering
\footnotesize
  \begin{tabular}{c|r|rrr|rrr|rrr}
\toprule
    \multirow{2}{*}{\centering $\frac2c$} & \multirow{2}{*}{\centering $t$} & \multicolumn{3}{c|}{$\rho=0.01$} & \multicolumn{3}{c|}{$\rho=0.02$} & \multicolumn{3}{c}{$\rho=0.04$} \\
& & Alg \ref{alg:data-driven} & Alg \ref{alg:data-driven-subgradient-averaging} & Alg \ref{alg:projected-perceptron} & Alg \ref{alg:data-driven} & Alg \ref{alg:data-driven-subgradient-averaging} & Alg \ref{alg:projected-perceptron} & Alg \ref{alg:data-driven} & Alg \ref{alg:data-driven-subgradient-averaging} & Alg \ref{alg:projected-perceptron} \\ 
    \midrule
    \multirow{2}{*}{$0.8\rho$} &  
        250 &     9 &    37 &    37 &     9 &    28 &    35 &     9 &    31 &    32 \\
    & 15000 &     9 &   308 &   985 &     9 &   130 &   861 &     9 &   319 &   631 \\
    \midrule
    \multirow{2}{*}{$1.0\rho$} &  
        250 &     9 &    37 &    37 &     9 &    28 &    35 &     9 &    31 &    32 \\
    & 15000 &     9 &   309 &   954 &     9 &   130 &   799 &     9 &   311 &   630 \\
    \midrule
    \multirow{2}{*}{$1.2\rho$} &  
        250 &     9 &    37 &    37 &     9 &    28 &    35 &     9 &    31 &    32 \\
    & 15000 &     9 &   310 &  1017 &     9 &   128 &   823 &     9 &   311 &   625 \\
    \bottomrule \end{tabular}
\caption{Number of mistakes made by \cref{alg:projected-perceptron,alg:data-driven,alg:data-driven-subgradient-averaging} on loan data, with different margins  $\rho\in\{0.01, 0.02, 0.04\}$, $2/c\in\{0.8\rho, \rho, 1.2\rho\}$, and $\sigma=0$.}
  \label{tab:mistake-real-data-noiseless}
\end{table}

\begin{table}[htbp]
  \centering
\footnotesize
  \begin{tabular}{c|r|rrr|rrr|rrr}
\toprule
    \multirow{2}{*}{\centering $\frac2c$} & \multirow{2}{*}{\centering $t$} & \multicolumn{3}{c|}{$\rho=0.01$} & \multicolumn{3}{c|}{$\rho=0.02$} & \multicolumn{3}{c}{$\rho=0.04$} \\
& & Alg \ref{alg:data-driven} & Alg \ref{alg:data-driven-subgradient-averaging} & Alg \ref{alg:projected-perceptron} & Alg \ref{alg:data-driven} & Alg \ref{alg:data-driven-subgradient-averaging} & Alg \ref{alg:projected-perceptron} & Alg \ref{alg:data-driven} & Alg \ref{alg:data-driven-subgradient-averaging} & Alg \ref{alg:projected-perceptron} \\ 
    \midrule
    \multirow{2}{*}{$0.8\rho$} &  
        250 &     0 &     0 &     0 &     2 &     1 &     0 &     2 &     3 &     6 \\
    & 15000 &     2 &    70 &    83 &     4 &    63 &   155 &     4 &   175 &   241 \\
    \midrule
    \multirow{2}{*}{$1.0\rho$} &
        250 &     0 &     1 &     0 &     2 &     3 &     0 &     3 &     5 &     6 \\
    & 15000 &     9 &    96 &   102 &    13 &    81 &   184 &    15 &   243 &   312 \\
    \midrule
    \multirow{2}{*}{$1.2\rho$} &
        250 &     0 &     2 &     0 &     4 &     3 &     0 &     4 &     7 &     6 \\
    & 15000 &    45 &   121 &   131 &    97 &    94 &   209 &   197 &   304 &   355 \\
    \bottomrule \end{tabular}
\caption{Number of manipulations caused by \cref{alg:projected-perceptron,alg:data-driven,alg:data-driven-subgradient-averaging} on loan data, with different margins  $\rho\in\{0.01, 0.02, 0.04\}$, $2/c\in\{0.8\rho, \rho, 1.2\rho\}$, and $\sigma=0$.}
  \label{tab:manipulation-real-data-noiseless}
\end{table}

\begin{table}[htbp]
  \centering
  \footnotesize
  \begin{tabular}{c|r|rrr|rrr|rrr}
\toprule
    \multirow{2}{*}{\centering $\frac2c$} & \multirow{2}{*}{\centering $t$} & \multicolumn{3}{c|}{$\rho=0.01$} & \multicolumn{3}{c|}{$\rho=0.02$} & \multicolumn{3}{c}{$\rho=0.04$} \\
& & Alg \ref{alg:data-driven} & Alg \ref{alg:data-driven-subgradient-averaging} & Alg \ref{alg:projected-perceptron} & Alg \ref{alg:data-driven} & Alg \ref{alg:data-driven-subgradient-averaging} & Alg \ref{alg:projected-perceptron} & Alg \ref{alg:data-driven} & Alg \ref{alg:data-driven-subgradient-averaging} & Alg \ref{alg:projected-perceptron} \\
    \midrule
    \multirow{2}{*}{$0.8\rho$} &
        250 &  0.54 &  0.07 &  0.01 &  0.54 &  0.06 &  0.01 &  0.51 &  0.06 &  0.01 \\
    & 15000 & 24.91 &  5.63 &  0.47 & 26.49 &  5.58 &  0.46 & 25.78 &  5.60 &  0.46 \\
    \midrule
    \multirow{2}{*}{$1.0\rho$} &
        250 &  0.54 &  0.07 &  0.01 &  0.54 &  0.05 &  0.01 &  0.51 &  0.05 &  0.01 \\
    & 15000 & 24.93 &  5.59 &  0.47 & 26.66 &  5.56 &  0.47 & 26.57 &  5.59 &  0.47 \\
    \midrule
    \multirow{2}{*}{$1.2\rho$} &
        250 &  0.54 &  0.07 &  0.01 &  0.53 &  0.05 &  0.01 &  0.51 &  0.06 &  0.01 \\
    & 15000 & 25.32 &  5.58 &  0.47 & 25.26 &  5.86 &  0.46 & 26.28 &  5.61 &  0.47 \\
    \bottomrule \end{tabular}
  \caption{CPU running time (in seconds) of \cref{alg:projected-perceptron,alg:data-driven,alg:data-driven-subgradient-averaging} on loan data, with different margins  $\rho\in\{0.01, 0.02, 0.04\}$, $2/c\in\{0.8\rho, \rho, 1.2\rho\}$, and $\sigma=0$.}
  \label{tab:time-real-data-noiseless}
\end{table}

\bibliographystyle{plainnat}
\bibliography{ref}

\newpage
\begin{appendix}

\section{Proofs of auxiliary results}

\begin{lemma}\label{lem:subdiff-dual-norm}
  The subdifferential set of the function $\|\cdot\|_*$ has the following characterizations:
\begin{align*} 
      \partial\|y\|_* 
      &= \{v\in\R^d: v^\top y = \|y\|_*, \|v\|\leq1\} \\
      &= \argmax_{v\in\R^d}\{v^\top y:\|v\|\leq1\}. \end{align*}
  \end{lemma}

  \begin{proof}[Proof of \cref{lem:subdiff-dual-norm}]
  Fix $y\in\R^d$. By definition of the subdifferential, we have
  \begin{align*}
      \partial\|y\|_* = \set{v\in\R^d:~ \|y'\|_* \geq \|y\|_* + v^\top (y'-y),\ \forall y'\in\R^d}. 
  \end{align*}
  Let us define the set $\Delta := \set{v\in\R^d:~ y^\top v = \|y\|_*,~ \|v\|\leq1}$. 
First, for any $v\in \Delta$ and any $y'\in\R^d$,  using $1\geq \|v\|$ and H\"{o}lder's inequality we arrive at
  \begin{align*}
      \|y'\|_* \geq \|v\|\|y'\|_* \geq v^\top y' = v^\top y + v^\top (y'-y)= \|y\|_* + v^\top (y'-y), 
  \end{align*}
  and so $v\in\Delta$ implies $v\in \partial\|y\|_* $.
  
  Conversely, for any $v\in\partial\|y\|_*$ and any $y'\in\R^d$, from the subdifferential inequality we arrive at
  \begin{align*}
      \|y'\|_* - v^\top y' \geq \|y\|_* - v^\top y. 
  \end{align*}
As this holds for all $y'\in\R^d$, we conclude
  \begin{align*}
      \inf_{y'} \set{ \|y'\|_* - v^\top y' } 
\geq \|y\|_* - v^\top y.
  \end{align*}
As the right-hand side is a finite number, we must have $ \inf_{y'} \set{ \|y'\|_* - v^\top y' }>-\infty$. Note that whenever $\|v\|>1$, we have $ \inf_{y'} \set{ \|y'\|_* - v^\top y' }=-\infty$. So, we conclude $\|v\|\leq1$, and in this case we have $ \inf_{y'} \set{ \|y'\|_* - v^\top y' }=0$.
This implies $0 \geq \|y\|_* - v^\top y \geq \|v\|\|y\|_* - v^\top y \geq 0$, where the second inequality follows from $1\ge \|v\|$ and the last one from H\"{o}lder's inequality. Thus, each inequality in this chain must hold as equality, and $v^\top y=\|y\|_*$. Thus, we have shown $v\in\partial\|y\|_*$ implies $v\in \Delta$. 
  
  To show the last equivalent expression, it suffices to notice that $v^\top y\leq\|y\|_*$ for $\|v\|\leq1$. 
\end{proof}

\begin{lemma}\label{lem:cone-projection}
	Suppose $K\subseteq\R^m$ is a closed convex cone. Then, the map $\Proj_K(\cdot)$, i.e., the projection operation onto the cone $K$ is positively homogeneous. That is, for any {$\gamma \geq 0$,}
	\begin{align*}
		\Proj_K(\gamma x) = \gamma \Proj_K(x) \qquad\forall x\in\R^m. 
	\end{align*}
\end{lemma}
\begin{proof}[Proof of \cref{lem:cone-projection}]
	The stated relation clearly holds when $\gamma=0$ as $K$ is a closed convex cone. So, we assume $\gamma>0$.
	By definition, the projection operation is given by
	\begin{align*}
		\Proj_K(x) 
		&= \argmin_{y\in K} \|x-y\|_2. 
	\end{align*}
	Thus, {for any $\gamma>0$, }
	\begin{align*}
		\Proj_K(\gamma x)
		&= \argmin_{y\in K} \|\gamma x-y\|_2 
		= \argmin_{\gamma z\in K} \|\gamma x-\gamma z\|_2 
		= \gamma
		\argmin_{z\in K} \|x-z\|_2
		= \gamma
		\Proj_K(x), 
	\end{align*}
	where the second equation follows from the fact that $\gamma>0$ and $K$ is a closed convex cone.
\end{proof}

\begin{remark}\label{rem:resp_invariance_to_classifier_pos_scaling}
Note that under \cref{assum:unique-direction}, the response and proxy data functions are invariant to positive scalings of a classifier. That is, consider the classifiers given by $(y,b)$ and $(\bar y,\bar b)$ such that  $(\bar y,\bar b)=\gamma (y,b)$ for some $\gamma>0$. Then, for any point $A\in\cA$, we have $r(A,y,b)= r(A,\bar y,\bar b)$ and $s(A,y,b)= s(A,\bar y,\bar b)$.
		
To see this, observe that 
\[
\frac{\bar{y}^\top A+\bar{b}}{\|\bar{y}\|_*} = \frac{\gamma (y^\top A+b)}{\gamma \|y\|_*} = \frac{y^\top A+b}{\|y\|_*}. 
\]
Then, the if conditions in $r(\cdot)$ and $s(\cdot)$ will be satisfied precisely at the same time for both of the classifiers. 
Moreover, by definition of $v(\cdot)$ in \eqref{eq:manipulation-direction} and as $\gamma>0$, we have  $v(\bar{y})=v(y)$ under \cref{assum:unique-direction}. Thus, we conclude that $r(A,y,b)= r(A,\bar y,\bar b)$ and $s(A,y,b)= s(A,\bar y,\bar b)$ hold as well.
\end{remark}
	
\begin{lemma}\label{lem:perceptron-stepsize_invariance}
	Suppose \cref{assum:unique-direction} holds and $\L$ is a closed convex cone.
	Given a sequence of points $\set{A_t}_{t\in[T]}$, let $(y_t,b_t)$ be the sequence of classifiers obtained from \cref{alg:projected-perceptron} when we set $\gamma=1$. For the same sequence of points, let $(y_t', b_t')$ be the sequence of classifiers obtained from \cref{alg:projected-perceptron} for a fixed $\gamma>0$. Then, $(y_t',b_t') = \gamma (y_t,b_t)$, {$r(A_t,y_t',b_t')=r(A_t,y_t,b_t)$} and $s(A_t,y_t',b_t')=s(A_t,y_t,b_t)$ for all $t\in[T]$.
\end{lemma}
\begin{proof}[Proof of \cref{lem:perceptron-stepsize_invariance}]
	To ease our notation, we let $q_t\coloneq(y_t,b_t)$, $q_t'\coloneq(y_t',b_t')$, $\ell_t\coloneq\lbl(A_t)$, $r_t\coloneq r(A_t,y_t,b_t)$, $r_t'\coloneq r(A_t,y_t',b_t')$, $\xi_t\coloneq(s(A_t,y_t,b_t),1)$ and $\xi_t'\coloneq(s(A_t,y_t',b_t'),1)$. 
	
	We will prove that $q_t'=\gamma q_t$, $r_t'=r_t$, and $\xi_t'=\xi_t$ holds for all $t$ by induction. The base case holds as $q_0' = q_0 =0$. 
	Suppose by induction hypothesis that $q_t' = \gamma q_t$ for some $t$. 
	Then, by \cref{rem:resp_invariance_to_classifier_pos_scaling}, we have $r_t'=r_t$ and $\xi_t'=\xi_t$. 
	Let $z_t'$ be computed based on \eqref{eq:projected-gd-stepsize} for the classifier $q_t'$. Thus,
	\begin{align*}
		q_{t+1}' &= \begin{cases}
			\Proj_{\L}( q_t' + \gamma \ell_t \cdot \xi_t'), & \text{if } \plbl(r_t',y_t',b_t') \neq \ell_t, \\
			\Proj_{\L}( q_t' ), & \text{otherwise} \\
		\end{cases} \\
		&= \begin{cases}
			\gamma \Proj_{\L}( q_t + \ell_t \cdot \xi_t), & \text{if } \plbl(r_t',y_t',b_t') \neq \ell_t, \\
			\gamma \Proj_{\L}( q_t), & \text{otherwise} \\
		\end{cases} \\
		&= \gamma q_{t+1}, 
	\end{align*}
	where the first equation follows by definition and the second equation follows from the induction hypothesis that $q_t'=\gamma q_t$, $\xi_t'=\xi_t$ and applying \cref{lem:cone-projection}. In order to justify the third equation, note that as $r_t'=r_t$ and $(y_t',b_t')=\gamma (y_t,b_t)$, we have
	\begin{align*}
		\plbl(r_t',y_t',b_t') = \plbl(r_t,\gamma y_t,\gamma b_t) 
		= \sign\left( \gamma (y_t^\top r_t+b_t) - \tfrac{2\gamma \|y_t\|_*}{c} \right) 
		&= \sign\left( y_t^\top r_t+b_t - \tfrac{2 \|y_t\|_*}{c} \right) \\
		&= \plbl(r_t,y_t,b_t), 
	\end{align*}
	where the third equation follows from $\gamma>0$.
\end{proof}

\begin{lemma}\label{lem:l2-norm-symmetry}
{If a norm $\|\cdot\|$ is different to any positive multiple of $\|\cdot\|_2$}, then there exists $y,v\in\R^d$ such that $v\in\partial\|y\|_*$ and $y$, $v$ are not parallel. 
\end{lemma}
\begin{proof}[Proof of \cref{lem:l2-norm-symmetry}]
	Assume for contradiction that $\partial\|y\|_*=\{y/\|y\|\}$ for any $y\neq0$. Then, since $\partial\|y\|_*$ is a singleton for any $y\neq0$, we deduce that it is differentiable at any $y\neq0$.
{Since $\|\cdot\|$ is not a positive multiple of $\|\cdot\|_2$, the same must hold for the dual norm. Then, consider any $w$ such that $\|w\|_* = 1$, and let $\alpha := \|w\|_2$. The sets $\{y : \|y\|_*=1\}$ and $\{y : \|y\|_2 = \alpha\}$ must be different, therefore there exists $w'$ such that $\|w'\|_* = 1$ but $\|w'\|_2 \neq \alpha$. We let $y_0$ be the point with the smaller $\ell_2$-norm, and $y_1$ be the point with the larger $\ell_2$-norm, i.e., we have $\|y_0\|_* = \|y_1\|_* = 1$ and $\|y_0\|_2 < \|y_1\|_2$.}
In particular, this implies $y_0,\ y_1$ cannot be parallel, hence $ty_1+(1-t)y_0\neq0$ for $t\in[0,1]$. Therefore, we can define the following parameterized curve on the dual norm ball: 
	\begin{align*}
		y(t) = \frac{ty_1+(1-t)y_0}{\|ty_1+(1-t)y_0\|_*},\quad t\in[0,1]. 
	\end{align*}
	Then, we have $y(0)=y_0$, $y(1)=y_1$ (as $\|y_0\|_*=\|y_1\|_*=1$) and $y(t)$ is differentiable (as $y_0,\ y_1\in\R^d\setminus\{0\}$ we have $ty_1+(1-t)y_0\neq0$ for any $t\in[0,1]$), and $\|y(t)\|_*=1$ holds for all $t\in[0,1]$. Next, define $f(t)\coloneqq\frac12\|y(t)\|_2^2$, then using  $\|y_0\|_2<\|y_1\|_2$ we arrive at
	\begin{align*}
		0 
		< \frac12\|y_1\|_2^2 - \frac12\|y_0\|_2^2 
		= f(1) - f(0)
		= \int_0^1 f'(t)\mathrm{d}t 
		= \int_0^1 y(t)^\top y'(t)\mathrm{d}t. 
	\end{align*}
	Now, recall that $\|y(t)\|_*=1$ holds for all $t\in[0,1]$, and thus for any $t,\ t_0\in[0,1]$, we have
	\[ 0=\|y(t)\|_*-\|y(t_0)\|_*\geq \left(\frac{y(t_0)}{\|y(t_0)\|} \right)^\top(y(t)-y(t_0)),\] 
	where the inequality holds by the subgradient inequality applied to $\|\cdot\|_*$  and the assumption that $\partial\|y(t_0)\|_*=\{y(t_0)/\|y(t_0)\|\}$ (recall that $y(t_0)\neq0$ as $\|y(t_0)\|_*=1$ for all $t_0\in[0,1]$). 
	Dividing both sides by $t-t_0>0$ and letting $t\to t_0$, we obtain $y(t_0)^\top y'(t_0)\leq0$. {Since this holds for any $t_0 \in [0,1]$, this contradicts
	\[ \int_0^1 y(t)^\top y'(t) \mathrm{d}t > 0. \]
	Thus, there must exist some $y \neq 0$ such that $\partial \|y\|_* \neq \{y/\|y\|\}$.
	}
\end{proof}

\newpage
\section{Supplementary numerical results}\label{sec:experiment-supp}

\subsection{Real data with noisy agent responses}
\label{sec:experiment-noise}\label{sec:experiment-fig}

In this appendix, we investigate performance of the algorithms when noise is added to the agent responses. While we do not have theory for this setting we investigate the tolerance/robustness of our algorithms when the agent responses that the learner observes contain noise, which reflects the real world where agents may not be purely rational or our observations of their responses may be imperfect.
We assume that instead of the learner observing directly the agent's response $r(A_t,y_t,b_t)$, they observe $r(A_t,y_t,b_t) + \varepsilon_t$, where $\varepsilon_t \sim \mathcal{N}(0, \sigma^2 I_d)$ is i.i.d. Gaussian noise. Using the same dataset from~\cref{sec:numerical}, we tested the impact of having $\sigma\in\{0, 10^{-3},10^{-2}\}$. Note that $\sigma=0$ corresponds to the noiseless setting which was discussed in \cref{sec:numerical}.

\paragraph{Distance to best margin classifier with noise.} Comparing \cref{fig:distance-real-data,fig:distance-real-data-noise1e-3,fig:distance-real-data-noise1e-2}, we observe that \cref{alg:data-driven,alg:data-driven-subgradient-averaging} are more sensitive to noise $\sigma$. In particular, performance in terms of convergence to the classifier $(y_*,b_*)$ of \cref{alg:data-driven} drastically deteriorates as the noise increases to $\sigma=10^{-2}$. This is expected because the noise means that the proxy data are no longer guaranteed to be separable, and in fact inseparability becomes more likely as the noise level $\sigma$ increases. When inseparability occurs, problem~\eqref{eq:data-driven} in \cref{alg:data-driven} will give a solution of $(y_{t+1},b_{t+1})=(0,0)$ resulting in a prediction of $\plbl(r(A_{t+1},y_{t+1},b_{t+1}),y_{t+1},b_{t+1})=\sign(0)=+1$ always. As a result, \cref{alg:data-driven} breaks down and does not provide any useful information about $(y_*,b_*)$ once the separability assumption is violated, {e.g., when the noise level $\sigma$ is of the same order as $\rho$ and $2/c$ is large relative to $\rho$}.
{In contrast, \cref{alg:data-driven-subgradient-averaging} does not seem to encounter this issue even when data may be inseparable, and provides a reasonable estimate of $(y_*,b_*)$. It shows notably better performance than \cref{alg:projected-perceptron}.}
In fact, \cref{fig:distance-real-data-noise1e-3} shows that when $\sigma=10^{-3}$, it behaves mostly the same as in the noiseless case. As $\sigma$ further increases to $10^{-2}$, however, the performance of \cref{alg:data-driven-subgradient-averaging} deteriorates slightly as well. Finally, \cref{alg:projected-perceptron} seems to be robust to the noise, achieving roughly the same (yet not very good) performance as in the noiseless setting even for $\sigma=10^{-2}$.

\paragraph{Number of mistakes with noise.} \cref{fig:mistake-real-data,fig:mistake-real-data-noise1e-2,fig:mistake-real-data-noise1e-3,tab:mistake-bound-real-data} compare the performance of the algorithms with noise. In terms of number of mistakes, the performance of \cref{alg:data-driven} diminishes {when the data $\widetilde{\cA}_t^+$ and $\widetilde{\cA}_t^-$ become inseparable} as $(y_{t+1},b_{t+1})=(0,0)$ amounts to always predicting a label of $+1$ for any data point. In the cases where the proxy data are still separable, \cref{alg:data-driven} remains the best and makes the fewest mistakes. \cref{alg:data-driven-subgradient-averaging,alg:projected-perceptron} are more robust to noise and inseparability.
From \cref{tab:mistake-bound-real-data} we see that as $\sigma$ increases the performance degradation (in terms of number of mistakes) of \cref{alg:projected-perceptron} appears to be very minor compared to the other two algorithms.
{That said, from \cref{fig:mistake-real-data-noise1e-3,fig:mistake-real-data-noise1e-2} we see that \cref{alg:data-driven-subgradient-averaging} typically makes less mistakes than \cref{alg:projected-perceptron} in all parameter and noise settings.}
The robustness of \cref{alg:projected-perceptron} is somewhat unsurprising, since finite mistake bounds {for non-separable data} exist for the perceptron in the non-strategic setting.
That said, the robustness of \cref{alg:data-driven-subgradient-averaging} is somewhat surprising, and warrants further investigation in future work. One possible explanation is that, even when the proxy data $s(A_t,y_t,b_t)$ are inseparable, $(y_{t+1},b_{t+1})$ may be updated in  \cref{alg:data-driven-subgradient-averaging}, whereas \cref{alg:data-driven} will be stuck at $(y_{t+1},b_{t+1})=0$.

\paragraph{Number of manipulations with noise.} In \cref{fig:manipulation-real-data,fig:manipulation-real-data-noise1e-2,fig:manipulation-real-data-noise1e-3,tab:manipulation-real-data}, we compare the performance of the algorithms in terms of number of manipulations they induce under noise. Note that the results for \cref{alg:data-driven} are meaningful only when the dataset remains separable. Otherwise, no manipulation will happen once $(y_t,b_t)=(0,0)$, but this is not really an indication of better performance. For \cref{alg:data-driven-subgradient-averaging,alg:projected-perceptron}, we do not see drastic changes in the number of manipulations depending on the noise level, and in some cases there are even fewer manipulations when noise is present.

\paragraph{Solution times with noise.} \cref{tab:time} shows very similar performance across all three algorithms when compared to the noiseless case.

\begin{landscape}
	\begin{figure}
		\centering
		\includegraphics[scale=.9]{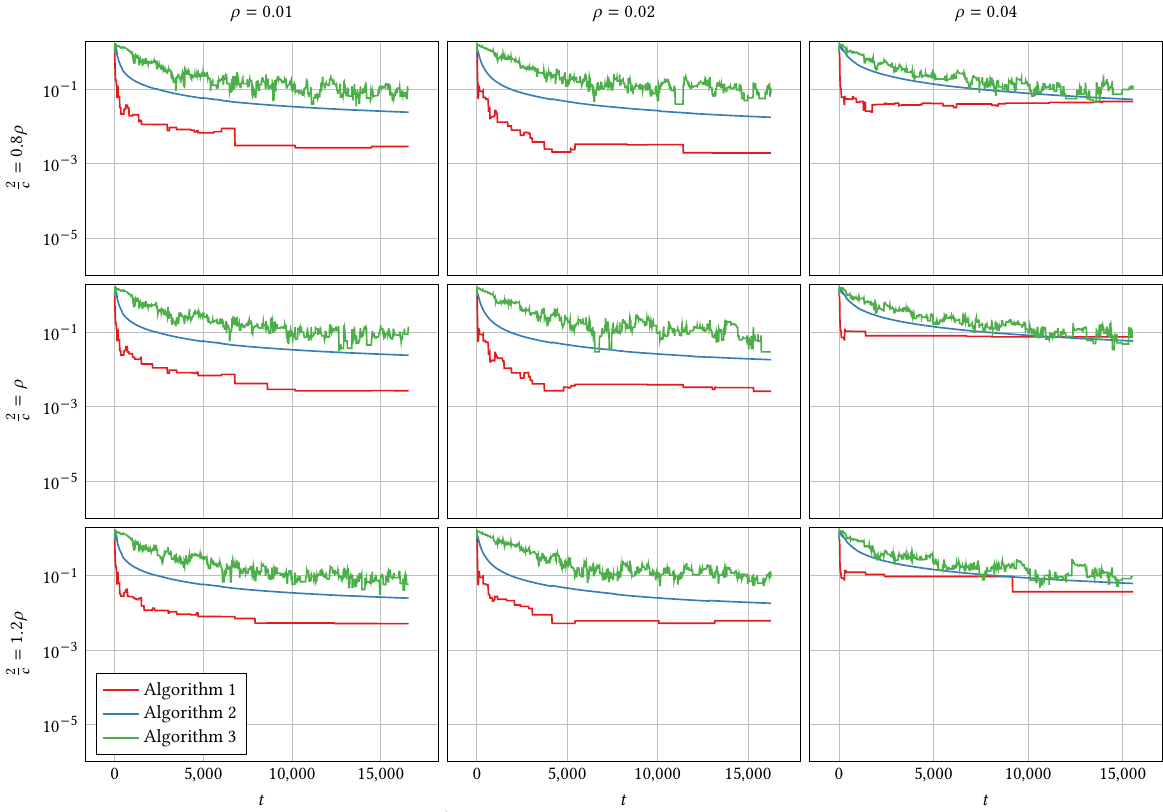}
		\caption{Distance $\left\|\frac{(y_t,b_t)}{\|y_t\|_2} - \frac{(y_*,b_*)}{\|y_*\|_2}\right\|_2$ between $(y_*,b_*)$ and $(y_t,b_t)$ normalized by $y_*$ and $y_t$ respectively for \cref{alg:projected-perceptron,alg:data-driven,alg:data-driven-subgradient-averaging} on loan data, with different margins $\rho\in\{0.01, 0.02, 0.04\}$ and $2/c\in\{0.8\rho, \rho, 1.2\rho\}$ and agent response noise level $\sigma=10^{-3}$.}
		\label{fig:distance-real-data-noise1e-3}
		
	\end{figure}
\end{landscape}

\begin{landscape}
	\begin{figure}
		\centering
		\includegraphics[scale=.9]{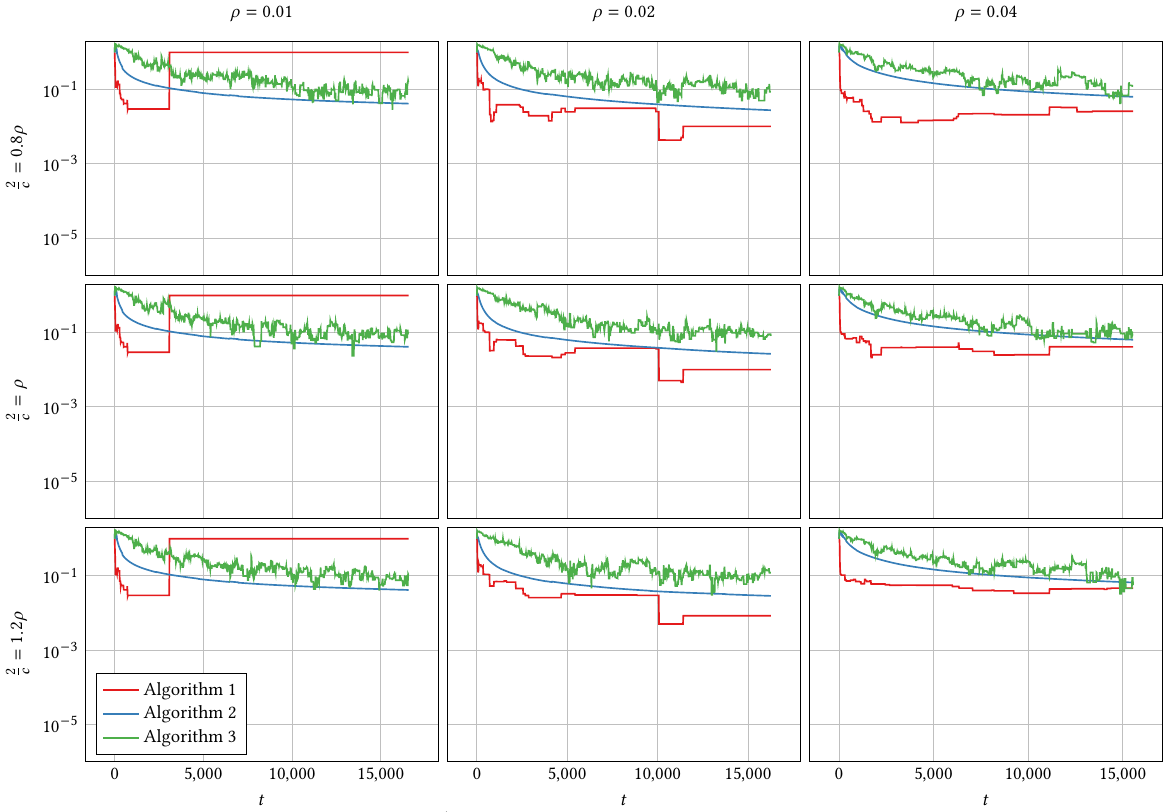}
		\caption{Distance $\left\|\frac{(y_t,b_t)}{\|y_t\|_2} - \frac{(y_*,b_*)}{\|y_*\|_2}\right\|_2$ between $(y_*,b_*)$ and $(y_t,b_t)$ normalized by $y_*$ and $y_t$ respectively for \cref{alg:projected-perceptron,alg:data-driven,alg:data-driven-subgradient-averaging} on loan data, with different margins $\rho\in\{0.01, 0.02, 0.04\}$ and $2/c\in\{0.8\rho, \rho, 1.2\rho\}$ and agent response noise level $\sigma=10^{-2}$.}
		\label{fig:distance-real-data-noise1e-2}
		
	\end{figure}
\end{landscape}

\begin{landscape}
	\begin{figure}
		\centering
		\includegraphics[scale=.9]{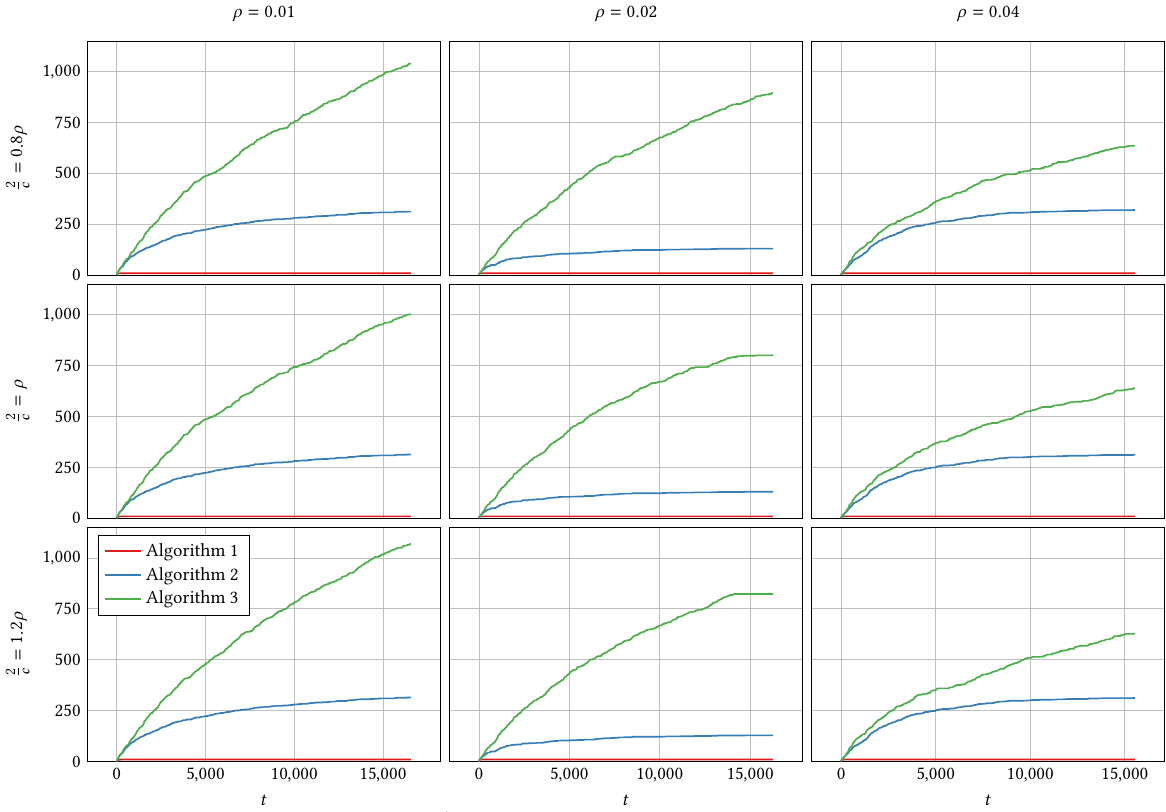}
		\caption{Number of mistakes made by \cref{alg:projected-perceptron,alg:data-driven,alg:data-driven-subgradient-averaging} on loan data, with different margins $\rho\in\{0.01, 0.02, 0.04\}$ and $2/c\in\{0.8\rho, \rho, 1.2\rho\}$ and agent response noise level $\sigma=0$.}
		\label{fig:mistake-real-data}
		
	\end{figure}
\end{landscape}

\begin{landscape}
	\begin{figure}
		\centering
		\includegraphics[scale=.9]{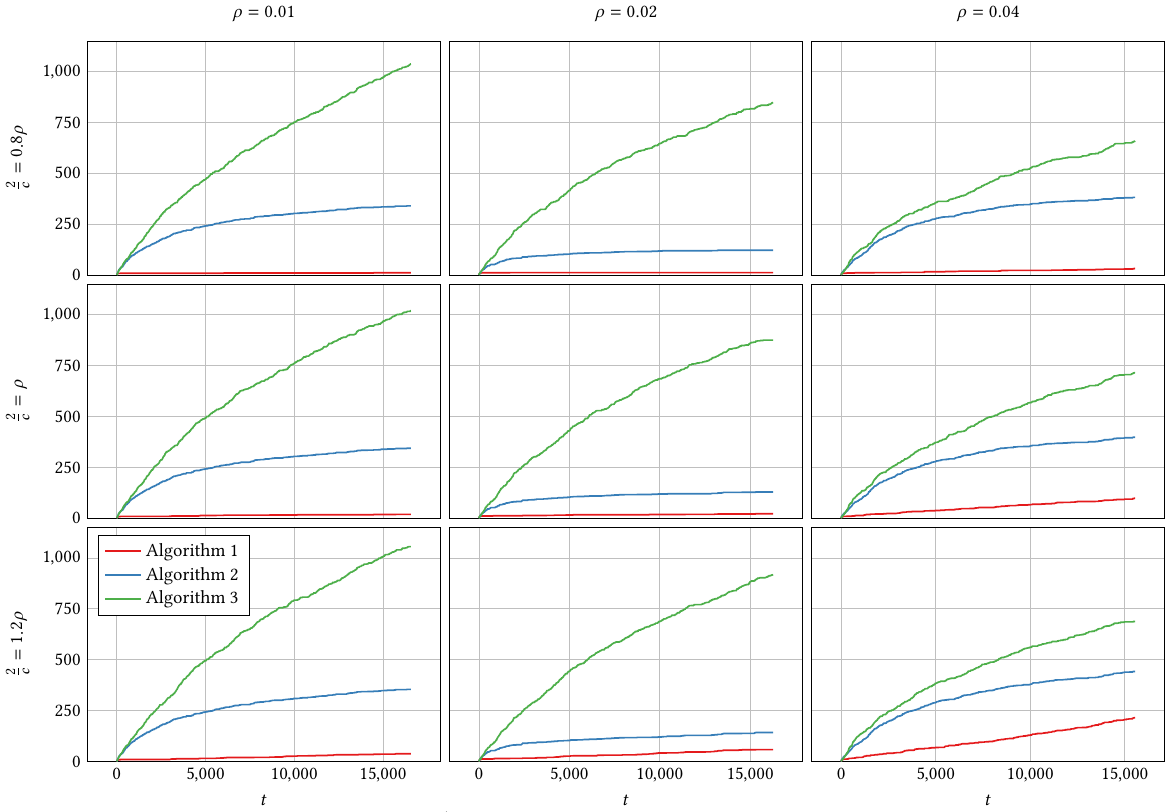}
		\caption{Number of mistakes made by \cref{alg:projected-perceptron,alg:data-driven,alg:data-driven-subgradient-averaging} on loan data, with different margins $\rho\in\{0.01, 0.02, 0.04\}$ and $2/c\in\{0.8\rho, \rho, 1.2\rho\}$ and agent response noise level $\sigma=10^{-3}$.}
		\label{fig:mistake-real-data-noise1e-3}
		
	\end{figure}
\end{landscape}

\begin{landscape}
	\begin{figure}
		\centering
		\includegraphics[scale=.9]{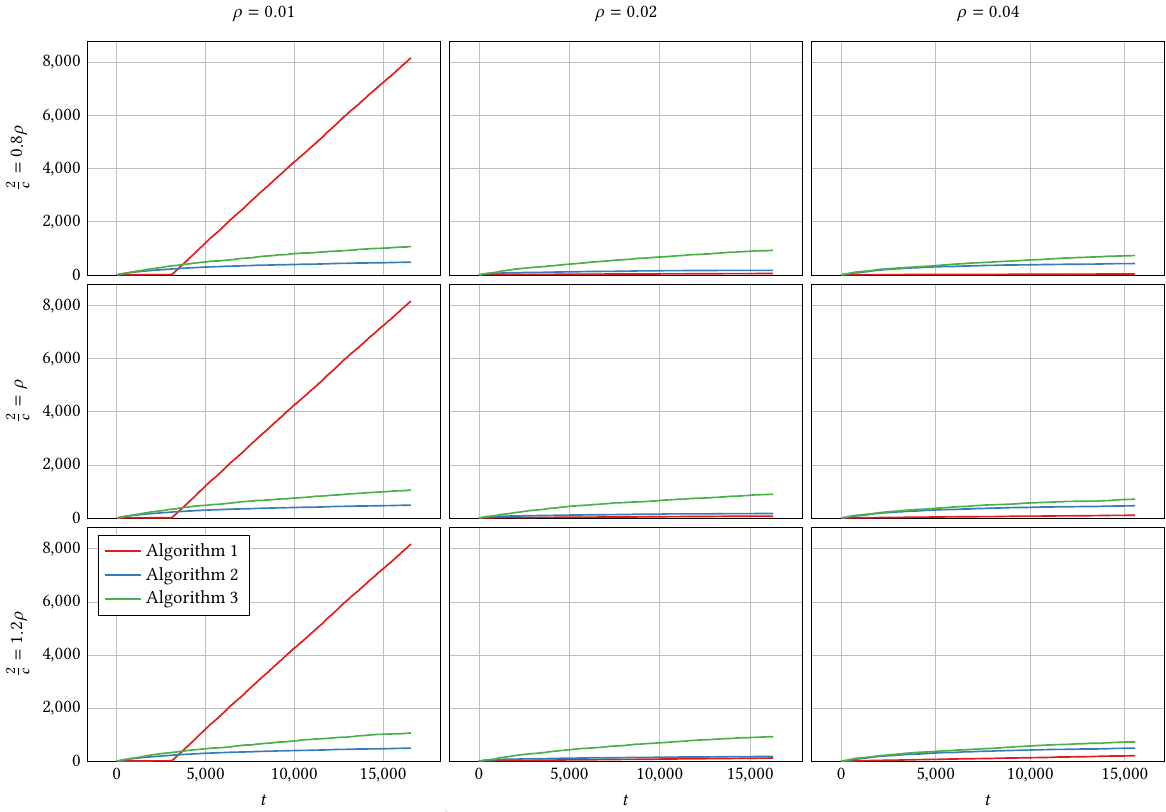}
		\caption{Number of mistakes made by \cref{alg:projected-perceptron,alg:data-driven,alg:data-driven-subgradient-averaging} on loan data, with different margins $\rho\in\{0.01, 0.02, 0.04\}$ and $2/c\in\{0.8\rho, \rho, 1.2\rho\}$ and agent response noise level $\sigma=10^{-2}$.}
		\label{fig:mistake-real-data-noise1e-2}
		
	\end{figure}
\end{landscape}

\begin{landscape}
	\begin{figure}
		\centering
		\includegraphics[scale=.9]{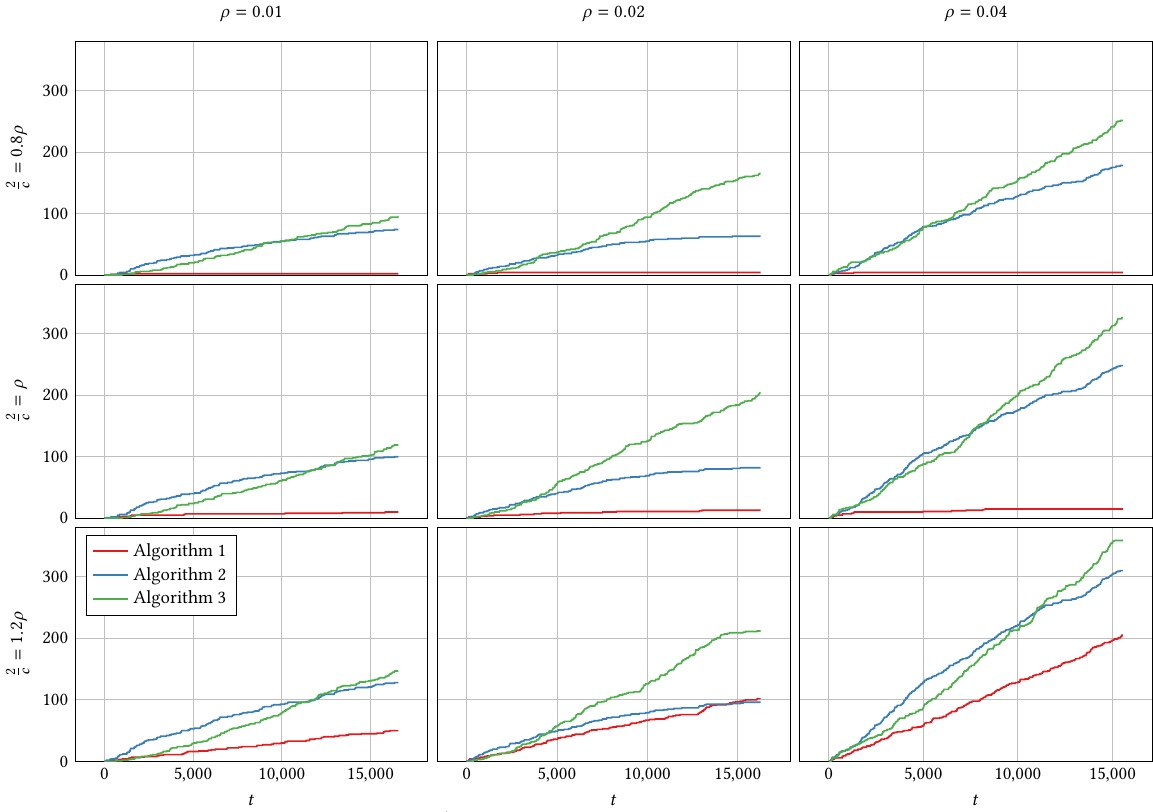}
		\caption{Number of manipulations incurred by \cref{alg:projected-perceptron,alg:data-driven,alg:data-driven-subgradient-averaging} on loan data, with different margins $\rho\in\{0.01, 0.02, 0.04\}$ and $2/c\in\{0.8\rho, \rho, 1.2\rho\}$ and agent response noise level $\sigma=0$.}
		\label{fig:manipulation-real-data}
		
	\end{figure}
\end{landscape}

\begin{landscape}
	\begin{figure}
		\centering
		\includegraphics[scale=.9]{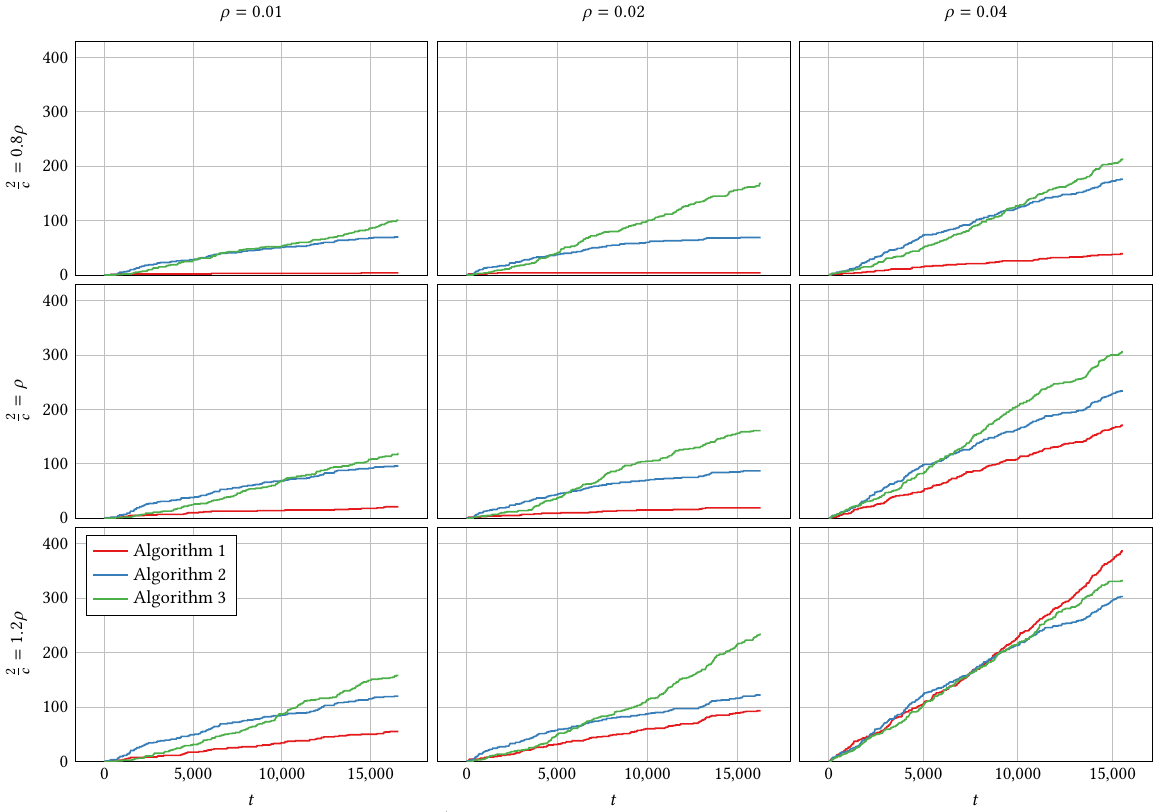}
		\caption{Number of manipulations incurred by \cref{alg:projected-perceptron,alg:data-driven,alg:data-driven-subgradient-averaging} on loan data, with different margins $\rho\in\{0.01, 0.02, 0.04\}$ and $2/c\in\{0.8\rho, \rho, 1.2\rho\}$ and agent response noise level $\sigma=10^{-3}$.}
		\label{fig:manipulation-real-data-noise1e-3}
		
	\end{figure}
\end{landscape}

\begin{landscape}
	\begin{figure}
		\centering
		\includegraphics[scale=.9]{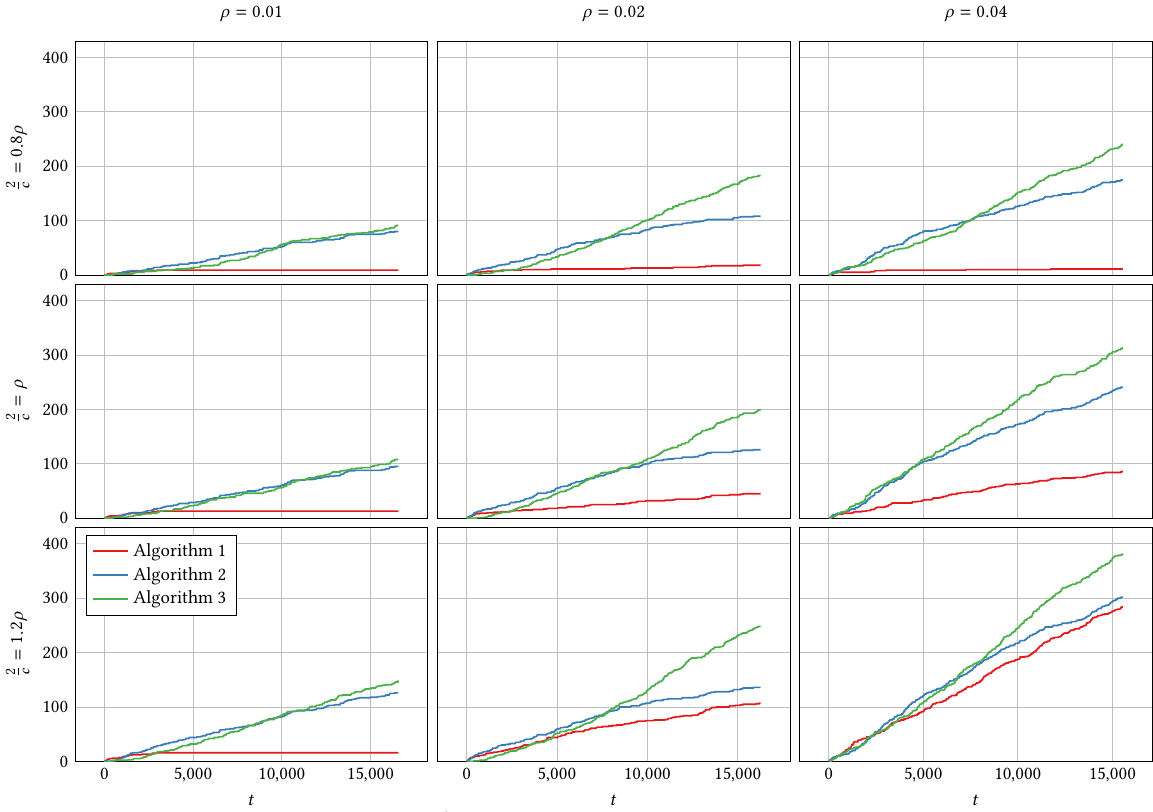}
		\caption{Number of manipulations incurred by \cref{alg:projected-perceptron,alg:data-driven,alg:data-driven-subgradient-averaging} on loan data, with different margins $\rho\in\{0.01, 0.02, 0.04\}$ and $2/c\in\{0.8\rho, \rho, 1.2\rho\}$ and agent response noise level $\sigma=10^{-2}$.}
		\label{fig:manipulation-real-data-noise1e-2}
		
	\end{figure}
\end{landscape}

\begin{table}[htbp]
	\centering
	\footnotesize
	\begin{tabular}{c|r|rrr|rrr|rrr}
		\multicolumn{11}{c}{$\sigma = 0$} \\
		\toprule
		\multirow{2}{*}{\centering $\frac2c$} & \multirow{2}{*}{\centering $t$} & \multicolumn{3}{c|}{$\rho=0.01$} & \multicolumn{3}{c|}{$\rho=0.02$} & \multicolumn{3}{c}{$\rho=0.04$} \\
& & Alg \ref{alg:data-driven} & Alg \ref{alg:data-driven-subgradient-averaging} & Alg \ref{alg:projected-perceptron} & Alg \ref{alg:data-driven} & Alg \ref{alg:data-driven-subgradient-averaging} & Alg \ref{alg:projected-perceptron} & Alg \ref{alg:data-driven} & Alg \ref{alg:data-driven-subgradient-averaging} & Alg \ref{alg:projected-perceptron} \\ 
		\midrule
		\multirow{2}{*}{$0.8\rho$} &  
		250 &     9 &    37 &    37 &     9 &    28 &    35 &     9 &    31 &    32 \\
		& 15000 &     9 &   308 &   985 &     9 &   130 &   861 &     9 &   319 &   631 \\
		\midrule
		\multirow{2}{*}{$1.0\rho$} &  
		250 &     9 &    37 &    37 &     9 &    28 &    35 &     9 &    31 &    32 \\
		& 15000 &     9 &   309 &   954 &     9 &   130 &   799 &     9 &   311 &   630 \\
		\midrule
		\multirow{2}{*}{$1.2\rho$} &  
		250 &     9 &    37 &    37 &     9 &    28 &    35 &     9 &    31 &    32 \\
		& 15000 &     9 &   310 &  1017 &     9 &   128 &   823 &     9 &   311 &   625 \\
		
		\bottomrule\addlinespace[.5em]
		\multicolumn{11}{c}{$\sigma = 10^{-3}$} \\
		\toprule
		\multirow{2}{*}{\centering $\frac2c$} & \multirow{2}{*}{\centering $t$} & \multicolumn{3}{c|}{$\rho=0.01$} & \multicolumn{3}{c|}{$\rho=0.02$} & \multicolumn{3}{c}{$\rho=0.04$} \\
& & Alg \ref{alg:data-driven} & Alg \ref{alg:data-driven-subgradient-averaging} & Alg \ref{alg:projected-perceptron} & Alg \ref{alg:data-driven} & Alg \ref{alg:data-driven-subgradient-averaging} & Alg \ref{alg:projected-perceptron} & Alg \ref{alg:data-driven} & Alg \ref{alg:data-driven-subgradient-averaging} & Alg \ref{alg:projected-perceptron} \\ 
		\midrule
		\multirow{2}{*}{$0.8\rho$} &  
		250 &     9 &    37 &    37 &    11 &    29 &    35 &     9 &    32 &    32 \\
		& 15000 &    11 &   335 &   971 &    12 &   122 &   817 &    30 &   380 &   648 \\
		\midrule
		\multirow{2}{*}{$1.0\rho$} &  
		250 &     9 &    36 &    37 &    11 &    29 &    35 &     9 &    31 &    33 \\
		& 15000 &    18 &   339 &   966 &    21 &   128 &   860 &    94 &   396 &   706 \\
		\midrule
		\multirow{2}{*}{$1.2\rho$} &  
		250 &     9 &    36 &    37 &    12 &    30 &    35 &    11 &    32 &    33 \\
		& 15000 &    36 &   348 &  1006 &    57 &   138 &   883 &   206 &   439 &   686 \\
		\bottomrule\addlinespace[.5em]
		\multicolumn{11}{c}{$\sigma = 10^{-2}$} \\
		\toprule
		\multirow{2}{*}{\centering $\frac2c$} & \multirow{2}{*}{\centering $t$} & \multicolumn{3}{c|}{$\rho=0.01$} & \multicolumn{3}{c|}{$\rho=0.02$} & \multicolumn{3}{c}{$\rho=0.04$} \\
& & Alg \ref{alg:data-driven} & Alg \ref{alg:data-driven-subgradient-averaging} & Alg \ref{alg:projected-perceptron} & Alg \ref{alg:data-driven} & Alg \ref{alg:data-driven-subgradient-averaging} & Alg \ref{alg:projected-perceptron} & Alg \ref{alg:data-driven} & Alg \ref{alg:data-driven-subgradient-averaging} & Alg \ref{alg:projected-perceptron} \\ 
		\midrule
		\multirow{2}{*}{$0.8\rho$} &  
		250 &     8 &    37 &    39 &    15 &    30 &    35 &    10 &    32 &    33 \\
		& 15000 &  7235 &   463 &  1007 &    53 &   176 &   895 &    39 &   430 &   724 \\
		\midrule
		\multirow{2}{*}{$1.0\rho$} &  
		250 &     8 &    37 &    39 &    16 &    30 &    35 &    10 &    32 &    31 \\
		& 15000 &  7236 &   470 &   991 &    72 &   172 &   866 &   108 &   463 &   703 \\
		\midrule
		\multirow{2}{*}{$1.2\rho$} &  
		250 &     8 &    37 &    39 &    17 &    31 &    38 &    11 &    32 &    31 \\
		& 15000 &  7238 &   475 &  1020 &   115 &   179 &   895 &   205 &   493 &   728 \\
		\bottomrule\addlinespace[.5em]
	\end{tabular}
	\caption{Number of mistakes made by \cref{alg:projected-perceptron,alg:data-driven,alg:data-driven-subgradient-averaging} on loan data, with different margins  $\rho\in\{0.01, 0.02, 0.04\}$, $2/c\in\{0.8\rho, \rho, 1.2\rho\}$, and agent response noise level $\sigma\in\{0,10^{-3},10^{-2}\}$.}
	\label{tab:mistake-bound-real-data}
\end{table}

\begin{table}[htbp]
	\centering
	\footnotesize
	\begin{tabular}{c|r|rrr|rrr|rrr}
		\multicolumn{11}{c}{$\sigma = 0$} \\
		\toprule
		\multirow{2}{*}{\centering $\frac2c$} & \multirow{2}{*}{\centering $t$} & \multicolumn{3}{c|}{$\rho=0.01$} & \multicolumn{3}{c|}{$\rho=0.02$} & \multicolumn{3}{c}{$\rho=0.04$} \\
& & Alg \ref{alg:data-driven} & Alg \ref{alg:data-driven-subgradient-averaging} & Alg \ref{alg:projected-perceptron} & Alg \ref{alg:data-driven} & Alg \ref{alg:data-driven-subgradient-averaging} & Alg \ref{alg:projected-perceptron} & Alg \ref{alg:data-driven} & Alg \ref{alg:data-driven-subgradient-averaging} & Alg \ref{alg:projected-perceptron} \\ 
		\midrule  
		\multirow{2}{*}{$0.8\rho$} &  
		250 &     0 &     0 &     0 &     2 &     1 &     0 &     2 &     3 &     6 \\
		& 15000 &     2 &    70 &    83 &     4 &    63 &   155 &     4 &   175 &   241 \\
		\midrule
		\multirow{2}{*}{$1.0\rho$} &
		250 &     0 &     1 &     0 &     2 &     3 &     0 &     3 &     5 &     6 \\
		& 15000 &     9 &    96 &   102 &    13 &    81 &   184 &    15 &   243 &   312 \\
		\midrule
		\multirow{2}{*}{$1.2\rho$} &
		250 &     0 &     2 &     0 &     4 &     3 &     0 &     4 &     7 &     6 \\
		& 15000 &    45 &   121 &   131 &    97 &    94 &   209 &   197 &   304 &   355 \\
		\bottomrule\addlinespace[.5em]
		\multicolumn{11}{c}{$\sigma = 10^{-3}$} \\
		\toprule
		\multirow{2}{*}{\centering $\frac2c$} & \multirow{2}{*}{\centering $t$} & \multicolumn{3}{c|}{$\rho=0.01$} & \multicolumn{3}{c|}{$\rho=0.02$} & \multicolumn{3}{c}{$\rho=0.04$} \\
& & Alg \ref{alg:data-driven} & Alg \ref{alg:data-driven-subgradient-averaging} & Alg \ref{alg:projected-perceptron} & Alg \ref{alg:data-driven} & Alg \ref{alg:data-driven-subgradient-averaging} & Alg \ref{alg:projected-perceptron} & Alg \ref{alg:data-driven} & Alg \ref{alg:data-driven-subgradient-averaging} & Alg \ref{alg:projected-perceptron} \\ 
		\midrule  
		\multirow{2}{*}{$0.8\rho$} &  
		250 &     0 &     0 &     0 &     2 &     1 &     0 &     1 &     3 &     3 \\
		& 15000 &     4 &    68 &    86 &     4 &    68 &   157 &    38 &   173 &   205 \\
		\midrule
		\multirow{2}{*}{$1.0\rho$} &
		250 &     0 &     1 &     0 &     2 &     1 &     0 &     1 &     5 &     4 \\
		& 15000 &    18 &    92 &   108 &    19 &    85 &   156 &   166 &   229 &   300 \\
		\midrule
		\multirow{2}{*}{$1.2\rho$} &
		250 &     0 &     1 &     0 &     4 &     3 &     0 &     6 &     7 &     5 \\
		& 15000 &    50 &   115 &   149 &    89 &   116 &   216 &   371 &   296 &   331 \\
		\bottomrule\addlinespace[.5em]
		\multicolumn{11}{c}{$\sigma = 10^{-2}$} \\
		\toprule
		\multirow{2}{*}{\centering $\frac2c$} & \multirow{2}{*}{\centering $t$} & \multicolumn{3}{c|}{$\rho=0.01$} & \multicolumn{3}{c|}{$\rho=0.02$} & \multicolumn{3}{c}{$\rho=0.04$} \\
& & Alg \ref{alg:data-driven} & Alg \ref{alg:data-driven-subgradient-averaging} & Alg \ref{alg:projected-perceptron} & Alg \ref{alg:data-driven} & Alg \ref{alg:data-driven-subgradient-averaging} & Alg \ref{alg:projected-perceptron} & Alg \ref{alg:data-driven} & Alg \ref{alg:data-driven-subgradient-averaging} & Alg \ref{alg:projected-perceptron} \\ 
		\midrule  
		\multirow{2}{*}{$0.8\rho$} &  
		250 &     2 &     0 &     0 &     3 &     3 &     0 &     3 &     5 &     3 \\
		& 15000 &     9 &    75 &    79 &    17 &   105 &   167 &    11 &   171 &   232 \\
		\midrule
		\multirow{2}{*}{$1.0\rho$} &
		250 &     3 &     0 &     0 &     4 &     4 &     0 &     3 &     6 &     3 \\
		& 15000 &    13 &    88 &    93 &    43 &   123 &   186 &    84 &   234 &   305 \\
		\midrule
		\multirow{2}{*}{$1.2\rho$} &
		250 &     4 &     0 &     0 &     7 &     5 &     0 &     6 &     7 &     4 \\
		& 15000 &    16 &   118 &   134 &   103 &   132 &   231 &   276 &   294 &   373 \\
		\bottomrule\addlinespace[.5em]
	\end{tabular}
	\caption{Number of manipulations caused by \cref{alg:projected-perceptron,alg:data-driven,alg:data-driven-subgradient-averaging} on loan data, with different margins  $\rho\in\{0.01, 0.02, 0.04\}$, $2/c\in\{0.8\rho, \rho, 1.2\rho\}$, and agent response noise level $\sigma\in\{0,10^{-3},10^{-2}\}$.}
	\label{tab:manipulation-real-data}
\end{table}

\begin{table}[htbp]
  \centering
  \footnotesize
  \begin{tabular}{c|r|rrr|rrr|rrr}
    \multicolumn{11}{c}{$\sigma = 0$} \\
    \toprule
    \multirow{2}{*}{\centering $\frac2c$} & \multirow{2}{*}{\centering $t$} & \multicolumn{3}{c|}{$\rho=0.01$} & \multicolumn{3}{c|}{$\rho=0.02$} & \multicolumn{3}{c}{$\rho=0.04$} \\
& & Alg \ref{alg:data-driven} & Alg \ref{alg:data-driven-subgradient-averaging} & Alg \ref{alg:projected-perceptron} & Alg \ref{alg:data-driven} & Alg \ref{alg:data-driven-subgradient-averaging} & Alg \ref{alg:projected-perceptron} & Alg \ref{alg:data-driven} & Alg \ref{alg:data-driven-subgradient-averaging} & Alg \ref{alg:projected-perceptron} \\
    \midrule
    \multirow{2}{*}{$0.8\rho$} &
        250 &  0.54 &  0.07 &  0.01 &  0.54 &  0.06 &  0.01 &  0.51 &  0.06 &  0.01 \\
    & 15000 & 24.91 &  5.63 &  0.47 & 26.49 &  5.58 &  0.46 & 25.78 &  5.60 &  0.46 \\
    \midrule
    \multirow{2}{*}{$1.0\rho$} &
        250 &  0.54 &  0.07 &  0.01 &  0.54 &  0.05 &  0.01 &  0.51 &  0.05 &  0.01 \\
    & 15000 & 24.93 &  5.59 &  0.47 & 26.66 &  5.56 &  0.47 & 26.57 &  5.59 &  0.47 \\
    \midrule
    \multirow{2}{*}{$1.2\rho$} &
        250 &  0.54 &  0.07 &  0.01 &  0.53 &  0.05 &  0.01 &  0.51 &  0.06 &  0.01 \\
    & 15000 & 25.32 &  5.58 &  0.47 & 25.26 &  5.86 &  0.46 & 26.28 &  5.61 &  0.47 \\
    \bottomrule\addlinespace[.5em]
    \multicolumn{11}{c}{$\sigma = 10^{-3}$} \\
    \toprule
    \multirow{2}{*}{\centering $\frac2c$} & \multirow{2}{*}{\centering $t$} & \multicolumn{3}{c|}{$\rho=0.01$} & \multicolumn{3}{c|}{$\rho=0.02$} & \multicolumn{3}{c}{$\rho=0.04$} \\
& & Alg \ref{alg:data-driven} & Alg \ref{alg:data-driven-subgradient-averaging} & Alg \ref{alg:projected-perceptron} & Alg \ref{alg:data-driven} & Alg \ref{alg:data-driven-subgradient-averaging} & Alg \ref{alg:projected-perceptron} & Alg \ref{alg:data-driven} & Alg \ref{alg:data-driven-subgradient-averaging} & Alg \ref{alg:projected-perceptron} \\
    \midrule
    \multirow{2}{*}{$0.8\rho$} &
        250 &  0.53 &  0.07 &  0.01 &  0.53 &  0.05 &  0.01 &  0.51 &  0.05 &  0.01 \\
    & 15000 & 25.71 &  5.59 &  0.47 & 25.78 &  5.61 &  0.47 & 27.94 &  5.65 &  0.47 \\
    \midrule
    \multirow{2}{*}{$1.0\rho$} &
        250 &  0.53 &  0.07 &  0.01 &  0.52 &  0.05 &  0.01 &  0.54 &  0.05 &  0.01 \\
    & 15000 & 25.28 &  5.71 &  0.47 & 26.07 &  5.57 &  0.47 &  24.4 &  5.60 &  0.48 \\
    \midrule
    \multirow{2}{*}{$1.2\rho$} &
        250 &  0.53 &  0.07 &  0.01 &  0.52 &  0.05 &  0.01 &  0.52 &  0.05 &  0.01 \\
    & 15000 & 25.34 &  5.57 &  0.47 & 24.65 &  5.49 &  0.47 & 24.21 &  5.67 &  0.47 \\
    \bottomrule\addlinespace[.5em]
    \multicolumn{11}{c}{$\sigma = 10^{-2}$} \\
    \toprule
    \multirow{2}{*}{\centering $\frac2c$} & \multirow{2}{*}{\centering $t$} & \multicolumn{3}{c|}{$\rho=0.01$} & \multicolumn{3}{c|}{$\rho=0.02$} & \multicolumn{3}{c}{$\rho=0.04$} \\
& & Alg \ref{alg:data-driven} & Alg \ref{alg:data-driven-subgradient-averaging} & Alg \ref{alg:projected-perceptron} & Alg \ref{alg:data-driven} & Alg \ref{alg:data-driven-subgradient-averaging} & Alg \ref{alg:projected-perceptron} & Alg \ref{alg:data-driven} & Alg \ref{alg:data-driven-subgradient-averaging} & Alg \ref{alg:projected-perceptron} \\
    \midrule  
    \multirow{2}{*}{$0.8\rho$} &
        250 &  0.50 &  0.07 &  0.01 &  0.55 &  0.05 &  0.01 &  0.50 &  0.05 &  0.01 \\
    & 15000 & 23.73 &  5.57 &  0.47 & 24.97 &  5.45 &  0.47 & 26.53 &  5.60 &  0.47 \\
    \midrule
    \multirow{2}{*}{$1.0\rho$} &
        250 &  0.49 &  0.07 &  0.01 &  0.55 &  0.05 &  0.01 &  0.50 &  0.05 &  0.01 \\
    & 15000 & 23.83 &  5.53 &  0.47 & 25.01 &  5.61 &  0.46 & 25.86 &  5.55 &  0.48 \\
    \midrule
    \multirow{2}{*}{$1.2\rho$} &
        250 &  0.49 &  0.07 &  0.01 &  0.54 &  0.05 &  0.01 &  0.50 &  0.06 &  0.01 \\
    & 15000 & 24.18 &  5.64 &  0.47 & 24.86 &  5.56 &  0.46 & 26.74 &  5.56 &  0.48 \\
    \bottomrule\addlinespace[.5em]
  \end{tabular}
  \caption{CPU running time (in seconds) of \cref{alg:projected-perceptron,alg:data-driven,alg:data-driven-subgradient-averaging} on loan data, with different margins  $\rho\in\{0.01, 0.02, 0.04\}$, $2/c\in\{0.8\rho, \rho, 1.2\rho\}$, and agent response noise level $\sigma\in\{0,10^{-3},10^{-2}\}$.}
  \label{tab:time}
\end{table}

\subsection{Numerical results on synthetic data}
\label{sec:experiment-synthetic}

In this section, we explore a synthetic dataset different from the loan dataset considered in \cref{sec:numerical}. Here, feature vectors are generated by the truncated normal distribution on $\{x\in\R^6: \|x\|_2\leq1/\sqrt5\}$ with mean $0$ and covariance $0.04I_6$. Labels are designated using a predefined classifier $y=\mathbf{1}$, $b=0$, i.e., $\lbl(A_t) = \sign(y^\top A_t)$. Similar to the preprocessing procedure in \cref{sec:numerical}, we remove the points that are within a distance of $\rho>0$ to the hyperplane defined by $(y,b)$ to create a strictly positive margin. We also follow the same process outlined in~\cref{sec:numerical} to compute the best margin classifier $(y_*,b_*')$. Then, the data are translated by $\frac12(x_++x_-)$ where $x_+, x_-$ are computed from \cref{lem:margin-prod} (see its proof for an explicit expression using the primal and dual variables of the margin maximization problem, which can be retrieved from the MOSEK solver) so that the best margin classifier of the resulting data becomes $(y_*,b_*)=(y_*,0)$. In this way, we can implement and test the performance of \cref{alg:projected-perceptron} with $\L=\R^d\times\{0\}$. We mention that this in fact makes the task more challenging for \cref{alg:data-driven,alg:data-driven-subgradient-averaging}, as they do not have access to the information $b_*=0$ and need to figure it out via the mechanism of the algorithms themselves. However, as we will see in the numerical results, \cref{alg:data-driven,alg:data-driven-subgradient-averaging} still have better performance in spite of this disadvantage. We do not add agent response noise for our synthetic experiments.
The numerical results for the synthetic data are shown in \cref{fig:distance-synthetic,fig:manipulation-synthetic,fig:mistake-synthetic,tab:mistake-synthetic,tab:manipulation-synthetic,tab:time-synthetic}. Our observations here are consistent with the ones for the real data
in terms of performance comparison among the three algorithms. On the other hand, in the synthetic data experiments the performance of \cref{alg:data-driven-subgradient-averaging}, in terms of both convergence to $(y_*,b_*)$ and number of mistakes, seems to improve and approach to that of \cref{alg:data-driven} in some parameter settings. Also, in almost all parameter combinations tested, the number of mistakes made by any algorithm is lower than that of loan data experiments and they all seem to stabilize after certain number of iterations, unlike the loan data experiments in which only \cref{alg:data-driven} seems stabilize within the finitely many iterations tested.

\begin{landscape}
  \begin{figure}
    \centering
    \includegraphics[scale=.9]{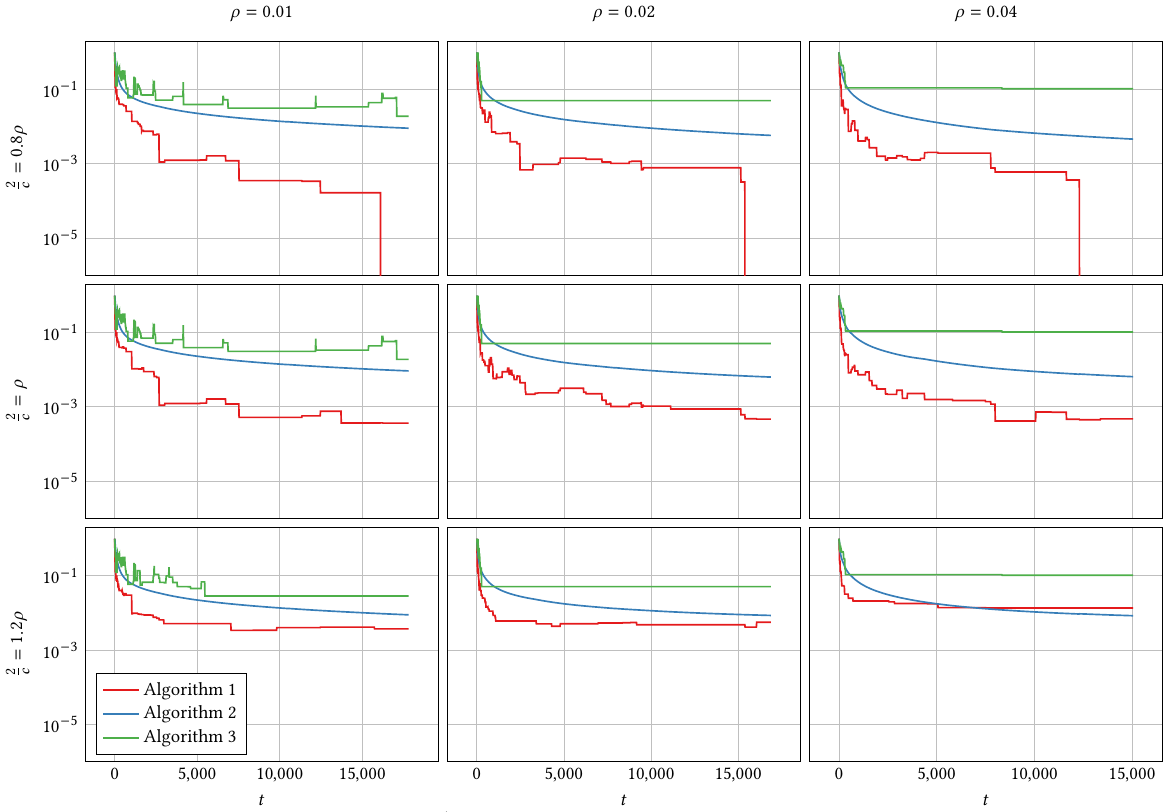} 
    \caption{Distance $\left\|\frac{(y_t,b_t)}{\|y_t\|_2} - \frac{(y_*,b_*)}{\|y_*\|_2}\right\|_2$ between $(y_*,b_*)$ and $(y_t,b_t)$ normalized by $y_*$ and $y_t$ respectively for \cref{alg:projected-perceptron,alg:data-driven,alg:data-driven-subgradient-averaging} on synthetic data, with different margins $\rho\in\{0.01, 0.02, 0.04\}$ and $2/c\in\{0.8\rho, \rho, 1.2\rho\}$.}
    \label{fig:distance-synthetic}
    
  \end{figure}
\end{landscape}  

\begin{landscape}
  \begin{figure}
    \centering
    \includegraphics[scale=.9]{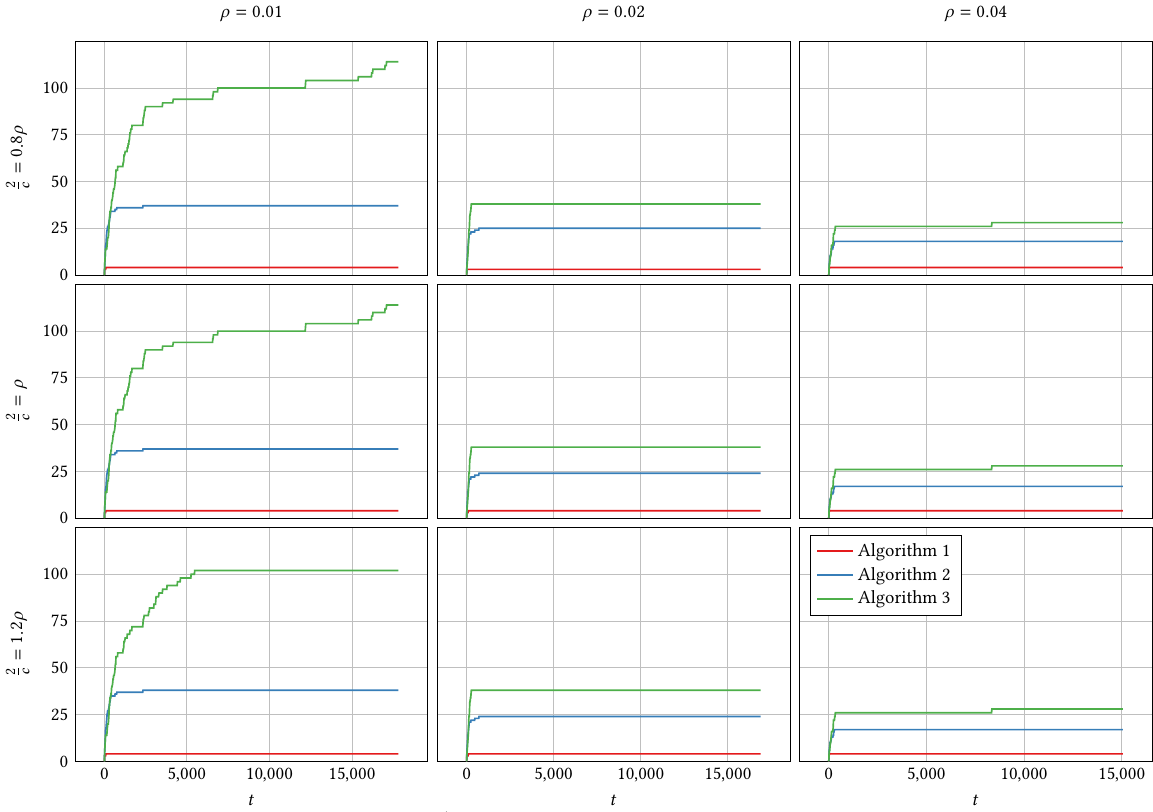}
    \caption{Number of mistakes made by \cref{alg:projected-perceptron,alg:data-driven,alg:data-driven-subgradient-averaging} on synthetic data, with different margins $\rho\in\{0.01, 0.02, 0.04\}$ and $2/c\in\{0.8\rho, \rho, 1.2\rho\}$.}
    \label{fig:mistake-synthetic}
    
  \end{figure}
\end{landscape}

\begin{landscape}
  \begin{figure}
    \centering
    \includegraphics[scale=.9]{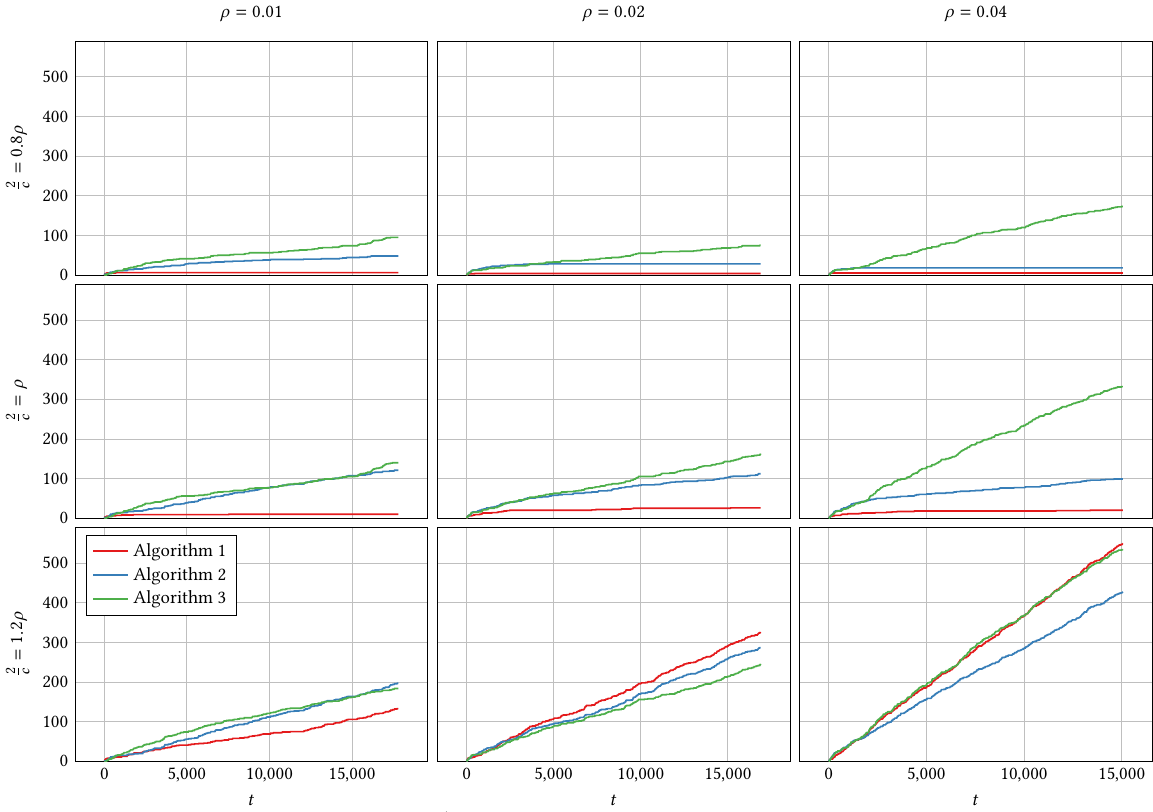}
    \caption{Number of manipulations incurred by \cref{alg:projected-perceptron,alg:data-driven,alg:data-driven-subgradient-averaging} on synthetic data, with different margins $\rho\in\{0.01, 0.02, 0.04\}$ and $2/c\in\{0.8\rho, \rho, 1.2\rho\}$.}
    \label{fig:manipulation-synthetic}
    
  \end{figure}
\end{landscape}

\begin{table}[htbp]
  \centering
\footnotesize
  \begin{tabular}{c|r|rrr|rrr|rrr}
\toprule
    \multirow{2}{*}{\centering $\frac2c$} & \multirow{2}{*}{\centering $t$} & \multicolumn{3}{c|}{$\rho=0.01$} & \multicolumn{3}{c|}{$\rho=0.02$} & \multicolumn{3}{c}{$\rho=0.04$} \\
& & Alg \ref{alg:data-driven} & Alg \ref{alg:data-driven-subgradient-averaging} & Alg \ref{alg:projected-perceptron} & Alg \ref{alg:data-driven} & Alg \ref{alg:data-driven-subgradient-averaging} & Alg \ref{alg:projected-perceptron} & Alg \ref{alg:data-driven} & Alg \ref{alg:data-driven-subgradient-averaging} & Alg \ref{alg:projected-perceptron} \\ 
    \midrule  
    \multirow{2}{*}{$0.8\rho$} &  
        250 &     4 &    27 &    22 &     3 &    23 &    36 &     4 &    16 &    22 \\
    & 15000 &     4 &    37 &   104 &     3 &    25 &    38 &     4 &    18 &    28 \\
    \midrule  
    \multirow{2}{*}{$1.0\rho$} &  
        250 &     4 &    27 &    22 &     4 &    22 &    36 &     4 &    14 &    22 \\
    & 15000 &     4 &    37 &   104 &     4 &    24 &    38 &     4 &    17 &    28 \\
    \midrule  
    \multirow{2}{*}{$1.2\rho$} &  
        250 &     4 &    28 &    22 &     4 &    22 &    36 &     4 &    14 &    22 \\
    & 15000 &     4 &    38 &   102 &     4 &    24 &    38 &     4 &    17 &    28 \\
    \bottomrule \end{tabular}
\caption{Number of mistakes made by \cref{alg:projected-perceptron,alg:data-driven,alg:data-driven-subgradient-averaging} on synthetic data, with different margins  $\rho\in\{0.01, 0.02, 0.04\}$, $2/c\in\{0.8\rho, \rho, 1.2\rho\}$.}
  \label{tab:mistake-synthetic}
\end{table}

\begin{table}[htbp]
  \centering
\footnotesize
  \begin{tabular}{c|r|rrr|rrr|rrr}
\toprule
    \multirow{2}{*}{\centering $\frac2c$} & \multirow{2}{*}{\centering $t$} & \multicolumn{3}{c|}{$\rho=0.01$} & \multicolumn{3}{c|}{$\rho=0.02$} & \multicolumn{3}{c}{$\rho=0.04$} \\
& & Alg \ref{alg:data-driven} & Alg \ref{alg:data-driven-subgradient-averaging} & Alg \ref{alg:projected-perceptron} & Alg \ref{alg:data-driven} & Alg \ref{alg:data-driven-subgradient-averaging} & Alg \ref{alg:projected-perceptron} & Alg \ref{alg:data-driven} & Alg \ref{alg:data-driven-subgradient-averaging} & Alg \ref{alg:projected-perceptron} \\ 
    \midrule  
    \multirow{2}{*}{$0.8\rho$} &  
        250 &     5 &     4 &     4 &     3 &     9 &     9 &     5 &    10 &    10 \\
    & 15000 &     6 &    44 &    74 &     4 &    28 &    68 &     5 &    18 &   173 \\
    \midrule  
    \multirow{2}{*}{$1.0\rho$} &  
        250 &     5 &     5 &     4 &     6 &    11 &    12 &     6 &    12 &    11 \\
    & 15000 &    10 &   107 &   105 &    25 &   103 &   143 &    20 &    99 &   332 \\
    \midrule  
    \multirow{2}{*}{$1.2\rho$} &  
        250 &     8 &     6 &     5 &     9 &    11 &    13 &    10 &    16 &    15 \\
    & 15000 &   106 &   164 &   162 &   291 &   259 &   214 &   547 &   425 &   534 \\
    \bottomrule \end{tabular}
\caption{Number of manipulations caused by \cref{alg:projected-perceptron,alg:data-driven,alg:data-driven-subgradient-averaging} on synthetic data, with different margins  $\rho\in\{0.01, 0.02, 0.04\}$, $2/c\in\{0.8\rho, \rho, 1.2\rho\}$.}
  \label{tab:manipulation-synthetic}
\end{table}

\begin{table}[htbp]
  \centering
  \footnotesize
  \begin{tabular}{c|r|rrr|rrr|rrr}
\toprule
    \multirow{2}{*}{\centering $\frac2c$} & \multirow{2}{*}{\centering $t$} & \multicolumn{3}{c|}{$\rho=0.01$} & \multicolumn{3}{c|}{$\rho=0.02$} & \multicolumn{3}{c}{$\rho=0.04$} \\
& & Alg \ref{alg:data-driven} & Alg \ref{alg:data-driven-subgradient-averaging} & Alg \ref{alg:projected-perceptron} & Alg \ref{alg:data-driven} & Alg \ref{alg:data-driven-subgradient-averaging} & Alg \ref{alg:projected-perceptron} & Alg \ref{alg:data-driven} & Alg \ref{alg:data-driven-subgradient-averaging} & Alg \ref{alg:projected-perceptron} \\
    \midrule  
    \multirow{2}{*}{$0.8\rho$} &  
        250 &  0.58 &  0.06 &  0.01 &  0.61 &  0.06 &  0.01 &  0.56 &  0.06 &  0.01 \\
    & 15000 & 24.86 &  5.56 &  0.45 & 25.42 &  5.52 &  0.44 & 25.31 &  5.70 &  0.44 \\
    \midrule  
    \multirow{2}{*}{$1.0\rho$} &  
        250 &  0.58 &  0.06 &  0.01 &  0.58 &  0.05 &  0.01 &  0.53 &  0.05 &  0.01 \\
    & 15000 & 25.29 &  5.52 &  0.45 & 25.88 &  5.50 &  0.44 & 25.81 &  5.51 &  0.44 \\
    \midrule  
    \multirow{2}{*}{$1.2\rho$} &  
        250 &  0.60 &  0.06 &  0.01 &  0.61 &  0.05 &  0.01 &  0.53 &  0.05 &  0.01 \\
    & 15000 & 25.08 &  5.75 &  0.44 & 24.64 &  5.71 &  0.44 & 24.37 &  5.45 &  0.44 \\
    \bottomrule \end{tabular}
  \caption{CPU running time (in seconds) of \cref{alg:projected-perceptron,alg:data-driven,alg:data-driven-subgradient-averaging} on synthetic data, with different margins $\rho\in\{0.01, 0.02, 0.04\}$, $2/c\in\{0.8\rho, \rho, 1.2\rho\}$.}
  \label{tab:time-synthetic}
\end{table}
 \end{appendix}

\end{document}